\documentclass[11pt,english]{article}

\usepackage{graphicx} 
\usepackage[T1]{fontenc}

\usepackage[utf8]{inputenc}

\usepackage{geometry}
\usepackage{babel}
\usepackage{amsmath}
\usepackage{amsthm}
\usepackage{amssymb}
\usepackage{enumitem}
\usepackage[authoryear]{natbib}
\usepackage[unicode=true,pdfusetitle,
 bookmarks=true,bookmarksnumbered=false,bookmarksopen=false,
 breaklinks=true,pdfborder={0 0 0},pdfborderstyle={},backref=page,colorlinks=false]
 {hyperref}
\hypersetup{unicode=true, bookmarksnumbered=false,bookmarksopen=false, breaklinks=true,pdfborder={0 0 0},pdfborderstyle={},colorlinks=true,citecolor=blue}

\newcommand\numberthis{\addtocounter{equation}{1}\tag{\theequation}}

\makeatletter
\theoremstyle{plain}
\newtheorem{thm}{\protect\theoremname}
\theoremstyle{plain}
\newtheorem{lem}[thm]{\protect\lemmaname}
\theoremstyle{remark}
\newtheorem{rem}[thm]{\protect\remarkname}
\theoremstyle{plain}

\theoremstyle{plain}

\usepackage{dsfont,appendix}

\makeatother

\providecommand{\corollaryname}{Corollary}
\providecommand{\lemmaname}{Lemma}
\providecommand{\remarkname}{Remark}
\providecommand{\theoremname}{Theorem}
\providecommand{\conjecturename}{Conjecture}

\usepackage{algorithm}
\usepackage{algpseudocode}
\usepackage{float}

\usepackage{subfigure}

\date{}

\title{Gap-Free Clustering: Sensitivity and Robustness of SDP}
\usepackage{authblk}
\author{Matthew Zurek}
\author{Yudong Chen}
\affil{Department of Computer Sciences, University of Wisconsin-Madison\\\texttt{\{matthew.zurek,yudong.chen\}@wisc.edu}}

\begin{document}

\maketitle

\begin{abstract}%
      We study graph clustering in the Stochastic Block Model (SBM) in the presence of both large clusters and small, unrecoverable clusters. Previous convex relaxation approaches achieving exact recovery do not allow any small clusters of size $o(\sqrt{n})$, or require a size gap between the smallest recovered cluster and the largest non-recovered cluster. We provide an algorithm based on semidefinite programming (SDP) which removes these requirements and provably recovers large clusters regardless of the remaining cluster sizes. Mid-sized clusters pose unique challenges to the analysis, since their proximity to the recovery threshold makes them highly sensitive to small noise perturbations and precludes a closed-form candidate solution. We develop novel techniques, including a leave-one-out-style argument which controls the correlation between SDP solutions and noise vectors even when the removal of one row of noise can drastically change the SDP solution. We also develop improved eigenvalue perturbation bounds of potential independent interest. Our results are robust to certain semirandom settings that are challenging for alternative algorithms. Using our gap-free clustering procedure, we obtain efficient algorithms for the problem of clustering with a faulty oracle with superior query complexities, notably achieving $o(n^2)$ sample complexity even in the presence of a large number of small clusters. Our gap-free clustering procedure also leads to improved algorithms for recursive clustering.
\end{abstract}


\global\long\def\E{\mathbb{E}}%
\global\long\def\EE{\mathbb{E}}%
\global\long\def\real{\mathbb{R}}%
\global\long\def\R{\mathbb{R}}%
\global\long\def\P{\mathbb{P}}%
\global\long\def\supp{\text{supp}}%
\global\long\def\indic{\operatorname{\mathds{1}}}%
\global\long\def\vct#1{\overrightarrow{#1}}%
\global\long\def\onevec{\mathbf{1}}%
\global\long\def\N{\mathbb{N}}%

\global\long\def\diag{\operatorname{diag}}%
\global\long\def\rank{\operatorname{rank}}%
\global\long\def\argmax{\operatorname*{argmax}}%
\global\long\def\argmin{\operatorname*{argmin}}%
\global\long\def\vect{\operatorname{vec}}%
\global\long\def\Var{\operatorname{Var}}%
\global\long\def\tr{\operatorname{Tr}}%
\global\long\def\ddup{\mathrm{d}}%
\global\long\def\proj{\operatorname{proj}}%
\global\long\def\col{\operatorname{Col}}%

\global\long\def\Ystar{Y^{*}}%
\global\long\def\Yb{\overline{Y}^{*}}%
\global\long\def\Ybb{\overline{\overline{Y}}^{*}}%
\global\long\def\Ys{\underline{Y}^{*}}%
\global\long\def\Db{\overline{D}}%
\global\long\def\Ds{\underline{D}}%
\global\long\def\Jb{\overline{J}}%
\global\long\def\Jbb{\overline{\overline{J}}}%
\global\long\def\Js{\underline{J}}%
\global\long\def\Jss{\underline{\underline{J}}}%
\global\long\def\smin{s_{\min}}%
\global\long\def\sb{\overline{s}}%
\global\long\def\eb{\overline{e}}%
\global\long\def\ss{\underline{s}}%
\global\long\def\kb{\overline{k}}%
\global\long\def\Yt{\widetilde{Y}^{*}}%
\global\long\def\Yhat{\widehat{Y}}%

\global\long\def\Yhats{\overline{Y}}%
\global\long\def\yhats{\overline{y}}%
\global\long\def\uhats{\overline{u}}%

\global\long\def\Wb{\overline{W}}%

\global\long\def\Ashift{\overleftrightarrow{A}}%

\global\long\def\Atilde{\widetilde{A}}%
\global\long\def\Wtilde{\widetilde{W}}%
\global\long\def\Ytilde{\widetilde{Y}}%
\global\long\def\ytilde{\widetilde{y}}%
\global\long\def\utilde{\widetilde{u}}%
\global\long\def\wtilde{\widetilde{w}}%
\global\long\def\vtilde{\widetilde{v}}%
\global\long\def\lambdatilde{\widetilde{\lambda}}%

\global\long\def\Vb{\overline{V}}%
\global\long\def\Vs{\underline{V}}%

\global\long\def\PT{\mathcal{P}_{\parallel}}%
\global\long\def\PP{\mathcal{P}_{\bot}}%
\global\long\def\PY{\mathcal{P}_{Y}}%
\global\long\def\PYP{\mathcal{P}_{Y\bot}}%

\global\long\def\twonorm#1{\left\Vert #1\right\Vert _{2}}%
\global\long\def\fronorm#1{\left\Vert #1\right\Vert _{\textrm{F}}}%
\global\long\def\opnorm#1{\left\Vert #1\right\Vert _{\textnormal{op}}}%
\global\long\def\infnorm#1{\left\Vert #1\right\Vert _{\infty}}%
\global\long\def\onenorm#1{\left\Vert #1\right\Vert _{1}}%
\global\long\def\nucnorm#1{\left\Vert #1\right\Vert _{\textnormal{nuc}}}%

\section{Introduction}
Graph clustering is a fundamental problem with various applications to the study of diverse real-world networks, such as social, biological, and information networks. A standard approach for benchmarking clustering algorithms is to consider probabilistic models for generating a graph with a planted set of unknown, ground-truth clusters. Beyond their modeling power of practical networks, the study of such models over the past decades has led to notable theoretical advancement in the interface of combinatorial optimization, statistical learning, information theory and applied probability, greatly enriching our understanding of computational and statistical tradeoffs. For a survey of the algorithmic and theoretical work in this area, see \cite{abbe_community_2018}.

In this paper, we consider the canonical model called Stochastic Block Model (SBM), also known as the Planted Partition Model \citep{holland83,condon2001algorithms}. We focus on the unbalanced clusters setting, formally defined below, where the planted clusters have different sizes. 

As we discuss in greater details below, most existing results, with a few notable exceptions, require all clusters to be sufficiently large or the existence of a gap between the large and small clusters. The goal of this paper is to eliminate such assumptions.

\section{Problem Setup and Prior Art}
\label{sec:problem_setup}

Let $ [n]:=\{1, \dots, n\}$ be a set of $n$ vertices partitioned into $K$ unknown clusters $V_1, \dots, V_K$. We observe a random graph on these vertices with adjacency matrix $A\in\{0,1\}^{n\times n}$, where for two numbers $1\ge p>q\ge0$, and independently across
all $i\leq j$, 
\[ A_{ij}\sim\begin{cases}
\text{Bernoulli}(p), & \text{if $i,j$ belong to the same cluster},\\
\text{Bernoulli}(q), & \text{if $i,j$ do not belong to the same cluster}.
\end{cases}
\]
and $A_{ij}=A_{ji}$ for all $i>j$. In other words, $p$ is the intra-cluster edge
probability and $q$ is the
inter-cluster edge probability. Note that $p,q$ and $K$ are allow to depend on $n$.

We let $(s_1, \dots, s_K) = (|V_1|, \dots, |V_K|)$ be the cluster sizes, where the sizes are ordered as $s_1 \geq s_2 \geq \dots s_K$. Since all our algorithms and analysis do not rely on the ordering of the vertices, we further assume, without loss of generality and for notational convenience, that the vertices are ordered by their cluster membership, so that $V_1 = \{1, \dots, s_1\}$, $V_2 = \{s_1+1, \dots, s_1 + s_2\}$, and so on. We encode the ground truth clustering using a matrix $\Ystar\in\{0,1\}^{n\times n}$,  defined by setting $\Ystar_{ij}=1$ if and only if $i$ and $j$ belong to the same cluster. By our ordering assumption $\Ystar$ is block-diagonal and has the form $\Ystar = \diag \left(J_{s_1 \times s_1}, \dots, J_{s_K \times s_K} \right)$, where $J_{m \times r}$ denotes an all-ones matrix of size $m \times r$. Note that the cluster sizes $\{s_i\}$  correspond to the eigenvalues of the  $\Ystar$. We  assume that $p \geq \log n /n$, a common assumption in the literature and known to be  necessary for exact recovery.

\subsection{Algorithms and Prior Results for Unbalanced SBM}
\label{sec:algorithms_for_unbalanced_sbm}

A plethora of algorithms have been developed for graph clustering, and many of them enjoy rigorous performance guarantees; see, for instance, \cite{mcsherry_spectral_2001,bollobas_max_2004,chen_clustering_2012,chaudhuri_spectral_2012,vu_simple_2018,abbe_community_2018}, and the references therein. Nearly all theoretical guarantees for the problem of exact recovery in the SBM require all clusters to be ``large'', typically of size $\widetilde{\Omega}(\sqrt{n})$. Under such an assumption, these algorithms recover all clusters. From theoretical and practical standpoints, it is desirable to have algorithms which, despite the presence of ``small'' clusters below this threshold, provably recover all sufficiently large clusters. While this is a natural goal, establishing such a result is surprisingly nontrivial. Only a few existing results allow the presence of small clusters, \citet{ailon2013breaking,ailon2015iterative}, \cite{mukherjee2022unbalanced}, and \cite{mukherjee_detecting_2022}. Postponing the discussion of the spectral-based methods \cite{mukherjee2022unbalanced} and \cite{mukherjee_detecting_2022} to after presenting our main theorem (in Section \ref{sec:our_contributions}), we first discuss the limitations of all existing convex-relaxation-based approaches which permit small clusters, including \citet{ailon2013breaking,ailon2015iterative}.

For simplicity of presentation, in order to illustrate the limitations of all existing convex-relaxation-based approaches, we present only the guarantees of own ``warm-up'' result which is of a very similar (but slightly improved and much simpler) form to \citet{ailon2013breaking,ailon2015iterative}, and defer discussion of the finer details of \citet{ailon2013breaking,ailon2015iterative} to Section \ref{sec:details_of_ailon}. All our results are based on solving the following trace-regularized semi-definite program (SDP):
\begin{subequations}
\label{eq:SDP_regularized_original}
\begin{align}
\Yhat=\argmax_{Y\in\real^{n\times n}}\quad & \Big\langle Y,A-\frac{p+q}{2}J_{n \times n}\Big\rangle -\lambda\tr(Y)\label{eq:SDP_regularized_original_1}\\
\text{s.t.}\quad & Y\succcurlyeq0,
  \quad 
  0\le Y_{ij}\le1,\forall i\in[n],j\in[n]\label{eq:SDP_regularized_original_3}
\end{align}
\end{subequations}
where $\lambda$ is a regularization parameter which will be specified later. We refer to \eqref{eq:SDP_regularized_original} as the \emph{recovery SDP}. We note that this SDP and its variants, usually without the trace regularizer, are considered in many prior works; we refer to the survey \cite{li2021convex} and the references therein, as well as \cite{ames_guaranteed_2014}, \cite{chen_statistical-computational_2014}, \cite{cai_robust_2015}, and \cite{amini_semidefinite_2018}. 

Using this SDP, we can obtain the following recovery guarantee. In the sequel, we write $a_n \lesssim b_n$ or $a_n = O(b_n)$ if $a_n \le C b_n,\forall n$ for a universal numerical constant $C>0$.

\begin{thm}[Informal version of Theorem \ref{thm:clustering_with_a_gap}]
\label{thm:informal_clustering_with_gap}
    If two consecutive cluster sizes $\sb > \ss$ satisfying
\begin{align}
    \left(1-\ss/\sb\right)^{2} {(p-q)^{2}\sb^{2}}/{pn}  \gtrsim\log n \label{eq:informal_clustering_with_gap_recovery_condition}
\end{align}
then with high probability the solution to the recovery SDP~\eqref{eq:SDP_regularized_original} with a suitable $\lambda$ is of the form \begin{equation}
\label{eq:ailon_old_solution}
\Yhat = \diag \left(J_{s_1\times s_1} , \dots, J_{\sb \times \sb}, 0_{\ss \times \ss}, \dots, 0_{s_K \times s_K} \right).
\end{equation}
That is, all big clusters are exactly recovered and all small clusters are completely ignored.
\end{thm}
Like \citet{ailon2013breaking,ailon2015iterative}, this result allows for small clusters, but additionally assumes there exists a constant \emph{multiplicative gap} between the sizes of large and small clusters.

We believe the essential limitation of both Theorem \ref{thm:informal_clustering_with_gap} and \citet{ailon2013breaking,ailon2015iterative} that necessitates a gap assumption is not fundamentally algorithmic, but rather the form of their guarantees. Without assuming a gap in cluster sizes, there may exist clusters on the boundary of recoverability. In different typical samples from the same SBM instance, the recovery SDP~\eqref{eq:SDP_regularized_original} solution may or may not recover such clusters. Furthermore, the corresponding block of the SDP solution may not be all-one nor all-zero, making it impossible to guarantee an SDP solution of the form~\eqref{eq:ailon_old_solution}. This is demonstrated experimentally in Figure \ref{fig:comparison_of_recovery_algorithms_c}, where the largest cluster is recovered with all entries equal to 1, the middle cluster is recovered but with a block which has entries between zero and one, and the smaller clusters are ignored. 
Nonetheless, such a block-diagonal solution, albeit having some non-binary values, suffices to recover the large and middle clusters. In our experiments, the SDP solution always has this block-diagonal form regardless of the values of the cluster sizes. One might then hope to rigorously prove this property, which is left as an open question in~\cite{ailon2015iterative}.

\section{Our Contributions}
\label{sec:our_contributions}

\begin{figure}[tp]
\centering
\subfigure[Ground truth $\Ystar$]{\includegraphics[width=.299\textwidth]{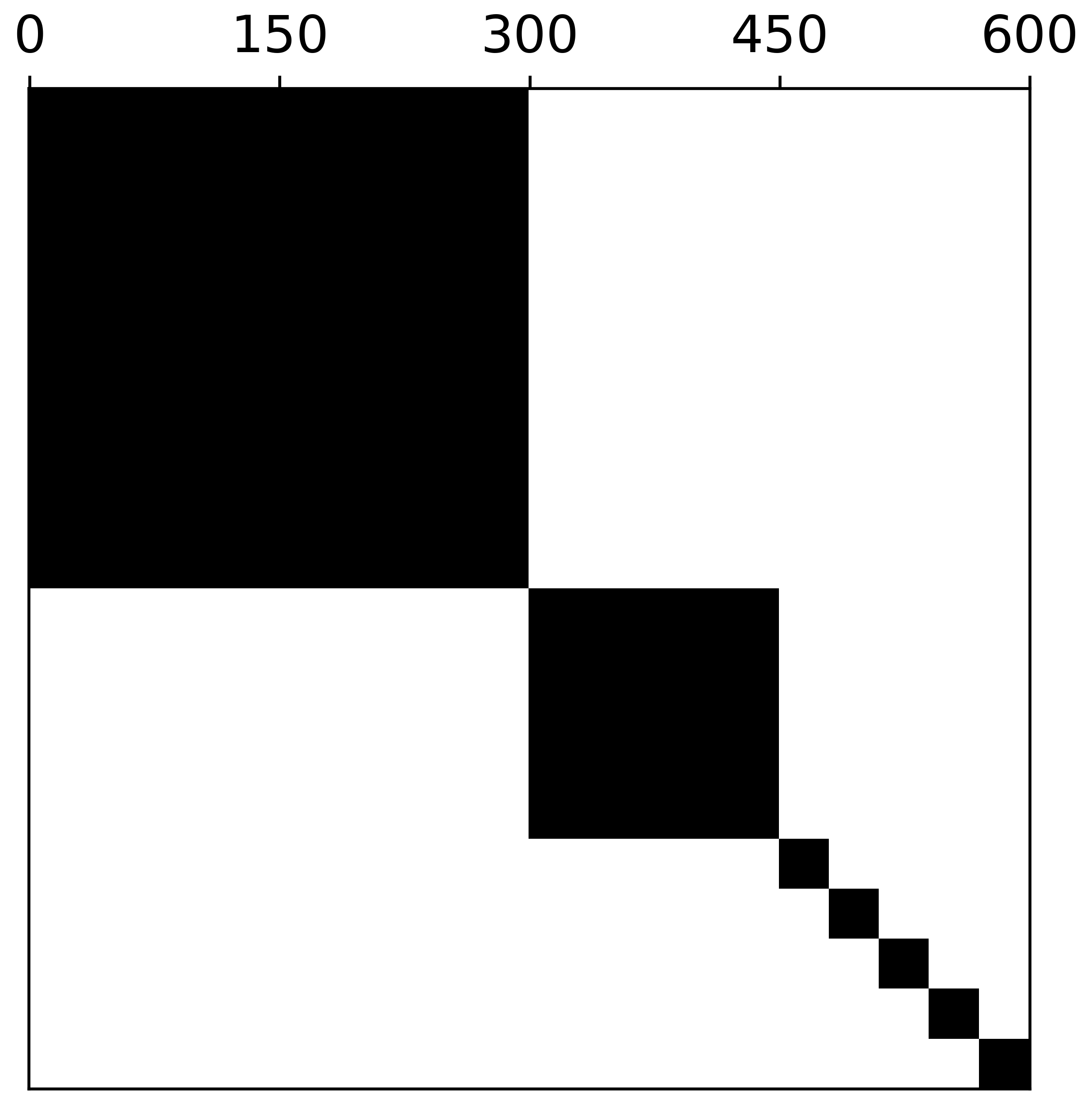}\label{fig:comparison_of_recovery_algorithms_a}}
\subfigure[$\Yhat$ with sufficient $\lambda$]{\includegraphics[width=.299\textwidth]{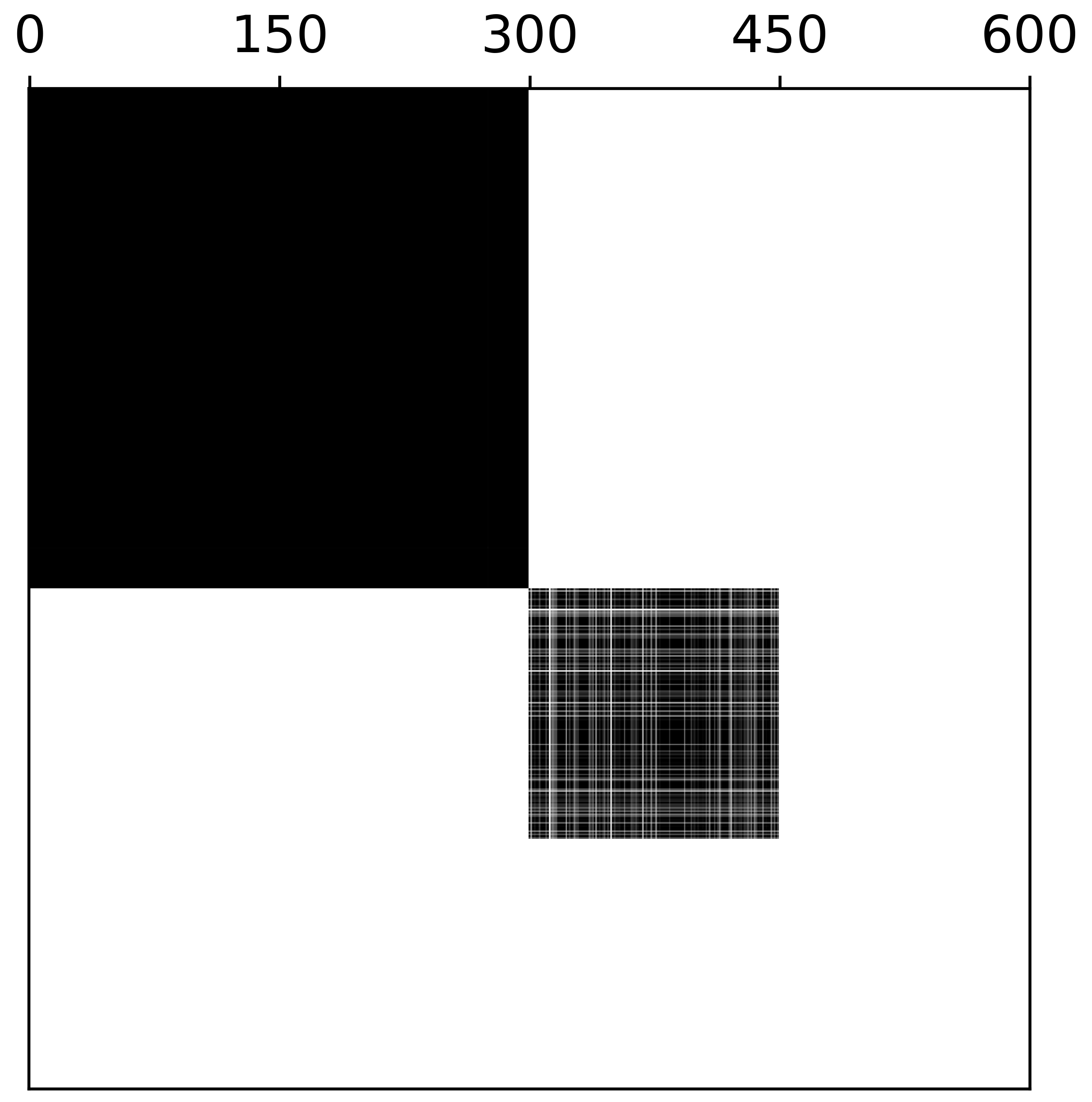}\label{fig:comparison_of_recovery_algorithms_c}}
\subfigure[$\Yhat$ with zero or small $\lambda$]{\includegraphics[width=.299\textwidth]{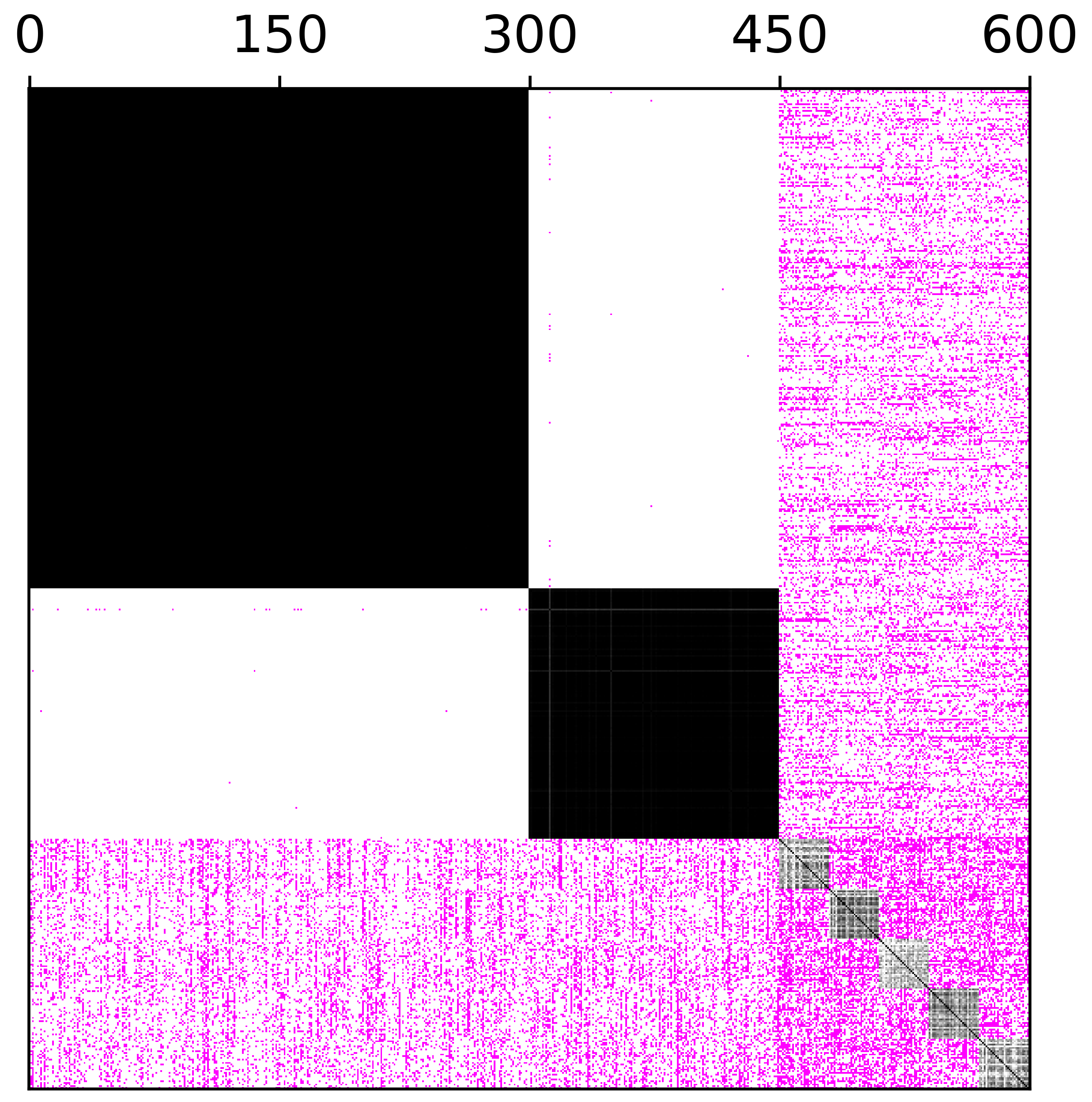}\label{fig:comparison_of_recovery_algorithms_b}}
\caption{SDP solutions and effects of regularization $\lambda$. (a): Ground truth clusters of sizes 300, 150, and 50 ($\times 5$). (b)\&(c): Solutions to SDP~\eqref{eq:SDP_regularized_original} with different $\lambda$. Nonzero off-block-diagonal entries are highlighted in pink; other  entries shown in grayscale ($\text{white}=0,\text{black}=1$). }
\label{fig:comparison_of_recovery_algorithms}
\end{figure}

We completely resolve the above question by showing the recovery SDP~\eqref{eq:SDP_regularized_original} returns a desired block-diagonal matrix regardless of the distribution of cluster sizes. Our result completely eliminates the gap assumption, guaranteeing recovery of large clusters in the presence of arbitrary middle and small clusters.
The middle, critically sized clusters greatly complicate the proof of this result. Not only do we no longer have a simple closed-form candidate solution for which we can verify optimality, but more significantly, threshold-sized clusters cause the SDP solution $\Yhat$ to be extremely sensitive to noise and thus unstable. In particular, as is seen both numerically and from our analysis, a small perturbation to the graph or cluster sizes may cause a zero block in $\Yhat$ to become near all-one, and vice versa. We overcome these challenges to give a very detailed description of the SDP solution. We believe our most important contributions are the ideas behind this analysis.

Algorithmically, the trace penalty term, which is ignored in most prior work, plays an important role in the recovery SDP~\eqref{eq:SDP_regularized_original}. This term is not needed when all clusters are large, in which case the hard, entrywise constraint in \eqref{eq:SDP_regularized_original_3} is sufficient for controlling the diagional entries of $\Yhat$. In the presence of small clusters, however, the trace penalty is essential for ensuring that the SDP solution only captures clusters with sufficient signal and ignores noise, thus yielding a block-diagonal solution. 
Figure \ref{fig:comparison_of_recovery_algorithms} provides an experimental demonstration: Figure \ref{fig:comparison_of_recovery_algorithms_b} uses an overly small $\lambda$ and leads to a solution with nonzero entries in the off-block-diagonals shown in pink, making the two largest clusters nontrivial to recover. In contrast, Figure \ref{fig:comparison_of_recovery_algorithms_c} shows a solution found with $\lambda$ set sufficiently large so that all off-block-diagonal entries are zero.

We now formally present our main result, which characterizes the solution to the recovery SDP~(\ref{eq:SDP_regularized_original}) in terms of the solutions to the following \textit{oracle SDPs}, which are $K$ smaller SDPs of the same form as the recovery SDP except that the full adjacency matrix is replaced by each of the $K$ submatrices  corresponding to only the within-cluster edges of each cluster: for $k = 1, \dots, K$, let $A^{(k)}=(A_{ij})_{i,j\in V_k} \in \real^{s_k \times s_k}$ and define the $k$th oracle SDP:
%
\begin{align}
\label{eq:SDP_regularized_oracle}
\Yhat^k=\argmax_{\substack{Y\in\real^{s_k\times s_k}\\ Y\succcurlyeq0, 0\le Y\le 1}}\quad & \Big\langle Y,A^{(k)}-\frac{p+q}{2}J_{s_k \times s_k}\Big\rangle -\lambda\tr(Y).
\end{align}



Of course the formation of the oracle SDPs requires knowledge of the ground-truth clustering $\Ystar$, so they are only usable for analysis. Now we can state our main theorem.
\begin{thm}
\label{thm:main_theorem_recovery_sdp}
The following holds for some absolute constants $B, \underline{B}, \overline{B}>2$: Fix $m$ such that $m \geq n$ and $p \geq \frac{\log m}{n}$. Set the regularization parameter $\lambda = \kappa \sqrt{pn \log m} + \delta$ for $\kappa \geq B$ and $\delta \sim \mathrm{Uniform}[0,0.1]$. With probability $1 - O(m^{-3})$, the unique solution $\Yhat$ to recovery SDP~\eqref{eq:SDP_regularized_original} is
\[
\Yhat = \diag \big( \Yhat^1 ,\dots, \Yhat^K \big),
\]
where $\Yhat^1 ,\dots, \Yhat^K$ are the unique solutions to the $K$ oracle SDPs~\eqref{eq:SDP_regularized_oracle}. Furthermore, for each $k \in \{1, \dots, K\}$, we have $\rank \Yhat^k \in \{0,1\}$, and
\begin{enumerate}[topsep=8pt, itemsep=1pt]
    \item \label{item:main_thm_all_ones_case} if $\frac{p-q}{2}s_k \geq \big(1 + \frac{1}{\overline{B}}\big) \kappa \sqrt{pn \log m} $, then $\Yhat^k = J_{s_k \times s_k} = \onevec \onevec^\top$, 
    \item if $\frac{p-q}{2}s_k \leq \big(1 - \frac{1}{\underline{B}}\big) \kappa \sqrt{pn \log m} $, then $\Yhat^{k} = 0_{s_k \times s_k} = 0 0^\top$,
    \item \label{item:main_thm_critical_case} if $\big(1 - \frac{1}{\underline{B}}\big) \kappa \sqrt{pn \log m} < \frac{p-q}{2}s_k < \big(1 + \frac{1}{\overline{B}}\big) \kappa \sqrt{pn \log m}$, then either $\Yhat^{k} = 0$ or $\Yhat^{k} = y^{k} y^{k\top}$ where $y^k \geq \frac{1}{2}$ entrywise.
\end{enumerate}
\end{thm}

Theorem \ref{thm:main_theorem_recovery_sdp} states that the recovery SDP solution is block-diagonal, and the blocks matches the oracle SDP solutions. We also precisely characterize the oracle SDP solutions, which all have rank zero or one, and furthermore are guaranteed to be all-one if the cluster size $s_k$ satisfies
\begin{align}
    (p-q)s_k  \gtrsim \kappa \sqrt{pn \log m}
    \qquad\text{or equivalently}\qquad 
    {(p-q)^2 s_k^2}/{pn} \gtrsim \log m. \label{eq:theorem_2_clustering_recovery_condition}
\end{align}
Comparing to the gap-dependent condition~\eqref{eq:informal_clustering_with_gap_recovery_condition}, there is no relative size gap term.
With respect to the critically-sized clusters, those satisfying case \ref{item:main_thm_critical_case}, their oracle solution may be zero, in which case they are not recovered, but if the oracle solution is non-zero, it is all-positive, hence recoverable. Even when the nodes are not ordered, one can easily postprocess the SDP solution $\Yhat$ to recover the nonzero diagonal blocks promised by the theorem as the off-diagonal blocks are all zero.

\begin{rem}
\label{rem:main_theorem_remark}
A few technical points about Theorem \ref{thm:main_theorem_recovery_sdp}:  Up to log factors and assuming $p/q=O(1)$, condition~\eqref{eq:theorem_2_clustering_recovery_condition} is needed in all known polynomial-time algorithms even when all clusters have the same size~\citep{abbe_community_2018,li2021convex}. 
We include an extra parameter $m$ for the failure probability to facilitate repeated application of the algorithm for recursive clustering and adaptive sampling.
The random perturbation $\delta$ is not practically necessary, but simplifies the theorem conclusions by ensuring the oracle and recovery SDPs have unique solutions; it avoids the boundary situation that $\lambda_1\big(\Ashift^{(k)} \big) = \lambda$ lies exactly on the recovery threshold. 
The number $(p+q)/2$ used in SDP~\eqref{eq:SDP_regularized_original} can be replaced by other values in the interval $(q,p)$ and thus be treated as a tuning parameter along with $\lambda$; similar observations were in prior work~\citep[e.g.,][]{CMM,amini_semidefinite_2018}.
\end{rem}

We believe our most significant contributions are the ideas behind the proof of Theorem~\ref{thm:main_theorem_recovery_sdp}. With a more extensive discussion given in the proof sketch in Section \ref{sec:proof_sketch_of_main_theorem}, here we briefly highlight some key steps. The challenges are mainly associated with the analysis of the oracle SDPs for clusters near the threshold of recovery (case \ref{item:main_thm_critical_case} of Theorem \ref{thm:main_theorem_recovery_sdp}). The crucial step is to show that the oracle SDP has a rank-one solution, which is the case as long as the non-negativity constraints in~\eqref{eq:SDP_regularized_oracle} are non-binding. Defining a relaxed oracle SDP by removing these constraints, we reduce the problem to showing that the relaxed oracle SDP solutions are weakly correlated with individual rows of the random noise, suggesting a leave-one-out-style (LOO) argument. However, the naive application of LOO ideas does not work due to the instability of clusters near the recovery threshold---the removal of one row of noise can drastically change the corresponding SDP solution. To overcome this challenge we define modified LOO versions of the oracle SDPs, using an improved eigenvalue perturbation bound which may be of independent interest. We hope that some proof techniques may be useful in more generality for problems featuring mixtures of strong and weak signals without a gap between them.

Now we discuss \cite{mukherjee2022unbalanced} and \cite{mukherjee_detecting_2022}, two recent works which also allows the presence of small clusters. Their algorithms, based on spectral clustering, also succeed without the assumption of a gap and recover clusters meeting similar to~\eqref{eq:theorem_2_clustering_recovery_condition}. In both cases, their algorithms and analysis are completely different to our SDP relaxation approach, making our work the first among a long line of convex-relaxation-based clustering methods to break the gap assumption. Beyond providing a different and arguably conceptually simpler clustering technique, our method inherits robustness properties generally unique to convex relaxation algorithms. We elaborate on this point in the following section.

\subsection{Semirandom Robustness}
\label{sec:semirandom}
To the best of our knowledge, no prior work has considered the semirandom model \citep{feige2001heuristics} in the presence of small clusters, but by virtue of our convex relaxation approach, we can tackle this challenging setting.
We first define a semirandom model for the small cluster regime. We generate an adjacency matrix $A'$ as before, but now allow a semirandom adversary to construct the observed adjacency matrix $A \in \{0,1\}^{n \times n}$ by changing $A'$ arbitrarily but restricted such that:
\begin{enumerate}[itemsep=1pt, topsep=6pt]
    \item For any $i, j \in [n]$ such that $i,j$ are in different clusters, $A_{ij} \leq A'_{ij}$.
    \item For any $i,j \in [n]$ such that $i,j$ are in the same cluster $k$, if $\frac{p-q}{2}s_k \geq \frac{3}{2}B\sqrt{p n \log n}$ then $A_{ij} \geq A'_{ij}$, otherwise $A_{ij} = A'_{ij}$.
\end{enumerate}
Note $B$ is the constant appearing in Theorem \ref{thm:main_theorem_recovery_sdp}. Thus the semirandom adversary is allowed to arbitrarily delete edges between vertices of different clusters, and allowed to arbitrarily add edges within large clusters. We refer to this setting as the \textit{large cluster semirandom model}. 

The above model includes a still challenging \emph{heterogeneous} setting as a special case: Suppose that for each $i,j \in [n]$ such that $i$ and $j$ are members of distinct clusters, there exist numbers $q_{ij}=q_{ji}$ such that $1 \geq p > q \geq q_{ij} \geq 0$, and that $A$ is sampled such that independently across all $i \leq j$,
\[ A_{ij}\sim\begin{cases}
\text{Bernoulli}(p), & \text{if $i,j$ belong to the same cluster},\\
\text{Bernoulli}(q_{ij}), & \text{if $i,j$ do not belong to the same cluster}.
\end{cases}
\]
By a coupling argument, there exists an adversary under which $A$ has this distribution \citep{chen2012sparseclustering}.

Despite seemingly easing the clustering task, the semirandom model disrupts many local structures of the SBM, foiling algorithms dependent on such structures. For instance, under the standard SBM, the expected degree of a node $i$ in cluster $k$ would satisfy the identity $\sum_{j\in[n]}\E A_{ij} = ps_k + q(n-s_k)$, in which case the observed degree of node $i$ can be used to estimate the cluster size $s_k$. \cite{mukherjee2022unbalanced} use such a subroutine for estimating the largest cluster size (and the algorithm of \cite{mukherjee_detecting_2022} also requires knowledge of the largest cluster size but they do not discuss its estimation). Under the large cluster semirandom SBM, such procedures would not directly work.
In contrast, the SDP \eqref{eq:SDP_regularized_original} and our same Theorem \ref{thm:main_theorem_recovery_sdp} continue to work:

\begin{thm}
\label{thm:semirandom_recovery}
    Under the large cluster semirandom SBM, Theorem~\ref{thm:main_theorem_recovery_sdp} holds with $\kappa = B$ and $m=n$.
\end{thm}
We prove Theorem \ref{thm:semirandom_recovery} in Section \ref{sec:semirandom_recovery_proofs}, but it follows easily from the observation that the semirandom adversary cannot change the optimal solution of the SDP \eqref{eq:SDP_regularized_original}. Similar phenomena, that proofs for the SDP relaxation approach generalize (often automatically) to heterogeneous or semi-random generative models, have been observed in the literature and recognized as a testament of the robustness of the SDP approach; see \cite{moitra2016robust,amini_semidefinite_2018} and the references therein. Although the above large cluster semirandom model excludes \textit{perturbations} to small clusters, all prior work on the semirandom model excluded the \textit{presence} of small clusters whatsoever. Thus the large cluster semirandom adversary strictly generalizes previous work, and so we believe Theorem \ref{thm:semirandom_recovery} is an important step in developing robust algorithms for unbalanced community detection.


\subsection{Recursive Clustering}
\label{sec:recursive_clustering_subsubsection}

Theorem \ref{thm:main_theorem_recovery_sdp} demonstrates that we can recover sufficiently large clusters even in the presence of arbitrarily-sized smaller clusters. As first recognized in \cite{ailon2015iterative}, this allows for the use of a recursive clustering procedure: whenever we recover some clusters, we can remove the corresponding nodes from the graph, thereby eliminating some noisy entries and lowering the detection threshold for another application of the recovery procedure. Specifically in our case, recovering more nodes enables us to lower the regularization parameter $\lambda$. We can thus repeatedly apply our clustering algorithm to uncover smaller and smaller clusters, until no more can be found. We emphasize that this recursive clustering procedure operates using only one sample $A$ from the SBM.

Here we informally discuss our recursive clustering results, and defer detailed presentation and comparison with \cite{ailon2015iterative} to Subsection \ref{sec:more_recursive_clustering}. \cite{ailon2015iterative} consider a setting where $n \to \infty$ and $p,q$ are constant. They recover all but $o(n)$ nodes via a recursive clustering procedure, but since their clustering subroutine requires a multiplicative gap between cluster sizes, they can only accommodate a total number of clusters $K \lesssim \log n$ to ensure that a sufficient gap remains after each round. By contrast, our gap free result recovers $n - o(n)$ nodes even if $K$ is polynomially large in $n$:

\begin{thm}[Informal version of Theorem \ref{thm:recursive_clsutering}]
    Suppose that $K \lesssim \frac{n^{0.5 - \alpha}}{\sqrt{\log n}}$ for some $\alpha < 0.5$. Then with high probability, our Algorithm \ref{alg:recursive_clustering} recovers all but $O(n^{1-2 \alpha})$ nodes.
\end{thm}

\subsection{Clustering With a Faulty Oracle}
\label{sec:clustering_with_faulty_oracle_subsubsection}
The problem of clustering with a faulty oracle was introduced by \cite{mazumdar_clustering_2017} and has since received significant attention. It can be viewed as an adaptive sampling version of the SBM, where we try to recover clusters by querying $o(n^2)$ individual entries of the adjacency matrix. Formally, we assume there exist $n$ vertices partitioned into $K$ unknown clusters $V_1, \dots, V_K$ of different sizes. For some bias parameter $\delta > 0$, we assume there exists a faulty clustering oracle $\mathcal{O}$ which, upon being queried with a pair of vertices $(i,j)$, returns an answer distributed as 
\[
\mathcal{O}(i,j)\sim\begin{cases}
\text{Bernoulli}(1/2 + \delta/2), & \text{if $i,j$ belong to the same cluster},\\
\text{Bernoulli}(1/2 - \delta/2), & \text{if $i,j$ do not belong to the same cluster}.
\end{cases}
\]
Answers are independent for different pairs, and repeating the same query produces the same answer.
Clustering with a faulty oracle has several important applications, such as crowdsourced entity resolution and edge prediction in social networks. We refer to \cite{mazumdar_clustering_2017} for more discussion, and refer to \cite{peng_towards_2021} and \cite{mukherjee2022unbalanced} for more in-depth comparison of prior work. We summarize existing algorithms which are computationally efficient and work for general $K$ in Table \ref{table:faulty_oracle_clustering_algorithms}, and defer more comments to Subsection \ref{sec:more_faulty_oracle_clustering}.

\begin{table}[t]
\centering
\begin{tabular}{|c|c|c|}
\hline
\multicolumn{1}{|c|}{Query Complexity} & Threshold & \multicolumn{1}{|c|}{Reference} \\ \hline
$O\Big( \frac{nK \log n}{\delta^2}$
$ + \frac{nK^2 \log n}{\delta^4} \wedge \frac{K^5\log^2 n}{\delta^8}\Big)$
& $\Omega \!\left(\frac{K^4 \log n}{\delta^2}\right)$ & \cite{mazumdar_clustering_2017}  \\ \hline
$O\!\left(\frac{nK \log n}{\delta^2}  + \frac{K^4\log^2 n}{\delta^4}\right)$ & $\Omega \!\left(\frac{n}{K} \right)$ & \cite{peng_towards_2021}  \\ \hline
$O\!\left(\frac{nK \log n}{\delta^2}  + \frac{K^{10}\log^2 n}{\delta^4}\right)$ & $\Omega \!\left(\frac{K^4 \log n}{\delta^2}\right)$ & \cite{peng_towards_2021}   \\ \hline
$O\!\left(\frac{n(K+ \log n)}{\delta^2}  + \frac{K^8\log^3 n}{\delta^4}\right)$ & $\Omega \!\left(\frac{K^4 \log n}{\delta^2}\right)$ & \cite{xia_optimal_2022}   \\ \hline
$O\!\left(\frac{nK \log n}{\delta^2}  \!+\! \frac{K^{9}\log \!K\log^2 n}{\delta^{12}}\right)$ & $\Omega \!\left(\frac{K^4 \log n}{\delta^6}\right)$ & \cite{pia_clustering_2022}  \\ \hline
$O\!\left(\frac{nK \log n}{\delta^2}  + \frac{K^5\log^2 n}{\delta^4}\right)$ & $\Omega \!\left(\frac{K^2 \log n}{\delta^2}\right)$ & Our Thm.~\ref{thm:faulty_oracle_clustering_improved_complexity_small_k}   \\ \hline
$O\!\left(\frac{n}{s}\frac{n\log n}{\delta^2} + \frac{n^4}{s^4}\frac{\log^2 n}{\delta^4} \right)$\footnotemark & $s$ & \cite{mukherjee2022unbalanced},
Our Thm.~\ref{thm:faulty_oracle_clustering_size_param}  \\ \hline
\end{tabular}
\caption{Comparison of algorithms for clustering with a faulty oracle which are computationally efficient and work for general $K$.}
\label{table:faulty_oracle_clustering_algorithms}
\end{table}

As identified by \cite{mukherjee2022unbalanced}, all prior algorithms (except that of \cite{mukherjee2022unbalanced}) fail when the number of clusters $K$ is large. In particular,  $K = \Omega(n^{1/4})$ causes all algorithms in the first 5 rows of Table \ref{table:faulty_oracle_clustering_algorithms} to have sample complexity $\Omega(n^2)$ and  a recovery threshold size of $\Omega(n)$, with smaller $K$ still presenting issues for some algorithms. 
With a gap-free clustering subroutine, we provide an algorithm with a different type of guarantee which circumvents the above issues under large $K$. Letting $s$ be an input parameter, our Algorithm \ref{alg:clustering_with_faulty_oracle_meta_algorithm} can recover all clusters of size at least $s$, and for an appropriate choice of $s$ this leads to an $o(n^2)$ sample complexity:

\begin{thm}
\label{thm:faulty_oracle_clustering_size_param}
    There exist absolute constants $C_1, C_2$ such that for any parameter $s \geq C_2\frac{\sqrt{n \log n}}{\delta}$, by setting $\gamma = C_2 \frac{n \log n}{s^2 \delta^2}$, with probability at least $1- O(n^{-3})$, Algorithm \ref{alg:clustering_with_faulty_oracle_meta_algorithm} recovers all clusters of size at least $s$ with query complexity $O\left(\frac{n^2}{s}\frac{\log n}{\delta^2} + \frac{n^4}{s^4}\frac{\log^2 n}{\delta^4} \right)$.
\end{thm}
We formally define Algorithm \ref{alg:clustering_with_faulty_oracle_meta_algorithm} and give further discussion in Section \ref{sec:more_faulty_oracle_clustering}. All proofs for this section are provided in Section \ref{sec:clustering_with_a_faulty_oracle_proofs}. Note the constant $C_1$ is used within the algorithm. Algorithm \ref{alg:clustering_with_faulty_oracle_meta_algorithm} requires a user-specified target cluster size $s$, but by using an adaptive resampling procedure, we can obtain the following instance-dependent query complexity without any prior knowledge of cluster sizes. Recall that $s_1$ is the largest cluster size.

\begin{thm}
 \label{thm:clustering_with_faulty_oracle_adaptive}
There exist absolute constants $C_1, C_2, C_3$ such that as long as $s_1 \geq C_2\frac{\sqrt{n \log n}}{\delta}$, with probability at least $1- O(n^{-2})$, Algorithm \ref{alg:clustering_with_faulty_oracle_adaptive} recovers a cluster of size at least $\frac{s_1}{C_3}$ with sample complexity $O\big(n\frac{\log n}{\delta^2} + \frac{n^4}{s_1^4}\frac{\log^2 n}{\delta^4} \big)$.
\end{thm}

We present Algorithm \ref{alg:clustering_with_faulty_oracle_adaptive} in Section \ref{sec:more_faulty_oracle_clustering}. The strategy of obtaining instance-dependent query complexity using adadptive sampling has a crucial reliance on the gap-free recovery guarantee of Theorem \ref{thm:main_theorem_recovery_sdp}, since our objective is to sample minimally so that only the largest cluster is recoverable, meaning the recovery procedure must not be sensitive to the presence of slightly smaller but still large clusters. We therefore hope that Theorem \ref{thm:main_theorem_recovery_sdp} lays the foundation for future instance-dependent algorithms for clustering with a faulty oracle, particularly in unbalanced cluster size settings.

Even if the number of clusters $K$ is small, by using our Theorem \ref{thm:main_theorem_recovery_sdp} instead of gap-requiring clustering subroutines, we  obtain the following result which improves previous query complexities and recovery thresholds, while also having a substantially simpler algorithm and proof. The Algorithm~\ref{alg:clustering_with_faulty_oracle_small_K}, along with further discussion of this result, is presented in Section \ref{sec:more_faulty_oracle_clustering}.

\begin{thm}
\label{thm:faulty_oracle_clustering_improved_complexity_small_k}
    With probability at least $1-O(n^{-2})$, Algorithm \ref{alg:clustering_with_faulty_oracle_small_K} recovers all clusters of size at least $\Omega\left(\frac{K^2\log n}{\delta^2}\right)$ with query complexity $O\left(\frac{nK \log n}{\delta^2}  + \frac{K^5\log^2 n}{\delta^4}\right)$.
\end{thm}

\subsection{Eigenvalue Perturbation Bounds}
\label{sec:eigenvalue_perturbation_simplified}

As elaborated in the proof sketch in Section \ref{sec:proof_sketch_of_main_theorem}, to overcome the instability of threshold-size cluster oracle SDP solutions, we modify standard leave-one-out techniques by also changing the regularization strength. This requires eigenvalue perturbation bounds which improve upon existing results and may be of independent interest. Here we compare our results to existing work \cite{eldridge2018unperturbed} in a simplified setting, and present the full versions, along with more background and detailed comparison in our clustering context, in Section \ref{sec:eigenvalue_perturbation_bounds_subsubsection}. Let $M, H \in \R^{n \times n}$ be symmetric matrices, viewing $H$ as a perturbation, and let $v_1$ be a top eigenvector of $M$. The classical Weyl's inequality gives that $\left|\lambda_1(M+H)-\lambda_1(M)\right| \leq \opnorm{H}.$ As argued by \cite{eldridge2018unperturbed}, one can sometimes obtain a bound of $|v_1^\top H v_1|$, which can be significantly smaller than $\opnorm{H}$ when $H$ is a random perturbation. Assuming an eigengap $\lambda_1(M) - \lambda_2(M) \geq 2\opnorm{H}$, \cite{eldridge2018unperturbed} show
\begin{align}
    \lambda_1(M+H) \leq \lambda_1(M) + v_1^\top H v_1 + \frac{2 \opnorm{H}^2}{\lambda_1(M) - \lambda_2(M)} \label{eq:eldridge_specialized}
\end{align}
Under this same assumption, our Theorem \ref{thm:general_eval_pert_bound} improves this to
\begin{align}
    \lambda_1(M+H) \leq \lambda_1(M) + v_1^\top H v_1 + \frac{2 \twonorm{Hv_1}^2}{\lambda_1(M) - \lambda_2(M)}. \label{eq:ours_specialized}
\end{align}
When $H$ is random, we often have $\twonorm{Hv_1} \ll \opnorm{H}$, making~\eqref{eq:ours_specialized} a substantial improvement.

\paragraph{Notation}
\label{sec:notation}

For an $n \times n$ matrix $M$, $M^{(ij)}$ is the $s_i \times s_j$ submatrix obtained by deleting all rows and columns not associated with cluster $i$ and all columns not associated with cluster $j$. Let $M^{(i)} = M^{(ii)}.$
For a vector $v \in \R^n$,  $v^{(i)}\in \R^{s_i}$ is obtained by deleting all entries not associated with cluster $i$. We use $M_{ij}$ and $v_i$ to denote the entries of $M$ and $v$. 
The convention is that $v^{(i)}_j=(v^{(i)})_j$.
For two matrices
$X$ and $Y$,  $\left\langle X,Y\right\rangle :=\sum_{ij}X_{ij}Y_{ij}$ denotes their inner product and $X\circ Y$  their element-wise product.
$\opnorm X$ is the largest singular value of $X$; $\infnorm X:=\max_{i,j}\left|X_{ij}\right|$. 
Denote by  $\diag\left( M_1, \dots, M_m\right)$ a block-diagonal matrix with diagonal blocks $M_1, \dots, M_m$.
We use $\onevec_m$ to denote a length-$m$ vector of all ones, occasionally dropping subscripts.  



\section{Proof Outline of Main Theorem}
\label{sec:proof_sketch_of_main_theorem}

In this section, we sketch proof of Theorem \ref{thm:main_theorem_recovery_sdp}. The general strategy is a primal-dual witness approach: we construct a candidate primal solution to the recovery SDP~\eqref{eq:SDP_regularized_original} as well as dual variables which certify its optimality. The construction relies on showing that the oracle SDPs~\eqref{eq:SDP_regularized_oracle} have rank-one and ``well-spread'' solutions, which is most challenging for the mid-size oracle blocks. As detailed below, we start by observing the rank-one property holds if the lower-bound constraints $\Yhat^k \geq 0$ are non-binding, which is in turn closely related to controlling the correlation between the oracle SDP solution and each row of the oracle noise matrix $W^{(k)}$.
This suggests a leave-one-out-style argument, where the correlation strength is quantified by the change of the solution when a row of $W^{(k)}$ is left out. However, due to the arbitrary proximity of the mid-size clusters to the recovery boundary, the standard argument of zeroing one row of noise (see, e.g., \cite{zhong_near-optimal_2018}) may have a drastically different solution. To remedy this issue, we also modify the regularization strength, and the correct modification requires improved eigenvalue perturbation bounds in order to ensure that the leave-one-out solution and the oracle SDP solution are suitably close.

In this sketch we simply set $m=n$.
Our regularization parameter $\lambda$ will be at least $\kappa \sqrt{pn \log n}$ for some $\kappa$ which is a fixed absolute constant which we can make arbitrarily large. We will write $\varepsilon$ in place of terms which are $\asymp \frac{1}{\kappa}$, which may be multiple (unequal) terms in the same expression.

\paragraph{Setting up primal-dual witness argument}

We seek to show the solution is block-diagonal with support contained within that of the ground truth $\Ystar$, but we do not have closed forms for each block. This motivates us to consider the $K$ oracle SDPs~\eqref{eq:SDP_regularized_oracle} associated with each cluster and combine their respective solutions $\Yhat^1, \dots, \Yhat^K$ to form our candidate solution $\Yhat = \diag(\Yhat^1, \dots, \Yhat^K )$.

To show optimality of $\Yhat$ to the recovery SDP, it suffices to find nonnegative dual variables $U, L \in \R^{n \times n}$ satisfying the KKT conditions:
\begin{equation}
\label{eq:sketch_recovery_KKT}
    U \circ \big( \Yhat - J_{n \times n} \big) = L \circ \Yhat = 0, 
    \quad
    \Ashift - \lambda I - U + L \preccurlyeq 0, 
    \quad 
    \big( \Ashift - \lambda I - U + L \big) \Yhat = 0.
\end{equation}
On the other hand, each oracle SDP solution $\Yhat^k$ must satisfy its own similar KKT conditions: there exist nonnegative dual variables $U^k, L^k \in \R^{s_k \times s_k}$ which, together with the oracle block $\Ashift^{(k)}$, satisfy similar equations as~\eqref{eq:sketch_recovery_KKT}.
To find a solution for the recovery conditions~\eqref{eq:sketch_recovery_KKT}, we can set $U = \diag(U^1, \dots, U^k)$, and it remains to choose the non-block-diagonal entries of $L$. This choice cannot be made without first gaining more detailed information about the oracle solutions $\Yhat^k$, particularly that they are rank-one and ``well-spread.'' The reasons will be made more clear once we fully construct $L$, but for now we can motivate these requirements with the following observations: Suppose for simplicity that $K=2$. 
The second equation in the oracle SDP version of~\eqref{eq:sketch_recovery_KKT} guarantees that each unit norm eigenvector $v \in \R^{s_1}$ of $\Yhat^1$ will satisfy $( \Ashift^{(1)} - U^1 + L^1 )v = \lambda v$. Now letting $\overline{v} \in \R^n$ be a zero-padded version of $v$ such that $ \overline{v}^{(1)} = v$, we will have
\begin{align}
    \Big(\Ashift - U + L\Big) \overline{v} = \begin{bmatrix}
        \Ashift^{(1)} -U^1 + L^1 & \Ashift^{(12)} + L^{(12)} \\
        \Ashift^{(21)} + L^{(21)} & \Ashift^{(2)} -U^2 + L^2
    \end{bmatrix} \begin{bmatrix}
        v \\ 0
    \end{bmatrix} = \begin{bmatrix}
        \lambda v \\ \big(\Ashift^{(21)} + L^{(21)}\big)v
    \end{bmatrix}. \label{eq:sketch_eigenvector_justification}
\end{align} 
The right hand side has norm at least $\lambda$, so the only way to satisfy the last equation in~\eqref{eq:sketch_recovery_KKT}
is for $\big(\Ashift^{(21)} + L^{(21)}\big)v = 0$. Therefore, each eigenvector (corresponding to a nonzero eigenvalue) of $\Yhat^1$ constrains $L^{(21)}$, and in order for the entries of $\Ashift^{(21)} v = \left(-\frac{p-q}{2}J_{s_2 \times s_1} + W^{(21)} \right)v$ to concentrate around $-\frac{p-q}{2}$ so that we may choose $L^{(21)}\geq0$, we will need $\infnorm{v}$ to be sufficiently small.

\paragraph{Analyzing oracle SDPs}
Now we can discuss how to establish these facts about the oracle SDPs. Fix a cluster $k$. If the signal $\frac{p-q}{2}s_k$ is sufficiently small or large, everything is relatively simple. If $\frac{p-q}{2}s_k \leq (1-\varepsilon) \lambda$, the top eigenvalue of $\Ashift - \lambda I$ will be negative, so $\Yhat^k$ will be zero. If $\frac{p-q}{2}s_k \geq (1+\varepsilon) \lambda$, we can apply the gap-dependent recovery Theorem \ref{thm:clustering_with_a_gap} to guarantee that $\Yhat^k = J_{s_k \times s_k}$ (since, with only one cluster, we need not worry about smaller clusters being sufficiently small).
This leaves the most interesting mid-size case, where $\frac{p-q}{2}s_k$ is within a small multiple of $\lambda$. 

\paragraph{Showing relaxed oracle SDP has rank-one solution}
In this case, there is no closed-form candidate for $\Yhat^k$ independent of the realization of $A$; entries may be neither 0 nor 1. Experimentally, when $\lambda$ is sufficiently large, the oracle solutions have no zero entries (unless the entire block is zero), meaning the non-negativity constraints are non-binding. We consider the oracle SDP~\eqref{eq:SDP_regularized_oracle} with these constraints removed, and refer to the resulting program the \emph{Relaxed Oracle SDP}.
A key simplification in the relaxed oracle SDP is that the last two conditions in the oracle SDP version of~\eqref{eq:sketch_recovery_KKT}
become
\begin{equation}
    \Ashift^{(k)} - \lambda I - \diag(u^k) \preccurlyeq 0_{s_k \times s_k}, \qquad
    \big( \Ashift^{(k)} - \lambda I - \diag(u^k)  \big) \Yhats^k = 0_{s_k \times s_k}. \label{eq:sketch_relaxed_oracle_KKT}
\end{equation}
Since $\Ashift^{(k)} = \frac{p-q}{2}J + W^{(k)}$, that is a noisy perturbation of a rank-one matrix, and also since $\diag(\overline{u}^k) \succcurlyeq 0$, if $\lambda > \opnorm{W^{(k)}}$ then we will be guaranteed that $\Ashift^{(k)} - \lambda I - \diag(\overline{u}^k)$ has all eigenvalues strictly negative except for at most one (which by first equation in~\eqref{eq:sketch_relaxed_oracle_KKT} must then be equal to zero). Thus we can combine this with the fact that $\Yhats^k \succcurlyeq 0_{s_k \times s_k}$ and second equation in~\eqref{eq:sketch_relaxed_oracle_KKT} to conclude that $\Yhats^k$ must have rank zero or one. Thus we can write $\Yhats^k = yy^\top$ (choosing the sign of $y$ so that $\onevec^\top y \geq 0$). We note that this step is somewhat related to \cite{sagnol_class_2011} which shows a similarly structured SDP has low-rank solution, but \cite{sagnol_class_2011} assumes that the objective matrix is low-rank, while we instead use the fact that it has a low number of positive eigenvalues.

\paragraph{Setting up leave-one-out technique to control noise correlation}
To justify using the relaxed oracle SDP, we must show that $\Yhats^k$ or equivalently $y$ are elementwise-non-negative, in which case the relaxed oracle SDP solution will also be feasible for the oracle SDP, and therefore the optimal solution for the oracle SDP will be $\Yhats^k = yy^\top$. This is easy to check when $\lambda_1 (\Ashift^{(k)} ) < \lambda$, since then $\Yhats^k = 0_{s_k \times s_k}$, so we can focus on the case where $\lambda_1 (\Ashift^{(k)} ) > \lambda$ and $\Yhats^k$ is nonzero. Now we argue why this reduces to bounding the correlation between $y$ and the rows of the noise matrix $W^{(k)}$.

Since the argument is identical for all entries, we focus on showing that $y_1 \geq 0$. If $y_1 = 1$ then we are done, and otherwise either $y_1 \in (-1,1)$ so $u_1^k = 0$, or $y_1 = -1$ and then $- u^k_1y_1 \geq 0$. By using the fact that $\Yhats^k$ is rank-one and rearranging the optimality condition~\eqref{eq:sketch_relaxed_oracle_KKT}, we obtain
$y_1 \ge \frac{1}{\lambda} \left(\frac{p-q}{2}\onevec^\top y + w^\top y \right),$
where $w^\top$ is the first row of $W^{(k)}$. 
Since in a mid-size cluster $(1-\varepsilon)\lambda \leq \frac{p-q}{2}s_k \leq (1+\varepsilon)\lambda$, and it is straightforward to show $\onevec^\top y \geq s(1-\varepsilon)$, the first term on the RHS is close to one. Therefore it suffices to show that $w$ and $y$ are uncorrelated, specifically that $\left| w^\top y \right| \lesssim \varepsilon \lambda$.

The entries of $w$ are bounded by $1$ and each have variance $\leq p$, while $y$ has $\twonorm{y} \leq \sqrt{s_k}$ and $\infnorm{y} \leq 1$, so if they were independent we could apply Bernstein's inequality to conclude $\left| w^\top y \right| \lesssim \sqrt{ps_k \log n} \leq \varepsilon \lambda$. The challenge is that they are not independent. Furthermore, naively bounding $\left|w^\top y\right| \leq \twonorm{w}\twonorm{y} \lesssim (\sqrt{ps_k})(\sqrt{s_k})$ does not work, because this bound can be much larger than $\lambda = \kappa \sqrt{ pn \log n}$. Still, the guiding intuition is that since $y$ needs to depend on all rows of $W^{(k)}$, it should not be strongly correlated with any single row. This is the standard motivation for the leave-one-out technique: to formalize this intuition, we can form a LOO relaxed oracle SDP where we set the noise in the first row and column of $\Ashift^{(k)}$ to zero, and let its solution (which is also rank-one) be $\ytilde \ytilde^\top$. Then $\ytilde$ will be independent of $w$, but we might hope it will also be very close to $y$, two facts which we could combine to show $|w^\top y| \approx |w^\top \ytilde| \lesssim \sqrt{ps_k \log n}$.

Due to the extreme noise sensitivity inherent to the mid-size clusters, this strategy will not work: after leaving out noise to get $\Atilde$, we may have $\lambda_1 ( \Atilde ) < \lambda$ (even though $\lambda_1 ( \Ashift^{(k)} ) > \lambda$), which means we will have solution $\ytilde \ytilde^\top = 0_{s_k \times s_k}$ for the LOO relaxed oracle SDP. In simpler terms, the usual LOO solution $\ytilde$ might be very different from $y$. To overcome this issue, we will \emph{slightly} reduce the regularization parameter $\lambda$ for the LOO problem, so that $\lambda_1 ( \Atilde ) > \lambdatilde$. Of course, if $\lambdatilde$ is too small then $\ytilde$ might be significantly different from $y$. Existing eigenvalue perturbation bounds from \cite[Theorem 6]{eldridge2018unperturbed} are too large to use in our remaining arguments, but our Theorem \ref{thm:general_eval_pert_bound} gives a much smaller bound which ensures $\lambdatilde$ is sufficiently close to $\lambda$. For the concrete comparison between these two bounds in our setting, see Section \ref{sec:eigenvalue_perturbation_bounds_subsubsection}.

\paragraph{Showing leave-one-out solution is close to relaxed oracle SDP solution}
With this resolved, we write $y = \alpha \ytilde + \beta t$ where $t $ is a unit vector orthogonal to $\ytilde$. It is easy to show $|\alpha - 1| \lesssim \varepsilon$. We use this decomposition is because there is a natural method to bound the norm of the rejection $\beta = \twonorm{y - \proj_{\ytilde}y}$. One interpretation of Lagrangian relaxation is that the optimal dual variables provide ``bonuses''  subtracted from the objective values of points which satisfy constraints, while preserving the optimal solution and the optimal value. Forming the Lagrangian for the LOO relaxed oracle SDP (after converting to a minimization problem for consistency with optimization conventions), the objective becomes
$
\mathcal{L}\big(Y, \utilde, \widetilde{P}\big) 
= -\big \langle \Atilde^{(k)} - \lambdatilde I, Y \big \rangle +\left \langle \diag(\utilde), Y - I \right \rangle + \big\langle -\widetilde{P}, Y\big \rangle,
$
where $\widetilde{P} = -(\Atilde - \lambdatilde I - \diag(\utilde))$ is the dual variable satisfying the KKT conditions corresponding to the PSD constraint $Y \succcurlyeq 0$ (which does not appear in other arguments because we eliminate it using the stationarity condition). 

Now using optimality of $\ytilde\ytilde^\top$ and feasibility of $yy^\top$, we can derive
$\langle \widetilde{P}, yy^\top \rangle 
    \leq
    - \langle \Atilde - \lambdatilde I, yy^\top - \ytilde\ytilde^\top \rangle.$
Similar to our earlier arguments, we can show that $\widetilde{P}\ytilde = 0$, while all other eigenvalues of $\widetilde{P}$ are positive and at least $(1-\varepsilon)\lambda$. Thus the bonus term directly measures $y - \proj_{\ytilde}y$ and is at least $(1-\varepsilon)\lambda \beta^2$. We would like the RHS 
to involve the difference between $\Atilde$ and $\Ashift^{(k)}$ (which we know is small), so we can add the term $- \langle \Ashift^{(k)} - \lambda I, \ytilde\ytilde^\top - yy^\top \rangle$ which is non-negative because it is the suboptimality of $\ytilde\ytilde^\top$ for the relaxed oracle SDP. This yields the final perturbation bound
\begin{align}
    \big\langle \widetilde{P}, yy^\top\big \rangle 
    \leq 
    \big \langle \big(\Atilde - \lambdatilde I\big) - \big(\Ashift^{(k)} - \lambda I\big), \ytilde\ytilde^\top - yy^\top \big \rangle. 
    \label{eq:sketch_pert_bound}
\end{align}
By bounding the RHS of~\eqref{eq:sketch_pert_bound} in terms of $\beta$ and combining with our earlier arguments, we obtain a relationship which implies that $\beta \lesssim \varepsilon$. We note that this step is the key location where we must use the improved eigenvalue perturbation bound on $| \lambda - \lambdatilde |$. Because this perturbation bound~\eqref{eq:sketch_pert_bound} is essential for controlling the difference between $y$ and $\ytilde$, we hope that our explanation of the simple underlying principles could be useful for other settings involving leave-one-out analysis.

Wrapping up, while we could use our bounds on $\alpha$ and $\beta$ to show $|w^\top y | \lesssim \varepsilon \lambda$, we simply use the fact $\ytilde_1 \geq 1 - \varepsilon$ since there is no noise in the first row, so
$
y_1 \geq \alpha \ytilde_1 - \left|\beta t_1\right| \geq (1-\varepsilon) (1-\varepsilon) - \varepsilon \geq 1- \varepsilon.
$
This proves that the oracle SDP solution $\Yhat^k$ is rank-one (when nonzero) for mid-size clusters. 

\paragraph{Constructing dual variables to complete primal-dual witness argument}
Now we can complete the proof by constructing $L$ which satisfies the KKT conditions of the recovery SDP~\eqref{eq:sketch_recovery_KKT} as outlined in the beginning of this proof sketch.
We discuss one final technical consideration. Consider again for simplicity the case that there are $K=2$ clusters. Above, we argued using~\eqref{eq:sketch_eigenvector_justification} why it was essential for the non-block-diagonal entries of $L$ to be set in such a way that whenever $yy^\top$ is a non-zero solution for the first oracle SDP, then the zero-padded vector $\overline{y} / \twonorm{y} \in \R^n$ (satisfying $\overline{y}^{(k)} = y$) is an eigenvector of $\Ashift + L - U$. Even when the first oracle SDP is zero and we have $U^1 = L^1 = 0$, if the first cluster is near the recovery threshold, it is possible to have $\lambda_1\big(\Ashift^{(k)} \big) = (1-\varepsilon)\lambda$. Now letting $v\in \R^{s_1}$ be the top eigenvector of $\Ashift^{(1)}-U^1 + L^1 = \Ashift^{(1)}$ (and $\overline{v} \in \R^n$ its zero-padded version), we have a nearly identical situation to that in~\eqref{eq:sketch_eigenvector_justification} because
$
    \big(\Ashift - U + L\big) \overline{v} 
    = \begin{bmatrix}
        (1-\varepsilon)\lambda v & \big(\Ashift^{(21)} + L^{(21)}\big)v
    \end{bmatrix}^\top.
$
The vector on the right hand side
has norm at least close to $\lambda$, which again places a constraint on our choice of $L^{(21)}$ as we will need to ensure that $\|(\Ashift^{(21)} + L^{(21)})v\|_2$ is very small. For this reason, we will treat clusters which are slightly below the recovery threshold identically to how we treat clusters which have nonzero oracle solutions (and thus are recovered), except we must use the top eigenvector of $\Ashift^{(k)}$ instead of the $y$ from the $k$th oracle solution $yy^\top$ (which, by the oracle SDP version of~\eqref{eq:sketch_recovery_KKT}, is a top eigenvector of $\Ashift^{(k)}-U^k + L^k$). 

\section*{Acknowledgement}
Y.\ Chen and M.\ Zurek were supported in part by National Science Foundation grants CCF-2233152 and DMS-2023239. We thank Lijun Ding and Yiqiao Zhong for inspiring discussion.

\bibliographystyle{plainnat}
\bibliography{refs}

\begin{thebibliography}{57}
\providecommand{\natexlab}[1]{#1}
\providecommand{\url}[1]{\texttt{#1}}
\expandafter\ifx\csname urlstyle\endcsname\relax
  \providecommand{\doi}[1]{doi: #1}\else
  \providecommand{\doi}{doi: \begingroup \urlstyle{rm}\Url}\fi

\bibitem[Abbe(2018)]{abbe_community_2018}
Emmanuel Abbe.
\newblock Community {Detection} and {Stochastic} {Block} {Models}: {Recent}
  {Developments}.
\newblock \emph{Journal of Machine Learning Research}, 18\penalty0
  (177):\penalty0 1--86, 2018.
\newblock URL \url{http://jmlr.org/papers/v18/16-480.html}.

\bibitem[Abbe and Sandon(2015)]{abbe2015community}
Emmanuel Abbe and Colin Sandon.
\newblock Community detection in general stochastic block models: Fundamental
  limits and efficient algorithms for recovery.
\newblock In \emph{2015 IEEE 56th Annual Symposium on Foundations of Computer
  Science}, pages 670--688. IEEE, 2015.

\bibitem[Abbe et~al.(2015)Abbe, Bandeira, and Hall]{abbe2015exact}
Emmanuel Abbe, Afonso~S Bandeira, and Georgina Hall.
\newblock Exact recovery in the stochastic block model.
\newblock \emph{IEEE Transactions on information theory}, 62\penalty0
  (1):\penalty0 471--487, 2015.

\bibitem[Abbe et~al.(2017)Abbe, Fan, Wang, and Zhong]{abbe2017entrywise}
Emmanuel Abbe, Jianqing Fan, Kaizheng Wang, and Yiqiao Zhong.
\newblock Entrywise eigenvector analysis of random matrices with low expected
  rank.
\newblock \emph{arXiv preprint arXiv:1709.09565}, 2017.

\bibitem[Abbe et~al.(2020)Abbe, Fan, Wang, and Zhong]{abbe_entrywise_2020}
Emmanuel Abbe, Jianqing Fan, Kaizheng Wang, and Yiqiao Zhong.
\newblock Entrywise eigenvector analysis of random matrices with low expected
  rank.
\newblock \emph{The Annals of Statistics}, 48\penalty0 (3):\penalty0
  1452--1474, June 2020.
\newblock ISSN 0090-5364, 2168-8966.
\newblock \doi{10.1214/19-AOS1854}.
\newblock Publisher: Institute of Mathematical Statistics.

\bibitem[Agarwal et~al.(2017)Agarwal, Bandeira, Koiliaris, and
  Kolla]{agarwal2017multisection}
Naman Agarwal, Afonso~S Bandeira, Konstantinos Koiliaris, and Alexandra Kolla.
\newblock Multisection in the stochastic block model using semidefinite
  programming.
\newblock In \emph{Compressed Sensing and its Applications: Second
  International MATHEON Conference 2015}, pages 125--162. Springer, 2017.

\bibitem[Ailon et~al.(2013)Ailon, Chen, and Xu]{ailon2013breaking}
Nir Ailon, Yudong Chen, and Huan Xu.
\newblock Breaking the small cluster barrier of graph clustering.
\newblock In \emph{International Conference on Machine Learning}, pages
  995--1003, 2013.

\bibitem[Ailon et~al.(2015)Ailon, Chen, and Xu]{ailon2015iterative}
Nir Ailon, Yudong Chen, and Huan Xu.
\newblock Iterative and active graph clustering using trace norm minimization
  without cluster size constraints.
\newblock \emph{The Journal of Machine Learning Research}, 16\penalty0
  (1):\penalty0 455--490, 2015.

\bibitem[Ames(2014)]{ames_guaranteed_2014}
Brendan~P. Ames.
\newblock Guaranteed clustering and biclustering via semidefinite programming.
\newblock \emph{Mathematical Programming: Series A and B}, 147\penalty0
  (1-2):\penalty0 429--465, October 2014.
\newblock ISSN 0025-5610.
\newblock \doi{10.1007/s10107-013-0729-x}.
\newblock URL \url{https://doi.org/10.1007/s10107-013-0729-x}.

\bibitem[Amini and Levina(2018)]{amini_semidefinite_2018}
Arash~A. Amini and Elizaveta Levina.
\newblock On semidefinite relaxations for the block model.
\newblock \emph{The Annals of Statistics}, 46\penalty0 (1):\penalty0 149--179,
  February 2018.
\newblock ISSN 0090-5364, 2168-8966.
\newblock \doi{10.1214/17-AOS1545}.
\newblock Publisher: Institute of Mathematical Statistics.

\bibitem[Bandeira and Van~Handel(2016)]{bandeira2016sharp}
Afonso~S Bandeira and Ramon Van~Handel.
\newblock Sharp nonasymptotic bounds on the norm of random matrices with
  independent entries.
\newblock 2016.

\bibitem[Bollobás and Scott(2004)]{bollobas_max_2004}
B.~Bollobás and A.~D. Scott.
\newblock Max {Cut} for {Random} {Graphs} with a {Planted} {Partition}.
\newblock \emph{Combinatorics, Probability and Computing}, 13\penalty0
  (4-5):\penalty0 451--474, July 2004.
\newblock ISSN 0963-5483, 1469-2163.
\newblock \doi{10.1017/S0963548304006303}.
\newblock URL
  \url{http://www.journals.cambridge.org/abstract_S0963548304006303}.

\bibitem[Bordenave et~al.(2015)Bordenave, Lelarge, and
  Massouli{\'e}]{bordenave2015non}
Charles Bordenave, Marc Lelarge, and Laurent Massouli{\'e}.
\newblock Non-backtracking spectrum of random graphs: community detection and
  non-regular ramanujan graphs.
\newblock In \emph{2015 IEEE 56th Annual Symposium on Foundations of Computer
  Science}, pages 1347--1357. IEEE, 2015.

\bibitem[Boyd and Vandenberghe(2004)]{boyd_convex_2004}
Stephen~P. Boyd and Lieven Vandenberghe.
\newblock \emph{Convex optimization}.
\newblock Cambridge University Press, Cambridge, UK ; New York, 2004.
\newblock ISBN 978-0-521-83378-3.

\bibitem[Cai and Li(2015)]{cai_robust_2015}
T.~Tony Cai and Xiaodong Li.
\newblock Robust and {Computationally} {Feasible} {Community} {Detection} in
  the {Presence} of {Arbitrary} {Outlier} {Nodes}.
\newblock \emph{The Annals of Statistics}, 43\penalty0 (3):\penalty0
  1027--1059, 2015.
\newblock ISSN 0090-5364.
\newblock URL \url{https://www.jstor.org/stable/43556546}.
\newblock Publisher: Institute of Mathematical Statistics.

\bibitem[Caltagirone et~al.(2017)Caltagirone, Lelarge, and
  Miolane]{caltagirone2017recovering}
Francesco Caltagirone, Marc Lelarge, and L{\'e}o Miolane.
\newblock Recovering asymmetric communities in the stochastic block model.
\newblock \emph{IEEE Transactions on Network Science and Engineering},
  5\penalty0 (3):\penalty0 237--246, 2017.

\bibitem[Carson and Impagliazzo(2001)]{carson2001hill}
Ted Carson and Russell Impagliazzo.
\newblock Hill-climbing finds random planted bisections.
\newblock In \emph{Proceedings of the 12th Annual SIAM Symposium on Discrete
  Algorithms}. Citeseer, 2001.

\bibitem[Chaudhuri et~al.(2012)Chaudhuri, Chung, and
  Tsiatas]{chaudhuri_spectral_2012}
Kamalika Chaudhuri, Fan Chung, and Alexander Tsiatas.
\newblock Spectral {Clustering} of {Graphs} with {General} {Degrees} in the
  {Extended} {Planted} {Partition} {Model}.
\newblock In \emph{Proceedings of the 25th {Annual} {Conference} on {Learning}
  {Theory}}, pages 35.1--35.23. JMLR Workshop and Conference Proceedings, June
  2012.
\newblock URL \url{https://proceedings.mlr.press/v23/chaudhuri12.html}.
\newblock ISSN: 1938-7228.

\bibitem[Chen and Xu(2014)]{chen_statistical-computational_2014}
Yudong Chen and Jiaming Xu.
\newblock Statistical-computational phase transitions in planted models: the
  high-dimensional setting.
\newblock In \emph{Proceedings of the 31st {International} {Conference} on
  {International} {Conference} on {Machine} {Learning} - {Volume} 32},
  {ICML}'14, pages II--244--II--252, Beijing, China, June 2014. JMLR.org.

\bibitem[Chen et~al.(2012)Chen, Sanghavi, and Xu]{chen_clustering_2012}
Yudong Chen, Sujay Sanghavi, and Huan Xu.
\newblock Clustering {Sparse} {Graphs}.
\newblock In \emph{Advances in {Neural} {Information} {Processing} {Systems}},
  volume~25. Curran Associates, Inc., 2012.
\newblock URL
  \url{https://papers.nips.cc/paper_files/paper/2012/hash/1e6e0a04d20f50967c64dac2d639a577-Abstract.html}.

\bibitem[Chen et~al.(2014)Chen, Sanghavi, and Xu]{chen2012sparseclustering}
Yudong Chen, Sujay Sanghavi, and Huan Xu.
\newblock Improved graph clustering.
\newblock \emph{IEEE Transactions on Information Theory}, 60\penalty0
  (10):\penalty0 6440--6455, 2014.

\bibitem[Chen et~al.(2018)Chen, Li, and Xu]{CMM}
Yudong Chen, Xiaodong Li, and Jiaming Xu.
\newblock {Convexified Modularity Maximization for Degree-Corrected Stochastic
  Block Models}.
\newblock \emph{Annals of Statistics}, 46\penalty0 (4):\penalty0 1573--1602,
  2018.

\bibitem[Chen et~al.(2021)Chen, Chi, Fan, Ma, et~al.]{chen2021spectral}
Yuxin Chen, Yuejie Chi, Jianqing Fan, Cong Ma, et~al.
\newblock Spectral methods for data science: A statistical perspective.
\newblock \emph{Foundations and Trends{\textregistered} in Machine Learning},
  14\penalty0 (5):\penalty0 566--806, 2021.

\bibitem[Coja-Oghlan et~al.(2017)Coja-Oghlan, Krzakala, Perkins, and
  Zdeborov{\'a}]{coja2017information}
Amin Coja-Oghlan, Florent Krzakala, Will Perkins, and Lenka Zdeborov{\'a}.
\newblock Information-theoretic thresholds from the cavity method.
\newblock In \emph{Proceedings of the 49th Annual ACM SIGACT Symposium on
  Theory of Computing}, pages 146--157, 2017.

\bibitem[Condon and Karp(2001)]{condon2001algorithms}
Anne Condon and Richard~M. Karp.
\newblock Algorithms for graph partitioning on the planted partition model.
\newblock \emph{Random Structures and Algorithms}, 18\penalty0 (2):\penalty0
  116--140, 2001.

\bibitem[Ding and Chen(2020)]{ding2018loo}
Lijun Ding and Yudong Chen.
\newblock Leave-one-out approach for matrix completion: Primal and dual
  analysis.
\newblock \emph{{IEEE Transactions on Information Theory}}, 66\penalty0
  (11):\penalty0 7274 -- 7301, 2020.

\bibitem[Eldridge et~al.(2018)Eldridge, Belkin, and
  Wang]{eldridge2018unperturbed}
Justin Eldridge, Mikhail Belkin, and Yusu Wang.
\newblock Unperturbed: spectral analysis beyond davis-kahan.
\newblock In \emph{Algorithmic Learning Theory}, pages 321--358. PMLR, 2018.

\bibitem[Feige and Kilian(2001)]{feige2001heuristics}
Uriel Feige and Joe Kilian.
\newblock Heuristics for semirandom graph problems.
\newblock \emph{Journal of Computer and System Sciences}, 63\penalty0
  (4):\penalty0 639--671, 2001.

\bibitem[Green~Larsen et~al.(2020)Green~Larsen, Mitzenmacher, and
  Tsourakakis]{green_larsen_clustering_2020}
Kasper Green~Larsen, Michael Mitzenmacher, and Charalampos Tsourakakis.
\newblock Clustering with a faulty oracle.
\newblock In \emph{Proceedings of {The} {Web} {Conference} 2020}, {WWW} '20,
  pages 2831--2834, New York, NY, USA, April 2020. Association for Computing
  Machinery.
\newblock ISBN 978-1-4503-7023-3.
\newblock \doi{10.1145/3366423.3380045}.
\newblock URL \url{https://doi.org/10.1145/3366423.3380045}.

\bibitem[Hoeffding(1994)]{hoeffding1994probability}
Wassily Hoeffding.
\newblock Probability inequalities for sums of bounded random variables.
\newblock \emph{The collected works of Wassily Hoeffding}, pages 409--426,
  1994.

\bibitem[Holland et~al.(1983)Holland, Laskey, and Leinhardt]{holland83}
Paul~W. Holland, Kathryn~B. Laskey, and Samuel Leinhardt.
\newblock Stochastic blockmodels: Some first steps.
\newblock \emph{Social Networks}, 5:\penalty0 109--137, 1983.

\bibitem[Jerrum and Sorkin(1998)]{jerrum1998metropolis}
Mark Jerrum and Gregory~B Sorkin.
\newblock The metropolis algorithm for graph bisection.
\newblock \emph{Discrete Applied Mathematics}, 82\penalty0 (1-3):\penalty0
  155--175, 1998.

\bibitem[Jog and Loh(2015)]{jog2015information}
Varun Jog and Po-Ling Loh.
\newblock Information-theoretic bounds for exact recovery in weighted
  stochastic block models using the renyi divergence.
\newblock \emph{arXiv preprint arXiv:1509.06418}, 2015.

\bibitem[Krivelevich and Vilenchik(2006)]{krivelevich2006coloring}
Michael Krivelevich and Dan Vilenchik.
\newblock Semirandom models as benchmarks for coloring algorithms.
\newblock In \emph{Third Workshop on Analytic Algorithmics and Combinatorics
  (ANALCO)}, pages 211--221, 2006.

\bibitem[Lelarge et~al.(2015)Lelarge, Massouli{\'e}, and
  Xu]{lelarge2015reconstruction}
Marc Lelarge, Laurent Massouli{\'e}, and Jiaming Xu.
\newblock Reconstruction in the labelled stochastic block model.
\newblock \emph{IEEE Transactions on Network Science and Engineering},
  2\penalty0 (4):\penalty0 152--163, 2015.

\bibitem[Li et~al.(2021)Li, Chen, and Xu]{li2021convex}
Xiaodong Li, Yudong Chen, and Jiaming Xu.
\newblock Convex relaxation methods for community detection.
\newblock \emph{Statistical Science}, 36\penalty0 (1):\penalty0 2--15, 2021.

\bibitem[Lu and Zhou(2016)]{lu2016lloyd}
Yu~Lu and Harrison~H. Zhou.
\newblock Statistical and computational guarantees of {Lloyd's} algorithm and
  its variants.
\newblock \emph{arXiv preprint arXiv:1612.02099}, 2016.

\bibitem[Marshall et~al.(1979)Marshall, Olkin, and
  Arnold]{marshall1979inequalities}
Albert~W Marshall, Ingram Olkin, and Barry~C Arnold.
\newblock Inequalities: theory of majorization and its applications.
\newblock 1979.

\bibitem[Massouli{\'e}(2014)]{massoulie2014community}
Laurent Massouli{\'e}.
\newblock Community detection thresholds and the weak ramanujan property.
\newblock In \emph{Proceedings of the forty-sixth annual ACM symposium on
  Theory of computing}, pages 694--703, 2014.

\bibitem[Mazumdar and Saha(2017)]{mazumdar_clustering_2017}
Arya Mazumdar and Barna Saha.
\newblock Clustering with {Noisy} {Queries}.
\newblock In I.~Guyon, U.~Von Luxburg, S.~Bengio, H.~Wallach, R.~Fergus,
  S.~Vishwanathan, and R.~Garnett, editors, \emph{Advances in {Neural}
  {Information} {Processing} {Systems}}, volume~30. Curran Associates, Inc.,
  2017.
\newblock URL
  \url{https://proceedings.neurips.cc/paper_files/paper/2017/file/db5cea26ca37aa09e5365f3e7f5dd9eb-Paper.pdf}.

\bibitem[McSherry(2001)]{mcsherry_spectral_2001}
F.~McSherry.
\newblock Spectral partitioning of random graphs.
\newblock In \emph{Proceedings 42nd {IEEE} {Symposium} on {Foundations} of
  {Computer} {Science}}, pages 529--537, October 2001.
\newblock \doi{10.1109/SFCS.2001.959929}.
\newblock ISSN: 1552-5244.

\bibitem[Mitzenmacher and Upfal(2005)]{mitzenmacher_probability_2005}
Michael Mitzenmacher and Eli Upfal.
\newblock \emph{Probability and {Computing}: {Randomized} {Algorithms} and
  {Probabilistic} {Analysis}}.
\newblock Cambridge University Press, Cambridge, 2005.
\newblock \doi{10.1017/CBO9780511813603}.
\newblock URL
  \url{https://www.cambridge.org/core/books/probability-and-computing/3A5B47DB315FC64B9256C5C8131C5EFA}.

\bibitem[Moitra et~al.(2016)Moitra, Perry, and Wein]{moitra2016robust}
Ankur Moitra, William Perry, and Alexander~S. Wein.
\newblock How robust are reconstruction thresholds for community detection?
\newblock In \emph{Proceedings of the 48th Annual ACM SIGACT Symposium on
  Theory of Computing}, pages 828--841. ACM, 2016.

\bibitem[Montanari and Sen(2016)]{montanari2016semidefinite}
Andrea Montanari and Subhabrata Sen.
\newblock Semidefinite programs on sparse random graphs and their application
  to community detection.
\newblock In \emph{Proceedings of the forty-eighth annual ACM symposium on
  Theory of Computing}, pages 814--827, 2016.

\bibitem[Mossel et~al.(2015{\natexlab{a}})Mossel, Neeman, and
  Sly]{mossel2015consistency}
Elchanan Mossel, Joe Neeman, and Allan Sly.
\newblock Consistency thresholds for the planted bisection model.
\newblock In \emph{Proceedings of the forty-seventh annual ACM symposium on
  Theory of computing}, pages 69--75, 2015{\natexlab{a}}.

\bibitem[Mossel et~al.(2015{\natexlab{b}})Mossel, Neeman, and
  Sly]{mossel2015reconstruction}
Elchanan Mossel, Joe Neeman, and Allan Sly.
\newblock Reconstruction and estimation in the planted partition model.
\newblock \emph{Probability Theory and Related Fields}, 162:\penalty0 431--461,
  2015{\natexlab{b}}.

\bibitem[Mukherjee and Zhang(2022)]{mukherjee_detecting_2022}
Chandra~Sekhar Mukherjee and Jiapeng Zhang.
\newblock Detecting {Hidden} {Communities} by {Power} {Iterations} with
  {Connections} to {Vanilla} {Spectral} {Algorithms}, November 2022.
\newblock URL \url{http://arxiv.org/abs/2211.03939}.
\newblock arXiv:2211.03939 [math] version: 1.

\bibitem[Mukherjee et~al.(2022)Mukherjee, Peng, and
  Zhang]{mukherjee2022unbalanced}
Chandra~Sekhar Mukherjee, Pan Peng, and Jiapeng Zhang.
\newblock Recovering unbalanced communities in the stochastic block model with
  application to clustering with a faulty oracle.
\newblock \emph{arXiv preprint arXiv:2202.08522}, 2022.

\bibitem[Peng and Zhang(2021)]{peng_towards_2021}
Pan Peng and Jiapeng Zhang.
\newblock Towards a query-optimal and time-efficient algorithm for clustering
  with a faulty oracle.
\newblock In Mikhail Belkin and Samory Kpotufe, editors, \emph{Proceedings of
  Thirty Fourth Conference on Learning Theory}, volume 134 of \emph{Proceedings
  of Machine Learning Research}, pages 3662--3680. PMLR, 15--19 Aug 2021.
\newblock URL \url{https://proceedings.mlr.press/v134/peng21a.html}.

\bibitem[Perry and Wein(2017)]{perry2017semidefinite}
Amelia Perry and Alexander~S Wein.
\newblock A semidefinite program for unbalanced multisection in the stochastic
  block model.
\newblock In \emph{2017 International Conference on Sampling Theory and
  Applications (SampTA)}, pages 64--67. IEEE, 2017.

\bibitem[Pia et~al.(2022)Pia, Ma, and Tzamos]{pia_clustering_2022}
Alberto~Del Pia, Mingchen Ma, and Christos Tzamos.
\newblock Clustering with {Queries} under {Semi}-{Random} {Noise}.
\newblock In \emph{Proceedings of {Thirty} {Fifth} {Conference} on {Learning}
  {Theory}}, pages 5278--5313. PMLR, June 2022.
\newblock URL \url{https://proceedings.mlr.press/v178/pia22a.html}.
\newblock ISSN: 2640-3498.

\bibitem[Sagnol(2011)]{sagnol_class_2011}
Guillaume Sagnol.
\newblock A class of semidefinite programs with rank-one solutions.
\newblock \emph{Linear Algebra and its Applications}, 435\penalty0
  (6):\penalty0 1446--1463, September 2011.
\newblock ISSN 0024-3795.
\newblock \doi{10.1016/j.laa.2011.03.027}.
\newblock URL
  \url{https://www.sciencedirect.com/science/article/pii/S0024379511002515}.

\bibitem[Snijders and Nowicki(1997)]{snijders1997estimation}
Tom~AB Snijders and Krzysztof Nowicki.
\newblock Estimation and prediction for stochastic blockmodels for graphs with
  latent block structure.
\newblock \emph{Journal of classification}, 14\penalty0 (1):\penalty0 75--100,
  1997.

\bibitem[Vershynin(2018)]{vershynin_high-dimensional_2018}
Roman Vershynin.
\newblock \emph{High-{Dimensional} {Probability}: {An} {Introduction} with
  {Applications} in {Data} {Science}}.
\newblock Cambridge {Series} in {Statistical} and {Probabilistic}
  {Mathematics}. Cambridge University Press, Cambridge, 2018.
\newblock ISBN 978-1-108-41519-4.
\newblock \doi{10.1017/9781108231596}.
\newblock URL
  \url{https://www.cambridge.org/core/books/highdimensional-probability/797C466DA29743D2C8213493BD2D2102}.

\bibitem[Vu(2018)]{vu_simple_2018}
Van Vu.
\newblock A {Simple} {SVD} {Algorithm} for {Finding} {Hidden} {Partitions}.
\newblock \emph{Combinatorics, Probability and Computing}, 27\penalty0
  (1):\penalty0 124--140, January 2018.
\newblock ISSN 0963-5483, 1469-2163.
\newblock \doi{10.1017/S0963548317000463}.
\newblock URL
  \url{https://www.cambridge.org/core/journals/combinatorics-probability-and-computing/article/abs/simple-svd-algorithm-for-finding-hidden-partitions/3691FFE262617493CEC469769154EA19}.
\newblock Publisher: Cambridge University Press.

\bibitem[Xia and Huang(2022)]{xia_optimal_2022}
Jinghui Xia and Zengfeng Huang.
\newblock Optimal {Clustering} with {Noisy} {Queries} via {Multi}-{Armed}
  {Bandit}, July 2022.
\newblock URL \url{http://arxiv.org/abs/2207.05376}.
\newblock arXiv:2207.05376 [cs].

\bibitem[Zhong and Boumal(2018)]{zhong_near-optimal_2018}
Yiqiao Zhong and Nicolas Boumal.
\newblock Near-{Optimal} {Bounds} for {Phase} {Synchronization}.
\newblock \emph{SIAM Journal on Optimization}, 28\penalty0 (2):\penalty0
  989--1016, January 2018.
\newblock ISSN 1052-6234.
\newblock \doi{10.1137/17M1122025}.
\newblock URL \url{https://epubs.siam.org/doi/10.1137/17M1122025}.
\newblock Publisher: Society for Industrial and Applied Mathematics.

\end{thebibliography}

\appendix

\clearpage

\section{Additional Notation}
\label{sec:additional-notation}

We introduce some additional notation used in the rest of the paper.
For a matrix $X$, 
$\nucnorm X$ is its nuclera norm (the sum of singular values), $\onenorm X:=\sum_{i,j}\left|X_{ij}\right|$ is its entrywise $\ell_1$ norm, and $\supp(X):=\left\{ (i,j):X_{ij}\neq0\right\} $ is its support.
We define $\eb=\onevec/\twonorm{\onevec}$ and we let $e_j$ denote a standard basis vector with entry $j$, both of which are only used when their size is clear from context. 
For any event $E$ we define $E^c$ to be its complement. For any natural number $h$ we define $[h]=\{1,\dots, h\}$.

\section{Additional Discussion of Results and Related Work}
\subsection{Further Details of Ailon et al. and Gap-Dependent Clustering}
\label{sec:details_of_ailon}

\cite{ailon2013breaking, ailon2015iterative} consider a convex relaxation approach based on semidefinite programming. As discussed in the introduction, their results allow for small clusters, but additionally assume there exists a constant \emph{multiplicative gap} between the sizes of large and small clusters. Again, letting $\sb >\ss$ be two consecutive cluster sizes (i.e., there are no clusters with sizes in $(\ss, \sb)$), they guarantee exact recovery of all clusters with sizes at least $\sb$ as long as the following conditions are satisfied:
\begin{equation}
\frac{(p-q)^{2}\sb^{2}}{pn}\gtrsim\log n 
\qquad\text{and}\qquad
\sb/\ss \gtrsim 1.
\label{eq:ailon_old}
\end{equation}
Under the above condition and after reordering nodes by cluster identity, their optimization program returns a block-diagonal solution of the form
\begin{equation}
\Yhat = \diag \left(J_{s_1\times s_1} , \dots, J_{\sb \times \sb}, 0_{\ss \times \ss}, \dots, 0_{s_K \times s_K} \right).
\end{equation}

Now we present the formal version of our gap-dependent recovery result in Theorem~\ref{thm:informal_clustering_with_gap}. While inferior to our main Theorem~\ref{thm:main_theorem_recovery_sdp}, this result has a simpler proof and is in fact used as a subroutine in the proof of  Theorem~\ref{thm:main_theorem_recovery_sdp}.\begin{thm}
\label{thm:clustering_with_a_gap}
There exists an absolute constant $C$ such that the following holds:
    If there exists some $k\in\{1,\ldots,K\}$ such that 
\begin{align}
\frac{p-q}{2}\left( s_k - s_{k+1}\right) \geq \kappa \sqrt{pn \log n} \quad \text{or equivalently} \quad \frac{(p-q)^{2}\left(s_{k}-s_{k+1}\right)^{2}}{pn}  \geq 4 \kappa^2 \log n \label{eq:exact_recovery_condition_with_gap}
\end{align}
and $\kappa \geq C$, then if we set $\lambda = \kappa \sqrt{pn \log n} + \frac{p-q}{2}s_{k+1}$ in~\eqref{eq:SDP_regularized_original_1}, with probability at least $1-O(n^{-3})$,
\[\Yhat=\diag \left(J_{s_1\times s_1}, \dots, J_{s_k \times s_k}, 0_{s_{k+1}\times s_{k+1}}, \dots, 0_{s_{K}\times s_{K}} \right)\]
is the unique optimal solution to the recovery SDP~\eqref{eq:SDP_regularized_original}.
\end{thm}
This theorem improves upon the result in \eqref{eq:ailon_old} because instead of requiring the size ratio $\sb/\ss$ between large and small clusters to be a large constant, it allows for the ratio to approach $1$ so long as the signal-to-noise ratio, $\frac{(p-q)^{2}\sb^{2}}{pn}$ is proportionally larger.

Note that Theorem \ref{thm:clustering_with_a_gap} also utilizes a value of $\lambda$ which depends on the size of the largest small cluster $s_{k+1}$, which may not be straightforward to find, while Theorem \ref{thm:main_theorem_recovery_sdp} can be set without this information.

Now we provide a simple instance where no gap-dependent algorithm can recover any cluster.
Suppose $p=1, q=1/2$, there are $n$ nodes, and $K = \Theta(n^{1/3})$ clusters with sizes $1, 2^2, 3^2, \dots, K^2$. (Note $n = \sum_{k=1}^K k^2 = \Theta(K^3)$.)  All known efficient algorithms have recovery thresholds requiring cluster sizes to be at least $\Theta(\sqrt{n}) = \Theta(K^{1.5})$. To recover any cluster, a gap dependent algorithm will require at least a constant multiplicative gap between some cluster of size at least $\Theta(\sqrt{n})$ and the next largest cluster. As $K$ and $n$ get large, the ratios between the sizes of all pairs of such consecutive clusters approaches $1$, hence the requisite gap does not exist. On the other hand, our gap-free algorithm can recover all clusters larger than $\Theta(K^{1.5} \log K)$. This example can be easily generalized to cluster sizes growing like $(k^\alpha)_k$ for any $\alpha > 1$, as well as to other settings of $p$ and $q$.

\subsection{Additional Related Work}
\label{sec:more_related_work}

In proceeding sections we have provided references to SBM algorithms focusing on SDP relaxation and spectral clustering methods. There are many other approaches, such as the EM algorithm \citep{snijders1997estimation} and local search methods \citep{carson2001hill, jerrum1998metropolis}, many of which may be coupled with an SDP-based or spectral algorithm with the latter serving as an initialization procedure~\citep{lu2016lloyd,vu_simple_2018,abbe_community_2018}.

Many works have considered the setting that all clusters are sufficiently large to be exactly recovered, leading to precise computational and information-theoretic thresholds in the case of two balanced clusters \citep{abbe2015exact, mossel2015consistency}. Much subsequent work has endeavored to obtain similar characterizations in more general settings, for instance in \cite{abbe2015community}, \cite{abbe2017entrywise}, \cite{perry2017semidefinite}, \cite{agarwal2017multisection}, and \cite{jog2015information}, where typical recovery conditions have the form $\frac{(p-q)^2}{pn}s_K \gtrsim \log n$ assuming $p,q \asymp \frac{\log n}{n}$, where $s_K$ is the size of the smallest cluster.
There is also a complementary weak or approximate recovery setting, where the objective is to achieve strictly better classification error than random guess, usually under the constant degree regime $p,q \asymp \frac{1}{n}$. Again precise phase transitions have been established for the two-balanced-clusters case \citep{massoulie2014community, lelarge2015reconstruction, mossel2015reconstruction}, and generalizations to larger $K$ or unbalanced settings, such as \cite{bordenave2015non}, \cite{coja2017information}, \cite{caltagirone2017recovering}, and \cite{montanari2016semidefinite}, develop recovery conditions of the form $\frac{(p-q)^2}{pn}s_K \gtrsim 1$. While the focus of the current paper is mostly orthogonal to the above work, our result, when specialized to the setting where all clusters are large, recovers the the condition $\frac{(p-q)^2}{pn}s_K \gtrsim \log n$ typical in existing work.

The semi-random SBM has been considered in many works without small clusters \citep{moitra2016robust,krivelevich2006coloring,feige2001heuristics}. The observation that any recovery result for semi-random SBM implies a result for SBM with heterogeneous edge probabilities is a standard reduction~\citep{chen2012sparseclustering}. 
For more on the leave-one-out approach applied to clustering and related matrix estimation problems, see \cite{zhong_near-optimal_2018,abbe_entrywise_2020,ding2018loo}, as well as the survey in \cite{chen2021spectral}. The work \cite{green_larsen_clustering_2020} studies faulty oracle clustering in the setting of $K=2$ clusters.

\subsection{Discussion and Comparison of Recursive Clustering Results}
\label{sec:more_recursive_clustering}

To analyze the performance of recursive clustering procedures, we follow \cite[Theorems 6 and 7]{ailon2015iterative} and consider the following setting: $p,q$ are considered fixed and independent of $n$, and $C(p,q)$ is used to denote constants which depend on $p,q$ but not $n$. (Generalization to other settings of $p,q$ is straightforward.) We examine what assumptions on the number of clusters $K$ suffice to ensure that the recursive clustering procedure recovers all but a vanishing fraction of nodes.

\citet[Theorem 7]{ailon2015iterative} roughly states there exists some $C(p,q)$ such that if $K \leq \beta C(p,q) \log n$ for some $\beta \leq 1$, then a recursive clustering procedure will recover clusters containing all but $O(n^\beta)+O(1)$ nodes. Since, as discussed in Subsection \ref{sec:algorithms_for_unbalanced_sbm}, their core recovery algorithm requires a certain-sized gap, they require the bound on $K$ to grow at most logarithmically to ensure that a sufficient gap remains after each round. 

If we instead apply our gap-free clustering guarantee from Theorem \ref{thm:main_theorem_recovery_sdp}, we obtain a significant improvement, allowing $K$ to be polynomially large. In particular, when specialized to the setting above, Theorem \ref{thm:main_theorem_recovery_sdp} guarantees that there exists some $C'(p,q)$ such that (with probability $1-O(m^{-3})$) all clusters $k$ satisfying $s_k \geq C'(p,q) \sqrt{n \log m}$ will be recovered. Using this improved guarantee, we follow the aforementioned recursive clustering strategy to define Algorithm~\ref{alg:recursive_clustering}. For a square matrix $X$, we use $X_I$ to denote the principal submatrix of $X$ indexed by $I$. 
When the SDP solution $\Yhat$ has the form stated in Theorem~\ref{thm:main_theorem_recovery_sdp}, each cluster $V_k$ corresponding an all-positive block $\Yhat_{V_k}$ is said to be recovered.

\renewcommand{\algorithmicrequire}{\textbf{Input:}}
\begin{algorithm}[H]
\caption{Recursive clustering using Theorem \ref{thm:main_theorem_recovery_sdp}} \label{alg:recursive_clustering}
\begin{algorithmic}[1]
\Require $n \times n$ adjacency matrix $A$ sampled from unbalanced SBM
\State Set $n_1 = n$, $I_1=\{1,\ldots,n\}$, round counter $\ell = 1$
\Repeat
\State Recover clusters by applying the SDP~\eqref{eq:SDP_regularized_original} to $A_{I_\ell}$, where $\lambda$ is set per Theorem~\ref{thm:main_theorem_recovery_sdp} with $m=n$
\State Remove all recovered clusters of size $\geq C'(p,q) \sqrt{n_\ell \log n}$ from $I_\ell$ to obtain $I_{\ell+1}$
\State Set $n_{\ell + 1}=|I_{\ell+1}|$; increase $\ell$ by 1
\Until{no clusters are removed}
\end{algorithmic}
\end{algorithm}

We can prove the following guarantees:
\begin{thm}
    \label{thm:recursive_clsutering}
    Under the above setting, suppose the number of clusters $K$ satisfies $K \leq \frac{n^{\frac{1}{2}-\alpha}}{C'(p,q) \sqrt{\log n}}$ for some $0 < \alpha < \frac{1}{2}$. Then with probability at least $1-O(n^{-2})$:
    \begin{enumerate}
        \item The algorithm terminates after at most $K$ rounds.
        \item Whenever the algorithm terminates, there is at most $n^{1-2\alpha}$ unrecovered nodes remaining.
        \item After $\ell$ rounds, there is at most $n^{1-\alpha \sum_{i = 0}^{\ell-1}2^{- i}}$ unrecovered nodes remaining.
    \end{enumerate}
\end{thm}

Note that we interpret an empty summation as zero, that is, $\sum_{\ell=0}^{-1} 2^{-\ell} =0$.
The algorithm will run for at most $K$ rounds, but it could also be terminated early to save on computation, as part 3 of Theorem~\ref{thm:recursive_clsutering} guarantees that a large number of nodes will be recovered after only a few rounds. The proof of Theorem \ref{thm:recursive_clsutering} is given in Section \ref{sec:recursive_clustering_proofs}.

\subsection{Discussion and Comparison of Faulty Oracle Clustering Results}
\label{sec:more_faulty_oracle_clustering}
First we make some comments regarding the algorithms in Table \ref{table:faulty_oracle_clustering_algorithms}. The listed algorithm of \cite{peng_towards_2021} assumes that all clusters are ``nearly balanced'': they have size at least $\frac{n}{bK}$ for some known constant $b>1$; the algorithm recovers all clusters under this assumption. The algorithm of \cite{pia_clustering_2022} is robust to a semirandom version of the faulty oracle clustering model (and we refer to their paper for its definition). 
Additionally, we note that the listed algorithm of \cite{mukherjee2022unbalanced} actually obtains a sample complexity of $O\left(\frac{n}{s}\frac{n\log^2 n}{\delta^2} + \frac{n^4}{s^4}\frac{\log^2 n}{\delta^4} \right)$, which is slightly worse by a $\log$ factor in one term compared to our Theorem \ref{thm:faulty_oracle_clustering_size_param}.

Now we formally define the Algorithm \ref{alg:clustering_with_faulty_oracle_meta_algorithm} behind our Theorem \ref{thm:faulty_oracle_clustering_size_param} and provide further discussion of Theorem \ref{thm:faulty_oracle_clustering_size_param}. 

\begin{algorithm}[H]
\caption{Clustering with a faulty oracle using Theorem \ref{thm:main_theorem_recovery_sdp}} \label{alg:clustering_with_faulty_oracle_meta_algorithm}
\begin{algorithmic}[1]
\Require Faulty clustering oracle $\mathcal{O}$, parameter $\gamma \in (0,1)$
\State Choose set $T \subseteq [n]$ by including vertex $i$ in $T$ with probability $\gamma$ (independently)
\State Query $(u,v)$ for all pairs $u,v \in T$ to form adjacency matrix $A \in \{0,1\}^{|T| \times |T|}$
\State Recover clusters from $A$ using the recovery SDP~\eqref{eq:SDP_regularized_original}, where $\lambda$ is set per Theorem~\ref{thm:main_theorem_recovery_sdp} with $m=n$
\For{each recovered cluster $S\subseteq T$ such that $|S| \geq \frac{C_1 \log n}{\delta^2}$}
\State Choose $S' \subseteq S$ with $|S'| = \frac{C_1 \log n}{\delta^2}$
\Statex $\quad\;\; \triangleright$ Recover remainder of cluster in $[n] \setminus T$ by majority voting using $S'$ as voters:
\For{each vertex $i \in [n] \setminus T$}
\State Query $(i, v)$ for all $v \in S'$ and add $i$ to $S$ if $\frac{1}{|S'|}\sum_{v\in S'} \mathcal{O}(i,v) \ge \frac{1}{2}$
\EndFor
\EndFor

\end{algorithmic}
\end{algorithm}

We prove Theorem \ref{thm:faulty_oracle_clustering_size_param} in Section \ref{sec:clustering_with_a_faulty_oracle_proofs}. We also note that by setting $s = \frac{n}{K}$ in this result, we recover identical query complexity and cluster recovery threshold to the algorithm of \cite{peng_towards_2021} for the nearly-balanced setting listed in Table \ref{table:faulty_oracle_clustering_algorithms}, except our algorithm has no restrictions on cluster sizes as opposed to requiring them to be nearly balanced. \cite{peng_towards_2021} conjecture that this sample complexity is optimal among efficient algorithms (up to $\log$ factors) for recovery of clusters of size $\Omega\left(\frac{n}{K}\right)$, based on statistical physics arguments.

Our Algorithm \ref{alg:clustering_with_faulty_oracle_meta_algorithm} used in Theorem \ref{thm:faulty_oracle_clustering_size_param} follows a common template for algorithms in the faulty oracle clustering problem: 
\begin{enumerate}[nolistsep]
    \item Choose a subset $T$ of the nodes and query all pairs between them. 
    \item Apply a clustering algorithm on the fully observed subgraph consisting of the nodes $T$ and the edges between them to recover subclusters. 
    \item Take a subset $S'$ from each recovered subcluster of minimal size such that it can be used to test all remaining nodes $i \in [n] \setminus T$ for membership in the cluster, by querying $(i, v)$ for all $v \in S'$ and using majority voting. 
\end{enumerate}
These essential steps are used in all algorithms in Table \ref{table:faulty_oracle_clustering_algorithms}, with one key difference being the clustering algorithm used in step 2 (for which we use our recovery SDP and Theorem \ref{thm:main_theorem_recovery_sdp}). \cite{mukherjee2022unbalanced} utilize their (gap-free) spectral clustering procedure and thus obtain a nearly identical result. This template also explains why all the sample complexities (for efficient algorithms) consist of two terms, with the first term roughly corresponding to the majority voting procedure and the second term roughly corresponding to the subsample of size $O(|T|^2)$, with possible additional steps for some algorithms. Notably, \cite{xia_optimal_2022} improve upon the majority voting procedure by making a connection to best arm identification, obtaining a first term of $\frac{n(K+ \log n)}{\delta^2}$ which they show is optimal.

Algorithm \ref{alg:clustering_with_faulty_oracle_meta_algorithm} requires a user-specified target cluster size $s$, which in turn determines the input parameter $\gamma$, as does \cite[Algorithm 5]{mukherjee2022unbalanced}. By growing the subsample $T$ geometrically over several rounds, stopping when we first recover a (sub)cluster, we can obtain our Theorem \ref{thm:clustering_with_faulty_oracle_adaptive} which obtains an instance-dependent query complexity.

\begin{algorithm}[H]
\caption{Instance-adaptive clustering with a faulty oracle} \label{alg:clustering_with_faulty_oracle_adaptive}
\begin{algorithmic}[1]
\Require Faulty clustering oracle $\mathcal{O}$
\State Set round counter $r = 1$
\Repeat
\State Set $\hat{s}_r = \frac{n}{2^{r-1}}$
\State Set $\gamma_r = C_2 \frac{n \log n}{\hat{s}_r^2 \delta^2} $
\State Choose set $T_r \subseteq [n]$ by including vertex $i$ in $T_r$ with probability $\gamma_r$ (independently)
\State Query $(u,v)$ for all pairs $u,v \in T_r$ to form adjacency matrix $A_r \in \{0,1\}^{|T_r| \times |T_r|}$
\State Attempt to recover subclusters from $A_r$ using the recovery SDP~\eqref{eq:SDP_regularized_original} (with $m = n$)
\Until{at least one subcluster is recovered}
\State Let $S$ be an arbitrary recovered subcluster of size at least $\frac{C_1 \log n}{\delta^2}$
\State Choose $S' \subseteq S$ with $|S'| = \frac{C_1 \log n}{\delta^2}$
\Statex $\triangleright$ Recover remainder of cluster in $[n] \setminus T_r$ by majority voting using $S'$ as voters:
\For{each vertex $i \in [n] \setminus T_r$}
\State Query $(i, v)$ for all $v \in S'$ and add $i$ to $S$ if the majority of queries are $1$
\EndFor
\end{algorithmic}
\end{algorithm}

Finally we discuss our algorithm for the setting with a small number of clusters $K$ which was the subject of our Theorem \ref{thm:faulty_oracle_clustering_improved_complexity_small_k}.

\begin{algorithm}[H]
\caption{Clustering with a faulty oracle for small $K$} \label{alg:clustering_with_faulty_oracle_small_K}
\begin{algorithmic}[1]
\Require Faulty clustering oracle $\mathcal{O}$
\State Set round counter $r = 1$
\State Set $n_1 = n$, $K_1 = K$
\While{$n_r \geq C_2 \frac{K_r^2 \log n}{ \delta^2}$}
\State Set $\gamma_r = C_2 \frac{K_r^2 \log n}{n_r \delta^2} $
\State Choose set $T_r \subseteq [n]$ by including vertex $i$ in $T_r$ with probability $\gamma_r$ (independently)
\State Query $(u,v)$ for all pairs $u,v \in T_r$ to form adjacency matrix $A_r \in \{0,1\}^{|T_r| \times |T_r|}$
\State Attempt to recover subclusters from $A_r$ using the recovery SDP~\eqref{eq:SDP_regularized_original} (with $m = n$)
\For{each recovered cluster $S\subseteq T$ such that $|S| \geq \frac{C_1 \log n}{\delta^2}$}
\State Choose $S' \subseteq S$ with $|S'| = \frac{C_1 \log n}{\delta^2}$
\State Recover remainder of cluster in $[n] \setminus T_r$ by majority voting using $S'$ as voters:
\For{each vertex $i \in [n] \setminus T_r$}
\State Query $(i, v)$ for all $v \in S'$ and add $i$ to $S$ if the majority of queries are $1$
\EndFor
\EndFor
\State Increase $r$ by $1$
\State Decrease $K_r$ by the number of clusters recovered in the previous round
\State Remove all nodes from clusters recovered in the previous round
\State Set $n_r$ to the remaining number of nodes
\EndWhile
\end{algorithmic}
\end{algorithm}

The Algorithm~\ref{alg:clustering_with_faulty_oracle_small_K} is nearly identical to Algorithm \ref{alg:clustering_with_faulty_oracle_meta_algorithm} except we use a different subsampling ratio $\gamma$ and repeat in at most $K$ rounds (choosing a new subsample $T$ each round). This use of $K$ rounds is the reason that the second term in the query complexity is $\frac{K^5\log^2 n}{\delta^4}$ rather than the $\frac{K^4\log^2 n}{\delta^4}$, which \cite{peng_towards_2021} conjecture is optimal for the nearly balanced setting. In the setting that all clusters are $\Omega\left(\frac{n}{K}\right)$, only one round of choosing $T \subseteq [n]$ and querying all pairs $u,v \in T $ would be needed to recover a subcluster from each of the $K$ clusters, while in general settings more rounds have the potential benefit of recovering smaller clusters after larger clusters are removed. (In fact, our Algorithm \ref{alg:clustering_with_faulty_oracle_small_K} would match this conjectured-optimal sample complexity in the nearly-balanced setting as it would terminate after a constant number of rounds.) In summary, our algorithm recovers smaller clusters and improves upon the second term in the sample complexity compared to all previous algorithms which work for the case of general size clusters. We also believe that the first term could be improved to $\frac{n(K+ \log n)}{\delta^2}$ using ideas from \cite{xia_optimal_2022}, but our focus was on the improvement available by using a gap-free clustering subroutine.

\subsection{Discussion of Eigenvalue Perturbation Bounds}
\label{sec:eigenvalue_perturbation_bounds_subsubsection}
As discussed in Subsection \ref{sec:eigenvalue_perturbation_simplified}, our analysis of the SDP solutions in the presence of mid-size clusters
requires sharp eigenvalue perturbation bounds which improve upon existing results.
We expand on our discussion in Subsection \ref{sec:eigenvalue_perturbation_simplified} and provide some brief background on related eigenvalue perturbation results. We refer to \cite{eldridge2018unperturbed} for more details. Again, let $M, H \in \R^{n \times n}$ be symmetric matrices, viewing $H$ as a perturbation, and let $v_1$ be a top eigenvector of $M$. If we would like to bound $\left|\lambda_1(M+H)-\lambda_1(M)\right|$, the classical Weyl's inequality gives that
\[
\left|\lambda_1(M+H)-\lambda_1(M)\right| \leq \opnorm{H}.
\]
As argued in \cite{eldridge2018unperturbed}, this bound can be overly pessimistic, especially when $H$ is a random perturbation whose projection onto $v_1$ is small relative to the eigengap $\lambda_1(M) - \lambda_2(M)$. In such a situation, intuitively the large eigengap should ensure that the new top eigenvector of $M+H$ is close to $v_1$, in which case the top eigenvalue of $M+H$ will be approximately $v_1^\top \left( M+H \right) v_1 = \lambda_1(M) + v_1^\top H v_1$, and then since $H$ is random, $\left|v_1^\top H v_1\right|$ should be much smaller than $\opnorm{H}$. The idea that $\left|\lambda_1(M+H)-\lambda_1(M)\right| \approx \left|v_1^\top H v_1\right|$ is stated as an informal guiding principle in \cite{eldridge2018unperturbed}, and they confirm the above intuition with both experiments and theoretical results. Let $\col Q$ denotes the span of the columns of a matrix $Q$. They prove the following upper bound:

\begin{thm}[\cite{eldridge2018unperturbed}, Theorem 6]
\label{thm:eldridge_eval_pert}
Let $T \in [n]$ and $h$ be such that $|x^\top H x| \leq h$ for all $x \in \col( U_{1:T})$ with $\twonorm{x}=1$, where $U_{1:T}$ are the top $T$ eigenvectors of $M$. Then for $t \leq T$, if $\lambda_t - \lambda_{T+1} > 2 \opnorm{H}-h$,
\begin{align*}
    \lambda_t(M+H) \leq \lambda_t(M) + h + \frac{\opnorm{H}^2}{\lambda_t(M) - \lambda_{T+1}(M) + h - \opnorm{H}}.
\end{align*}
\end{thm}

We modify their argument to prove the following improvement:

\begin{thm}
\label{thm:general_eval_pert_bound}
Let $T \in [n]$ and $h$ be such that $|x^\top H x| \leq h$ for all $x \in \textnormal{Col}( U_{t:T})$ with $\twonorm{x}=1$, where $U_{t:T}$ are eigenvectors $t$ through $T$ of $M$. Then for $t \leq T$, if $\lambda_t - \lambda_{T+1} > \opnorm{H} - h - 2\opnorm{HU_{t:T}}$,
\begin{align*}
    \lambda_t(M+H) \leq \lambda_t(M) + h + \frac{\opnorm{HU_{t:T}}^2}{\lambda_t(M) - \lambda_{T+1}(M) - \opnorm{H} + h + 2 \opnorm{HU_{t:T}}}.
\end{align*}
\end{thm}

We prove Theorem \ref{thm:general_eval_pert_bound} in Section \ref{sec:eigenvalue_perturbation_bound_proofs}. Notice that compared to Theorem \ref{thm:eldridge_eval_pert}, Theorem \ref{thm:general_eval_pert_bound} has slightly weakened assumptions, and always provides a better bound since $\opnorm{HU_{t:T}} \leq \opnorm{H}$. Especially when $H$ is random and $T$ is small, $\col(U_{t:T})$ is a small subspace and so we might hope that $\opnorm{HU_{t:T}} \ll \opnorm{H}$.

Now we compare these results within the context of the proof of Theorem \ref{thm:main_theorem_recovery_sdp}, and show that the improvement from our Theorem \ref{thm:general_eval_pert_bound} is indeed substantial. In this subsection we also write $a_n \asymp b_n$ if $a_n \lesssim b_n$ and $b_n \lesssim a_n$. Fix a cluster $k$ so that $\frac{p-q}{2}s_k \asymp \kappa \sqrt{pn\log n}$. Let $w^\top$ be the first row of $W^{(k)}$ except with its first entry halved. This is chosen so that by defining $H = e_1w^\top + we_1^\top$, $H$ contains the noise from the first row and column of $W^{(k)}$. Now let $M = \Ashift^{(k)} - H$ (so $M+H = \Ashift^{(k)}$). Our objective will be to bound $\left|\lambda_1\Big(\Ashift^{(k)}\Big) - \lambda_1\left(\Ashift^{(k)}-H\right)\right|$.
By the assumption on the size of the cluster $k$, there is sufficient signal so that the top eigenvector $v_1$ of $M$ satisfies $\infnorm{v_1} \lesssim \frac{1}{\sqrt{s_k}}$ (by applying an $\infty$-norm eigenvector perturbation bound \citep{chen2021spectral}). Then
\[
\left| v_1^\top H v_1 \right| = 2 |v_1^\top e_1 | |w^\top v_1| \lesssim \frac{1}{\sqrt{s_k}} \sqrt{p \log n}
\]
where the bound $|w^\top v_1| \lesssim \sqrt{p \log n}$ holds with high probability using the fact that $w$ is independent from $v_1$ (since $v_1$ is the top eigenvector of a matrix where the first row and column of noise have been set to $0$). The informal guiding principle from \cite{eldridge2018unperturbed} would then give us hope that $\left|\lambda_1\big(\Ashift^{(k)}\big) - \lambda_1\left(\Ashift^{(k)}-H\right)\right| \lesssim \frac{\sqrt{p \log n}}{\sqrt{s_k}}$, which is the order needed within our proof of Theorem~\ref{thm:main_theorem_recovery_sdp}. Unfortunately, Theorem \ref{thm:eldridge_eval_pert} only yields that
\[\lambda_1\left(\Ashift^{(k)}\right) - \lambda_1\left(\Ashift^{(k)}-H\right) \lesssim \frac{\sqrt{p \log n}}{\sqrt{s_k}} + \frac{\opnorm{H}^2}{\kappa \sqrt{pn \log n} - \opnorm{H}} \lesssim  \frac{\left(\sqrt{ps_k}\right)^2}{\kappa \sqrt{pn \log n}} \lesssim \frac{\sqrt{p} s_k}{\kappa \sqrt{n \log n}}\]
since $\opnorm{H} \asymp \twonorm{w} \asymp \sqrt{ps_k}$. However, since
\[
\opnorm{Hv_1} = \twonorm{w e_1^\top v_1 + e_1 w^\top v_1} \leq \twonorm{w} \left|e_1^\top v_1\right| + \left|w^\top v_1\right| \lesssim \sqrt{ps_k} \frac{1}{\sqrt{s_k}} + \sqrt{p \log n} \lesssim \sqrt{p \log n}
\]
our Theorem \ref{thm:general_eval_pert_bound} gives that
\[\lambda_1\left(\Ashift^{(k)}\right) - \lambda_1\left(\Ashift^{(k)}-H\right) \lesssim\frac{\sqrt{p \log n}}{\sqrt{s_k}} + \frac{\left( \sqrt{p \log n}\right)^2}{\kappa \sqrt{pn \log n}} \lesssim \frac{\sqrt{p \log n}}{\sqrt{s_k}} + \frac{\sqrt{p \log n}}{\kappa \sqrt{n}} \lesssim \frac{\sqrt{p \log n}}{\sqrt{s_k}}.\]

\section{Proof of Main Theorem}
\label{sec:proof_of_main_theorem}
Our main Theorem \ref{thm:main_theorem_recovery_sdp} can be decomposed into two parts: first, that the oracle SDPs (\ref{eq:SDP_regularized_oracle}) have solutions which are either zero or rank-one and all-positive, and second, that the recovery SDP solution (\ref{eq:SDP_regularized_original}) has a block-diagonal form where the blocks are the solutions to the oracle SDPs. Thus, we prove the first part in Subsection \ref{sec:oracle_SDP_solns} and the second part in Subsection \ref{sec:recovery_SDP_soln}. The standalone result on clustering with a gap, Theorem \ref{thm:clustering_with_a_gap}, is proved in Subsection \ref{sec:clustering_with_gap}. So as not to interrupt the flow of the proofs, we place most concentration inequalities or their proofs within Subsection \ref{sec:concentration_results}.

We define the shifted adjacency matrix $\Ashift := A - \frac{p+q}{2}J$. We also define the noise matrix $W := A - \E A$.

The following proofs make use of dual variables for the optimization problems (\ref{eq:SDP_regularized_original}) and (\ref{eq:SDP_regularized_oracle}), so we stop to introduce the KKT conditions for these problems.
\begin{lem}
\label{lem:oracle_KKT_and_recovery_KKT}
$\Yhat$ is an optimal solution of the recovery SDP (\ref{eq:SDP_regularized_original}) if and only if there exist $U, L \in \real^{n \times n}$ such that
\begin{subequations}
    \label{eq:original_KKT}
    \begin{gather}
    0 \leq \Yhat \leq 1, \Yhat \succcurlyeq 0 \label{eq:recovery_KKT_primal_feasibility}\\
    \Ashift - \lambda I - U + L \preccurlyeq 0, L \geq 0, U \geq 0 \label{eq:recovery_KKT_dual_feasibility}\\
    L_{ij}\Yhat_{ij} = 0, U_{ij}(\Yhat_{ij}-1) = 0 ~\forall i,j \in [n] \label{eq:recovery_KKT_complementary_slackness_1} \\
    \left( \Ashift - \lambda I - U + L \right) \Yhat = 0_{n \times n}. \label{eq:recovery_KKT_complementary_slackness_2}
\end{gather}
\end{subequations}

Likewise, for any $k \in \{1, \dots, K\}$, $\Yhat^k$ is an optimal solution of the $k$th oracle SDP (\ref{eq:SDP_regularized_oracle}) if and only if there exist $U^k, L^k \in \real^{s_k \times s_k}$ such that
\begin{subequations}
\label{eq:oracle_KKT}
\begin{gather}
    0 \leq \Yhat^k \leq 1, \Yhat^k \succcurlyeq 0 \label{eq:oracle_KKT_primal_feasibility}\\
    \Ashift^{(k)} - \lambda I - U^k + L^k \preccurlyeq 0, L^k \geq 0, U^k \geq 0 \label{eq:oracle_KKT_dual_feasibility}\\
    L^k_{ij}\Yhat^k_{ij} = 0, U^k_{ij}(\Yhat^k_{ij}-1) = 0 ~\forall i,j \in [s_k] \label{eq:oracle_KKT_complementary_slackness_1} \\
    \left( \Ashift^{(k)} - \lambda I - U^k + L^k \right) \Yhat^k = 0_{s_k \times s_k} \label{eq:oracle_KKT_complementary_slackness_2}
\end{gather}
\end{subequations}
\end{lem}
\begin{proof}
First we consider the recovery SDP~\eqref{eq:SDP_regularized_original}. First, we note that the feasible set is compact since $\{Y \in \R^{n \times n} : 0 \leq Y \leq 1\}$ is compact, $\{Y: Y \succcurlyeq 0 \}$ is closed, and the feasible set is the intersection of these two sets. Therefore there exists an optimal solution $\Yhat$.
    By the generalized Slater's condition within \cite[Section 5.9.1]{boyd_convex_2004}, since the SDP~\eqref{eq:SDP_regularized_original} is convex and the feasible point $\frac{1}{2}J_{n \times n}$ satisfies $0 < \frac{1}{2}J_{n \times n} < 1$ and $\frac{1}{2}J_{n \times n} \succ 0$, strong duality holds and thus from \cite[Section 5.9.2]{boyd_convex_2004} the following KKT conditions are necessary and sufficient for $\Yhat$ to be an optimal solution:
    \begin{gather}
    0 \leq \Yhat \leq 1, \Yhat \succcurlyeq 0 \nonumber\\
    \Ashift - \lambda I - U + L \preccurlyeq 0, L \geq 0, U \geq 0, P \succcurlyeq 0 \nonumber\\
    L_{ij}\Yhat_{ij} = 0, U_{ij}(\Yhat_{ij}-1) = 0 ~\forall i,j \in [n]  \nonumber\\
    P \Yhat = 0_{n \times n} \nonumber\\
    -\Ashift + \lambda I + U - L + P = 0\label{eq:recovery_SDP_KKT_stationarity}
\end{gather}
where $U, L, P \in \R^{n \times n}$. Now by eliminating $P$ using~\eqref{eq:recovery_SDP_KKT_stationarity}, we obtain the desired recovery SDP KKT conditions~\eqref{eq:original_KKT}. An identical argument suffices for the oracle SDP KKT conditions~\eqref{eq:oracle_KKT}.
\end{proof}

Lastly, we set the regularization parameter $\lambda$ which appears in both SDPs (\ref{eq:SDP_regularized_original}) and (\ref{eq:SDP_regularized_oracle}) as $\lambda = \kappa \sqrt{p n \log m}+\delta$ for $\kappa$ which is a constant which we will choose later and $\delta \sim \mathrm{Uniform}[0,0.1]$. Conceptually, we can make $\kappa$ arbitrarily large, and throughout the proofs we imagine any constant multiple of $1/\kappa$ as arbitrarily small. We do not use any particular properties of $\delta$ except that almost surely we will have $\kappa \sqrt{p n \log m} \leq \lambda \leq (\kappa+1) \sqrt{p n \log m}$ and $\lambda_1 \left( \Ashift^{(k)}\right) \neq \lambda$ (for all $k\in [K]$).

\subsection{Oracle SDP Solutions}
\label{sec:oracle_SDP_solns}
This subsection is devoted to the proof of the following lemma:
\begin{lem}
\label{lem:oracle_SDP_solns_main_lemma}
    There exist constants $B, \underline{B}, \overline{B}>2$ such that if $\kappa \geq B$, with probability at least $1-O(m^{-3})$, for each cluster $k$, the $k$th oracle SDP~\eqref{eq:SDP_regularized_oracle} has a unique solution $\Yhat^k$, and there exist dual variables $U^k, L^k$ which satisfy the oracle KKT equations~\eqref{eq:oracle_KKT} such that $L^k = 0$ and $U^k \succcurlyeq 0$. Furthermore,
    \begin{enumerate}
    \item if $\frac{p-q}{2}s_k \geq \left(1 + \frac{1}{\overline{B}}\right) \kappa \sqrt{pn \log m} $, then $\Yhat^k = J_{s_k \times s_k} = \onevec \onevec^\top$, 
    \item if $\frac{p-q}{2}s_k \leq \left(1 - \frac{1}{\underline{B}}\right) \kappa \sqrt{pn \log m} $, then $\Yhat^{k} = 0_{s_k \times s_k} = 0 0^\top$,
    \item if $\left(1 - \frac{1}{\underline{B}}\right) \kappa \sqrt{pn \log m} < \frac{p-q}{2}s_k < \left(1 + \frac{1}{\overline{B}}\right) \kappa \sqrt{pn \log m}$, then either $\Yhat^{k} = 0$ or $\Yhat^{k} = y^{k} y^{k\top}$ where $y^k \geq \frac{1}{2}$ entrywise.
\end{enumerate}
\end{lem}

First we establish an event upon which we will condition the rest of the proof, which will be the intersection of the events described in the following Lemmas \ref{lem:op_norm_concentration}, \ref{lem:infty_pert_bound_s5}, \ref{lem:LOO_evec_noise_inner_prod_concentration}, and \ref{lem:onevec_noise_inner_prod_concentration_all_ones_case}.

\begin{lem}
\label{lem:op_norm_concentration}
With probability at least $1-O(m^{-3})$, there exists an absolute constant $B_1$ such that:
\begin{enumerate}
    \item For all clusters $k$ such that $\frac{p-q}{2}s_k \geq \sqrt{pn \log m}$,
    \begin{align}
    \opnorm{W^{(k)}} \leq B_1 \sqrt{ps_k}. \label{eq:noise_op_norm_bound_cluster}
    \end{align}
    \item For all clusters $k$ such that $\frac{p-q}{2}s_k \geq \sqrt{pn \log m}$, for all columns $(W^{(k)}_{:,j})_{j=1}^{s_k}$,
    \begin{align}
        \twonorm{W^{(k)}_{:,j}} \leq B_1\sqrt{ps_k}. \label{eq:noise_two_norm_bound}
    \end{align}
    \item For all clusters $k$,
    \begin{align}
    \opnorm{W^{(k)}} \leq B_1 \sqrt{pn}. \label{eq:noise_op_norm_bound_all_oracle_blocks}
    \end{align}
\end{enumerate}
\end{lem}
Lemma \ref{lem:op_norm_concentration} is proven in Subsection \ref{sec:concentration_results}.

Recall that  for any $n \times n$ matrix $M$ and any $k \in [K]$, we define $M^{(k)}$ to be the $s_k \times s_k$ principal submatrix of $M$ corresponding to the rows and columns associated with cluster $V_k$. We later make use of leave-one-out (LOO) versions of the noise for analysis so we define them now: For each $k \in [K]$ and each $j \in [s_k]$, let $w^{k,j}$ be equal to $W^{(k)}_{:, j}$ (the $j$th column of $W^{(k)}$) but with the $j$th entry divided by two (so that we do not subtract $W^{(k)}_{jj}$ twice). Now define the $j$th LOO noise matrix $W^{k,j}:= W^{(k)} - e_j {w^{k,j}}^\top - {w^{k,j}} e_j^\top$ and the $j$th LOO adjacency matrix 
\[A^{k,j}:= A^{(k)} - e_j {w^{k,j}}^\top - {w^{k,j}} e_j^\top = \E A^{(k)} + W^{k,j}.\]
We also define the $j$th LOO shifted adjacency matrix $\Ashift^{k,j} := A^{k,j} - \frac{p+q}{2}J_{s_k \times s_k}$.

Let $v^k$ be the top eigenvector of $\Ashift^{(k)}$ and let $v^{k,j}$ be the top eigenvector of $\Ashift^{k,j}$ (by top, we mean corresponding to the largest eigenvalue). Of course these eigenvectors are only defined up to a global rotation, but the following lemma resolves this ambiguity and also establishes a key property of $v^k$ and $v^{k,j}$.
\begin{lem}
\label{lem:infty_pert_bound_s5}
There exist absolute constants $B_2, B_3$ such that with probability at least $1 - O(m^{-3})$, 
for all clusters $k$ such that $\frac{p-q}{2}s_k \geq B_3 \sqrt{pn \log m}$,
there exists a top eigenvector $v^k$ of $\Ashift^{(k)}$ such that
\begin{align}
    \infnorm{v^k - \frac{1}{\sqrt{s_k}}\onevec} &\leq 2B_2 \frac{\sqrt{p \log m}}{(p-q)s_k}, \label{eq:infty_pert_bound_orig_s5}
\end{align}
and for each $j \in [s_k]$, there exists a top eigenvector $v^{k,j}$ of $\Ashift^{k,j}$ such that
\begin{align}
    \infnorm{v^{k,j} - \frac{1}{\sqrt{s_k}}\onevec} &\leq 2B_2 \frac{\sqrt{p \log m}}{(p-q)s_k}. \label{eq:infty_pert_bound_loo_s5}
\end{align}
\end{lem}

The proof of Lemma \ref{lem:infty_pert_bound_s5} is in Subsection \ref{sec:concentration_results}, as are the proofs of the following Lemmas \ref{lem:LOO_evec_noise_inner_prod_concentration} and~\ref{lem:onevec_noise_inner_prod_concentration_all_ones_case}.

\begin{lem}
\label{lem:LOO_evec_noise_inner_prod_concentration}
There exist an absolute constants $B_4, B_5$ such that with probability at least $1-O(m^{-3})$, for every cluster $k$ such that $\frac{p-q}{2}s_k \geq B_4 \sqrt{pn \log m}$, for each $j \in [s_k]$,
\begin{align}
    \infnorm{v^{k,j}} \leq \frac{2}{\sqrt{s_k}} \label{eq:LOO_evec_inf_norm_simple_bound}
\end{align}
and
\begin{align}
    \left|\langle w^{k,j}, v^{k,j} \rangle \right| \leq B_5\sqrt{ p \log m}. \label{eq:LOO_evec_left_out_noise_ip_bound}
\end{align}
\end{lem}

\begin{lem}
    \label{lem:onevec_noise_inner_prod_concentration_all_ones_case}
    With probability at least $1-O(m^{-3})$, for all clusters $k$ such that $\frac{p-q}{2}s_k \geq B_4\sqrt{pn \log m}$,
    \begin{align*}
       \infnorm{  W^{(k)} \frac{1}{\sqrt{s_k}} \onevec_{s_k} } \leq B_5 \sqrt{p \log m}.
    \end{align*}
\end{lem}

Define $E_1$ to be the intersection of the events in Lemmas \ref{lem:op_norm_concentration}, \ref{lem:infty_pert_bound_s5}, \ref{lem:LOO_evec_noise_inner_prod_concentration}, and \ref{lem:onevec_noise_inner_prod_concentration_all_ones_case}. By the union bound $\P(E_1) \geq 1 - O(m^{-3})$. For most of the rest of the proof we condition on the event $E_1$ or supersets of $E_1$. Also we assume that $\kappa > B_1$, which implies $\lambda > \opnorm{W^{(k)}}$ for each $k \in [K]$.

We split the proof into three cases related to the size of $s_k$ corresponding to those in Lemma \ref{lem:oracle_SDP_solns_main_lemma}. 
Before considering each case, we remind the reader that we avoid the situation that $\lambda_1\left(\Ashift^{(k)}\right) = \lambda$ with probability one by our addition of the small random continuous perturbation to $\lambda$.

\subsubsection{All-Zero Solution}
First we handle the easiest case where $\lambda_1\left(\Ashift^{(k)}\right) < \lambda$, which will cause the oracle solution to be zero.
\begin{lem}
    \label{lem:oracle_zero_soln_case}
    On the event $E_1 \cap \left\{\lambda_1\left(\Ashift^{(k)}\right) < \lambda\right\}$, we have that $\Yhat^k = 0_{s_k \times s_k}$ is the unique solution of the oracle SDP~\eqref{eq:SDP_regularized_oracle}, and furthermore by setting $L^k = 0_{s_k \times s_k}$ and $U^k = 0_{s_k \times s_k}$, then $\Yhat^k, L^k, U^k$ satisfy the oracle SDP KKT conditions~\eqref{eq:oracle_KKT}. Also, for any cluster $k$ such that $\frac{p-q}{2}s_k \leq \left(1-\frac{B_1}{\kappa}\right) \kappa \sqrt{pn \log m}$, $E_1 = E_1 \cap \left\{\lambda_1\left(\Ashift^{(k)}\right) < \lambda\right\}$.
\end{lem}
\begin{proof}
    First we check $\Yhat^k = 0$ is the unique solution of the oracle SDP~\eqref{eq:SDP_regularized_oracle} when $\lambda_1\left(\Ashift^{(k)}\right)$$ < \lambda$. The oracle SDP objective~\eqref{eq:SDP_regularized_oracle} is
    \begin{align}
        \left\langle Y,A^{(k)}-\frac{p+q}{2}J_{s_k \times s_k}\right\rangle -\lambda\tr(Y) &= \left\langle Y, \Ashift^{(k)} - \lambda I \right\rangle \nonumber\\
        & \leq \sum_{i = 1}^{s_k} \lambda_i \left( Y \right) \lambda_i \left( \Ashift^{(k)} - \lambda I \right) \nonumber\\
        & \leq \lambda_1 \left( Y \right) \lambda_1 \left( \Ashift^{(k)} - \lambda I \right) \label{eq:oracle_zero_soln_obj_upper_bound}
    \end{align}
    where we obtain the first inequality using a Hermitian trace inequality \cite[Theorem 9.H.1.g and the discussion thereafter]{marshall1979inequalities}. We obtain the second inequality from the fact that all terms are non-positive since $\lambda_i \left( \Ashift^{(k)} - \lambda I \right) \leq \lambda_1\left(\Ashift^{(k)}\right) - \lambda < 0$ and $\lambda_i(Y) \geq 0$ since $Y \succcurlyeq 0$ by feasibility to the oracle SDP~\eqref{eq:SDP_regularized_oracle}. Now expression~\eqref{eq:oracle_zero_soln_obj_upper_bound} is strictly less than zero unless we have $\lambda_1 \left( Y \right) = 0$, and again since $Y \succcurlyeq 0$ this implies that $Y=0$. Therefore $\Yhat^k = 0$ is the unique solution of the oracle SDP.

    Next we check that $L^k = 0$ and $U^k = 0$ satisfy the oracle SDP KKT conditions~\eqref{eq:oracle_KKT}. We note that
    \[\Ashift^{(k)} - \lambda I - U^k + L^k = \Ashift^{(k)} - \lambda I \preccurlyeq 0\]
    where the PSD inequality follows from the assumption that we are in the event $\left\{\lambda_1\left(\Ashift^{(k)}\right) < \lambda\right\}$, satisfying the first part of~\eqref{eq:oracle_KKT_dual_feasibility}. All other conditions are trivial to check.
    
    Finally, note that the event $E_1$ contains the conclusions of Lemma \ref{lem:op_norm_concentration}, so by~\eqref{eq:noise_op_norm_bound_all_oracle_blocks} we have for any cluster $k$ that
    \[\lambda_1 \left( W^{(k)} \right) \leq \opnorm{W^{(k)}} \leq B_1 \sqrt{pn} \leq \frac{B_1}{\kappa} \kappa \sqrt{pn \log m}.\]
    Therefore if $\frac{p-q}{2}s_k \leq \left(1-\frac{B_1}{\kappa}\right) \kappa \sqrt{pn \log m}$, then by Weyl's inequality
    \[\lambda_1\left(\Ashift^{(k)}\right) \leq \lambda_1\left( \frac{p-q}{2}J_{s_k \times s_k}\right) + \lambda_1\left( W^{(k)} \right) \leq \frac{p-q}{2}s_k + \frac{B_1}{\kappa} \kappa \sqrt{pn \log m} < \lambda. \]
\end{proof}

\subsubsection{All-One Solution}
We also provide the following lemma summarizing the case that $\frac{p-q}{2}s_k$ is large-enough for the oracle solution to be all-ones.

\begin{lem}
    \label{lem:oracle_zero_ones_case}
    On the event $E_1$, for all clusters $k$ such that $\frac{p-q}{2}s_k \geq \left(1+\frac{B_6}{\kappa}\right)\kappa \sqrt{pn \log m}$, we have that $\Yhat^k = J_{s_k \times s_k}$ is the unique solution of the oracle SDP~\eqref{eq:SDP_regularized_oracle}, and furthermore there exist some $U^k, L^k$ such that $\Yhat^k, L^k, U^k$ satisfy the oracle SDP KKT conditions~\eqref{eq:oracle_KKT}, $L^k = 0$, and $U^k \succcurlyeq 0$.
\end{lem}
Since this essentially follows as a special case of Theorem \ref{thm:clustering_with_a_gap} (but with a larger $\lambda$), we prove Lemma \ref{lem:oracle_zero_ones_case} in Subsection \ref{sec:clustering_with_gap}.

\subsubsection{Nonzero Critical Size Solution}
Finally we consider the most challenging case, where 
\begin{gather}
    \left(1-\frac{B_1}{\kappa}\right)\kappa \sqrt{pn \log m} < \frac{p-q}{2}s_k < \left(1+\frac{B_6}{\kappa}\right)\kappa \sqrt{pn \log m}\label{eq:critical_signal_lambda_relationship_s5} \\
    \text{and} \quad \lambda_1\left(\Ashift^{(k)}\right) > \lambda.
\end{gather}
To analyze the oracle SDP in this case, we consider a new \textit{relaxed oracle SDP} where all lower-bound constraints are removed, which enables us to prove that its unique solution has rank $1$. Then we will analyze the relaxed oracle SDP and show that its solution is non-negative, which then implies that its solution is also the solution of the corresponding oracle SDP, and therefore the original oracle SDP has a unique solution with rank $\leq 1$. First we define the $k$th relaxed oracle SDP:
\begin{subequations}
\label{eq:SDP_relaxed_oracle}
\begin{align}
\Yhats^k=\argmax_{Y\in\real^{s_k\times s_k}}\quad & \left\langle Y,A^{(k)}-\frac{p+q}{2}J_{s_k \times s_k}\right\rangle -\lambda\tr(Y)\label{eq:SDP_relaxed_oracle_1}\\
\text{s.t.}\quad & Y\succcurlyeq0,\label{eq:SDP_relaxed_oracle_2}\\
 &  Y_{ii}\le1,\forall i\in[s_i].\label{eq:SDP_relaxed_oracle_3}
\end{align}
\end{subequations}
Beyond removing the lower bound constraints, the relaxed oracle SDP also has fewer upper bound constraints (\ref{eq:SDP_relaxed_oracle_3}), which when combined with the PSD constraint (\ref{eq:SDP_relaxed_oracle_2}) are equivalent to upper-bounding each entry of the solution, but having the reduced constraints are useful for showing that the relaxed oracle SDP solution $\Yhats^k$ must have rank $1$.

To show that $\Yhats^k$ is entrywise positive despite the removal of the lower bound constraints, we make use of $s_k$ \textit{leave-one-out (LOO) relaxed oracle SDPs} wherein one row and column of the noise are zeroed out. The key property of the LOO relaxed oracle SDPs is that we will be able to show that (for each $k,j$) their solutions have low correlation with the corresponding left-out noise vectors $w^{k,j}$, due to their independence. We would like to define a second high-probability event $E_2$ under which this occurs, but we cannot simply work under the event $E_1$ to establish that $E_2$ has high probability, since $E_1$ is not independent of the left-out noise. To remedy this we define new high-probability events.

For each cluster $k$ such that $\frac{p-q}{2}s_k \geq B_4 \sqrt{pn \log m}$, for each $j \in [s_k]$, define the event 
\[
E_1^{k,j} = \left \{\begin{aligned}
    \opnorm{W^{k,j}} &\leq B_1 \sqrt{ps_k} \\
    \infnorm{v^{k,j} - \frac{1}{\sqrt{s_k}}\onevec} &\leq 2B_2 \frac{\sqrt{p \log m}}{(p-q)s_k} 
\end{aligned} \right\}.
\]
These conditions are~\eqref{eq:noise_op_norm_bound_cluster} and~\eqref{eq:infty_pert_bound_loo_s5} which are both included in $E_1$, so we have that $E_1^{k,j} \supseteq E_1$. Also none of these conditions involve the $(k,j)$th left-out noise vector $w^{k,j}$, so event $E_1^{k,j}$ is independent of $w^{k,j}$.

Before defining the LOO relaxed oracle SDP, a second issue is that even when $\lambda_1\left(\Ashift^{(k)}\right) > \lambda$, we may have $\lambda_1\left(\Ashift^{k,j}\right) < \lambda$ after removing noise, in which case if we used the same regularization parameter $\lambda$ in the LOO relaxed oracle SDP, it would have a zero solution which would not be useful.
To resolve this issue, we will later show that under the event $E_1$, for all clusters $k$ such that $\frac{p-q}{2}s_k \geq B_4 \sqrt{pn \log m}$ (which includes all clusters in the critical regime~\eqref{eq:critical_signal_lambda_relationship_s5} assuming $\kappa$ is sufficiently large), $\lambda_1\left(\Ashift^{k,j}\right)$ is very close to $\lambda_1\left(\Ashift^{(k)}\right)$. 
This bound is of course dependent on the left-out noise $w^{k,j}$ so we state it later, but for now we mention this consideration in order to motivate our choice to define the LOO relaxed oracle SDPs with reduced regularization parameters $\lambda^k$.

For each $k$ such that $\frac{p-q}{2}s_k \geq B_4 \sqrt{pn \log m}$, we define the modified regularization parameter $\lambda^{k} := \lambda - \frac{B_8 \sqrt{p \log m}}{\sqrt{s_k}}$. Also we add the requirement that $\kappa$ is large enough so that under $E_1^{k,j}$ we are ensured $\lambda^k > \opnorm{W^{k,j}}$. Finally we can define the leave-one-out SDPs. For all $k$ such that $\frac{p-q}{2}s_k \geq B_4 \sqrt{pn \log m}$ and each $j \in [s_k]$, we define the $(k,j)$th LOO relaxed oracle SDP 
\begin{subequations}
\label{eq:SDP_relaxed_oracle_LOO}
\begin{align}
\Yhats^{k,j}=\argmax_{Y\in\real^{s_k\times s_k}}\quad & \left\langle Y,A^{k,j}-\frac{p+q}{2}J_{s_k \times s_k}\right\rangle -\lambda^{k} \tr(Y)\label{eq:SDP_relaxed_oracle_LOO_1}\\
\text{s.t.}\quad & Y\succcurlyeq0,\label{eq:SDP_relaxed_oracle_LOO_2}\\
 &  Y_{ii}\le1,\forall i\in[s_k].\label{eq:SDP_relaxed_oracle_LOO_3}
\end{align}
\end{subequations}

Now we are prepared to show that the relaxed oracle SDPs (\ref{eq:SDP_relaxed_oracle}) and the LOO relaxed oracle SDPs (\ref{eq:SDP_relaxed_oracle_LOO}) have rank-one solutions.

\begin{lem}
\label{lem:SDP_relaxed_oracle_rank_one}
Under the event $E_1 \cap \left\{\lambda_1\left(\Ashift^{(k)}\right) > \lambda\right\}$, the $k$th relaxed oracle SDP~\eqref{eq:SDP_relaxed_oracle} has rank-one solution $\Yhats^k =\yhats^k {\yhats^k}^\top$ where $\onevec^\top \yhats^k \geq 0$, and there exists $u^k \in \R^{s_k}$ such that:
\begin{subequations}
\label{eq:rank_one_relaxed_oracle_KKT}
\begin{gather}
    -1 \leq \yhats^k \leq 1 \label{eq:rank_one_relaxed_oracle_KKT_primal_feasibility}\\
    \Ashift^{(k)} - \lambda I - \diag(u^k) \preccurlyeq 0, u^k \geq 0 \label{eq:rank_one_relaxed_oracle_KKT_dual_feasibility}\\
    \left(\left(\yhats_i^k\right)^2 - 1\right) u^k_i = 0 ~\forall i \in [s_k]  \label{eq:rank_one_relaxed_oracle_KKT_complementary_slackness_1}\\
    \left( \Ashift^{(k)} - \lambda I - \diag(u^k)  \right) \yhats^k = 0. \label{eq:rank_one_relaxed_oracle_KKT_complementary_slackness_2}
\end{gather}
\end{subequations}

Also, for each $k$ such that $\frac{p-q}{2}s_k \geq B_4 \sqrt{pn \log m}$, for each $j \in [s_k]$, under the event $E_1^{k,j} \cap \left\{\lambda_1\left(\Ashift^{k,j}\right) > \lambda^k\right\}$, the $(k,j)$th LOO relaxed oracle SDP~\eqref{eq:SDP_relaxed_oracle_LOO} has rank-one solution $\Yhats^{k,j} =\yhats^{k,j} {\yhats^{k,j}}^\top$ where $\onevec^\top \yhats^{k,j} \geq 0$, and there exists $u^{k,j} \in \R^{s_k}$ such that
\begin{subequations}
\label{eq:rank_one_LOO_relaxed_oracle_KKT}
\begin{gather}
    -1 \leq \yhats^{k,j} \leq 1 \\
    \Ashift^{k,j} - \lambda I - \diag(u^{k,j}) \preccurlyeq 0,  u^{k,j} \geq 0\label{eq:rank_one_LOO_relaxed_oracle_KKT_dual_feasibility} \\
    \left( \left(\yhats^{k,j}_i\right)^2 - 1 \right) u^{k,j}_i = 0 ~\forall i \in [s_k] \label{eq:rank_one_LOO_relaxed_oracle_KKT_complementary_slackness_1} \\
    \left( \Ashift^{k,j} - \lambda I - \diag(u^{k,j})  \right) \yhats^{k,j} = 0. \label{eq:rank_one_LOO_relaxed_oracle_KKT_complementary_slackness_2}
\end{gather}
\end{subequations}
\end{lem}
\begin{proof}
We start by writing the KKT conditions for the relaxed oracle SDP (\ref{eq:SDP_relaxed_oracle}). By a nearly identical argument as in Lemma \ref{lem:oracle_KKT_and_recovery_KKT}, for $\Yhats^i$ to be optimal there must exist $u^k \in \real^{s_k}$ such that
\begin{subequations}
\label{eq:relaxed_oracle_KKT_general_rank}
    \begin{gather}
    \Yhats^k \succcurlyeq 0, \Yhats^k_{ii} \leq 1 ~\forall i \in [s_k] \label{eq:relaxed_oracle_KKT_general_rank_primal_feasibility}\\
    u^k \geq 0,   \Ashift^{(k)} - \lambda I - \diag(u^k) \preccurlyeq 0 \\
    ((\Yhats^k)_{ii} - 1) u^k_i = 0 ~\forall i \in [s_k] \label{eq:relaxed_oracle_KKT_general_rank_complementary_slackness_1} \\
    \left( \Ashift^{(k)} - \lambda I - \diag(u^k) \right) \Yhats^k = 0_{s_k \times s_k}. \label{eq:relaxed_oracle_KKT_general_rank_complementary_slackness_2}
\end{gather}
\end{subequations}

Note $\diag(u^k) \succcurlyeq 0$ since $u^k \geq 0$. Recall that we are working under an event where $\lambda_1\left(\Ashift^{(k)}\right) > \lambda$. Notice \[\Ashift^{(k)} - \lambda I - \diag(u^k) = \frac{p-q}{2}J_{s_k \times s_k} +  W^{(k)}- \lambda I - \diag(u^k),\] and $\frac{p-q}{2}J_{s_k \times s_k}$ is rank $1$. Therefore by Weyl's inequality we have for any $i \in [s_k]$ that
\begin{align*}
    \lambda_i\left(\Ashift^{(k)} - \lambda I - \diag(u^k)\right) &= \lambda_i\left(\frac{p-q}{2}J_{s_k \times s_k} +  W^{(k)}- \lambda I - \diag(u^k)\right) \\
    & \leq \lambda_i\left(\frac{p-q}{2}J_{s_k \times s_k} \right) + \lambda_1\left(W^{(k)}- \lambda I - \diag(u^k)\right) \\
    & < \lambda_i\left(\frac{p-q}{2}J_{s_k \times s_k} \right)
\end{align*}
because $W^{(k)} \prec \lambda I$. Since $\lambda_i\left(\frac{p-q}{2}J_{s_k \times s_k} \right) = 0$ for $i \geq 2$, we have that $ \lambda_i\left(\Ashift^{(k)} - \lambda I - \diag(u^k)\right) < 0$ for $i \geq 2$. Now applying a Hermitian version of von Neumann's trace inequality \cite[Theorem 9.H.1.g and the discussion thereafter]{marshall1979inequalities}, we have
\begin{align}
    \tr\left( \Ashift^{(k)} - \lambda I - \diag(u^k) \right) \Yhats^k &\leq \sum_{i=1}^{s_k} \lambda_{i}\left(\Ashift^{(k)} - \lambda I - \diag(u^k)\right)\lambda_{i}(\Yhats^k) \nonumber \\
    & \leq \sum_{i=2}^{s_k} \lambda_{i}\left(\Ashift^{(k)} - \lambda I - \diag(u^k\right)\lambda_{i}(\Yhats^k) \label{eq:relaxed_oracle_complementarity_trace_bound}
\end{align}
where in the second line we use the fact that $\lambda_1\left(\Ashift^{(k)} - \lambda I - \diag(u^k)\right) \leq 0$ and $\lambda_1(\Yhats^k) \geq 0$. Now the only way for~\eqref{eq:relaxed_oracle_complementarity_trace_bound} to be $0$ and not $< 0$ is if $\lambda_{i}(\Yhats^k) = 0$ for all $i \geq 2$. Therefore $\Yhats^k$ has rank at most $1$ in this case. But also $\Yhats^k$ cannot be $0$, since $0$ has a lower objective value than $v^k {v^k}^\top$ (the top eigenvector of $\Ashift^{(k)}$) since \[\left \langle v^k {v^k}^\top , \Ashift^{(k)} \right \rangle - \lambda \tr(v^k {v^k}^\top  )= {v^k}^\top  \Ashift^{(k)} v^k -\lambda > 0.\] Thus $\Yhats^k$ must be rank $1$.

Now letting $\Yhats^k =\yhats^k {\yhats^k}^\top$, to resolve the sign ambiguity of $\yhats^k$ we define $\onevec^\top \yhats^k \geq 0$. Finally we establish the optimality conditions~\eqref{eq:rank_one_relaxed_oracle_KKT}. \eqref{eq:relaxed_oracle_KKT_general_rank_primal_feasibility} follows from~\eqref{eq:relaxed_oracle_KKT_general_rank_primal_feasibility} once we note that $\Yhats^k_{ii} = \left(\yhats^k_i\right)^2$. Similarly~\eqref{eq:rank_one_relaxed_oracle_KKT_complementary_slackness_1} follows from~\eqref{eq:relaxed_oracle_KKT_general_rank_complementary_slackness_1}. \eqref{eq:rank_one_relaxed_oracle_KKT_complementary_slackness_2} follows from~\eqref{eq:relaxed_oracle_KKT_general_rank_complementary_slackness_2} by right-multiplying by $\yhats^k$ and noting that $\twonorm{\yhats^k} \neq 0$ since $\yhats^k \neq 0$.

A completely analogous argument suffices to prove the analogous statements about the LOO relaxed oracle SDP solutions.
\end{proof}

Next we are able to show that the LOO relaxed oracle SDP solutions $\yhats^{k,j}$ have low correlation with the corresponding left-out noise vectors $w^{k,j}$.

\begin{lem}
\label{lem:noise_LOO_soln_inner_product_concentration}
With probability at least $1 - O(m^{-3})$, for all clusters $k$ satisfying $\frac{p-q}{2}s_k \geq B_4\sqrt{pn \log m}$, for each $j \in [s_k]$,
\begin{align*}
    \left| \langle w^{k,j}, \yhats^{k,j} \rangle \right| \leq B_5\sqrt{ p s_k\log m}.
\end{align*}
\end{lem}

We define the event in Lemma \ref{lem:noise_LOO_soln_inner_product_concentration} to be $E_2$. We let $E_3 = E_1 \cap E_2$, and note by the union bound that $\P(E_3) \geq 1- O(m^{-3})$. From here on the proof will be deterministic. First we check that the LOO relaxed oracle problems have nonzero solutions when the corresponding relaxed oracle problem has a nonzero solution.

\begin{lem}
\label{lem:eval_pert_bound_s5}
Under the event $E_1$, for all clusters $k$ such that $\frac{p-q}{2}s_k \geq B_4 \sqrt{pn \log m}$, 
for each $j \in [s_k]$,
    \[\left| \lambda_1\left(\Ashift^{(k)}\right) - \lambda_1\left(\Ashift^{k,j}\right) \right| \leq \frac{B_8 \sqrt{p \log m}}{\sqrt{s_k}},\]
and therefore by the choice of $\lambda^k$, for any $k$ such that $\lambda_1 \left( \Ashift^{(k)}\right) > \lambda$, we have $\lambda_1 \left( \Ashift^{k,j}\right) > \lambda^k$ for all $j \in [s_k]$.
\end{lem}
\begin{proof} 
Fix a cluster $k$ such that $\frac{p-q}{2}s_k \geq B_4 \sqrt{pn \log m}$.
We can show the desired bound for arbitrary $j$. Let $\wtilde = w^{k,j}$, $\vtilde = v^{k,j}$, and $\Wtilde = W^{k,j}$. Note $\opnorm{\Wtilde} \leq \opnorm{W^{(k)}} \leq B_1 \sqrt{ps_k}$. Let $\Delta = \wtilde e_j^\top + e_j \wtilde^\top$ and $\Atilde = \Ashift^{k,j} = \Ashift^{(k)} - \Delta$. First we bound $|\vtilde^\top \Delta \vtilde|$ and $\twonorm{\Delta \vtilde}$.

Using Lemma \ref{lem:LOO_evec_noise_inner_prod_concentration},
\begin{align*}
    \left|\vtilde^\top \Delta \vtilde\right| = 2|\wtilde^\top \vtilde| |\vtilde_j | \leq \frac{2B_5\sqrt{ p \log m}}{\sqrt{s_k}}.
\end{align*}
Next, again using Lemma \ref{lem:LOO_evec_noise_inner_prod_concentration},
\begin{align*}
    \twonorm{\Delta \vtilde} = \twonorm{\wtilde \vtilde_j + e_j \wtilde^\top \vtilde} \leq \twonorm{\wtilde} \infnorm{\vtilde} + |\wtilde^\top \vtilde| \leq B_1 \sqrt{ps_k} \frac{2}{\sqrt{s_k}} + B_5\sqrt{ p \log m} \leq B_7 \sqrt{p \log m}.
\end{align*}
Now we can use these bounds in the eigenvalue perturbation inequalities. By Theorem \ref{thm:general_eval_pert_bound} we have
\begin{align*}
    \lambda_1\left(\Ashift^{(k)}\right) &= \lambda_1\left(\Atilde + \Delta\right) \\
    &\leq \lambda_1\left(\Atilde\right) + |\vtilde^\top \Delta \vtilde| + \frac{\twonorm{\Delta \vtilde}^2}{\lambda_1\left(\Atilde\right) - \lambda_2\left(\Atilde\right) - \opnorm{\Delta}} \\
    & \leq \lambda_1\left(\Atilde\right) + \frac{2B_5\sqrt{ p \log m}}{\sqrt{s_k}} + \frac{B_7^2 p \log m}{\frac{p-q}{2}s_k - 4B_1\sqrt{ps_k}} \\
    & \leq \lambda_1\left(\Atilde\right) + \frac{2B_5\sqrt{ p \log m}}{\sqrt{s_k}} + \frac{B_7^2 p \log m}{ (\kappa-B_1) \sqrt{pn \log m} - 5B_1\sqrt{ps_k}} \\
    & \leq \lambda_1\left(\Atilde\right) + \frac{2B_5\sqrt{ p \log m}}{\sqrt{s_k}} + \frac{B_7^2 p \log m}{ (\kappa-6B_1) \sqrt{pn \log m}} \\
    & \leq \lambda_1\left(\Atilde\right) + \frac{2B_5\sqrt{ p \log m}}{\sqrt{s_k}} + \frac{B_7^2 p \log m}{\sqrt{pn \log m}} \\
    & \leq \lambda_1\left(\Atilde\right) + \frac{B_8 \sqrt{p \log m}}{\sqrt{s_k}}.
\end{align*}
In the second inequality we used our bounds for $|\vtilde^\top \Delta \vtilde|$ and $\twonorm{\Delta \vtilde}$, as well as the facts that $\opnorm{\Delta} \leq 2 \twonorm{\wtilde} \leq 2B_1 \sqrt{ps_k}$, $\lambda_1\left(\Atilde\right) \geq \frac{p-q}{2}s_k - \opnorm{\Wtilde} \geq \frac{p-q}{2}s_k - B_1\sqrt{ps_k}$, and $\lambda_2\left(\Atilde\right) \leq \opnorm{\Wtilde} \leq B_1 \sqrt{ps_k}$. In the third inequality we use the fact that $\frac{p-q}{2}s_k > \left(1-\frac{B_1}{\kappa}\right)\kappa \sqrt{pn \log m} = (\kappa - B_1)\sqrt{pn \log m}$ from~\eqref{eq:critical_signal_lambda_relationship_s5}. Finally we assume $\kappa$ is large enough so that $ \kappa > 6B_1 $ in the fourth inequality.

For the other direction, we can simply use \cite[Theorem 5]{eldridge2018unperturbed} and our bound on $|\vtilde^\top \Delta \vtilde|$ to conclude
\begin{align*}
    \lambda_1\left(\Ashift^{(k)}\right) &= \lambda_1\left(\Atilde + \Delta\right) 
    \geq \lambda_1\left(\Atilde\right) - |\vtilde^\top \Delta \vtilde| \geq \lambda_1\left(\Atilde\right) - \frac{2B_5\sqrt{ p \log m}}{\sqrt{s_k}} \geq \lambda_1\left(\Atilde\right) - \frac{B_8 \sqrt{p \log m}}{\sqrt{s_k}}.
\end{align*}
\end{proof}
Now we prove two preliminary facts about the relaxed oracle SDP and the LOO relaxed oracle SDPs. Note that Lemma \ref{lem:eval_pert_bound_s5} states that the event $E_1 \cap \left\{\lambda_1\left(\Ashift^{(k)}\right) > \lambda\right\} $ implies the event $E_1^{k,j} \cap \left\{\lambda_1\left(\Ashift^{k,j}\right) > \lambda^k\right\}$, which could slightly simplify the statements of Lemmas \ref{lem:relaxed_oracle_norm_lower_bound} and \ref{lem:relaxed_oracle_inner_prod_with_onevec_lower_bound} below.

\begin{lem}
\label{lem:relaxed_oracle_norm_lower_bound}
Suppose $k$ satisfies the condition~\eqref{eq:critical_signal_lambda_relationship_s5}. Under the event $E_1 \cap \left\{\lambda_1\left(\Ashift^{(k)}\right) > \lambda\right\}$, we have 
\begin{align*}
    \twonorm{\yhats^k} \geq \sqrt{s_k}  \left( 1 - \frac{B_2}{\kappa - B_1} \right).
\end{align*}
Also, under the event $E_1^{k,j} \cap \left\{\lambda_1\left(\Ashift^{k,j}\right) > \lambda^k\right\}$, we have
\begin{align*}
    \twonorm{\yhats^{k,j}} \geq \sqrt{s_k}  \left( 1 - \frac{B_2}{\kappa - B_1} \right).
\end{align*}
\end{lem}
\begin{proof}
Recalling $v^k$ as the top eigenvector of $\Ashift^{(k)}$, $\frac{v^k}{\infnorm{v^k}}$ is feasible as all its entries are in $[-1,1]$. Furthermore, for rank-one solutions $xx^\top$, setting $x = \frac{v^k}{\infnorm{v^k}}$ is optimal among all vectors $x$ with norm $\leq \frac{1}{\infnorm{v^k}}$ (even comparing to $x$ with $\infnorm{x} > 1$) since $\frac{v^k}{\infnorm{v^k}}$ is a rescaled top eigenvector of $\Ashift^{(k)}$ and $\left \langle \Ashift^{(k)}, xx^\top \right \rangle = x^\top \Ashift^{(k)} x$. So we must have $\twonorm{\yhats^k} \geq \frac{1}{\infnorm{v^k}}$, as the optimal solution must have larger $2$-norm than $\frac{v^k}{\infnorm{v^k}}$. Now we use Lemma \ref{lem:infty_pert_bound_s5} to upper-bound $\infnorm{v^k}$.
 
By (\ref{eq:infty_pert_bound_orig_s5}) we have
\begin{align*}
    \frac{1}{\infnorm{v^k}} &\geq \frac{1}{\frac{1}{\sqrt{s_k}} + 2B_2 \frac{\sqrt{p \log m}}{(p-q)s_k}} \\
    &\geq \frac{\sqrt{s_k}}{1 + 2B_2 \frac{\sqrt{ps_k \log m}}{(p-q)s_k}} \\
    &\geq \sqrt{s_k} \left( 1 - 2B_2 \frac{\sqrt{ps_k \log m}}{(p-q)s_k} \right) \\
    & \geq \sqrt{s_k} \left( 1 - \frac{B_2}{\kappa - B_1} \right)
\end{align*}
where we used the fact that $\frac{1}{1+x} \geq 1-x$ and then that
\[\left(1 - \frac{B_1}{\kappa} \right) \kappa \sqrt{ps_k \log m} \leq \left(1 - \frac{B_1}{\kappa}  \right) \kappa \sqrt{pn \log m} < \frac{p-q}{2}s_k\]
using~\eqref{eq:critical_signal_lambda_relationship_s5}, which implies 
\[ \frac{2 \sqrt{p s_k \log m}}{(p-q)s_k} \leq \frac{1}{\left(1 - \frac{B_1}{\kappa} \right) \kappa} = \frac{1}{\kappa-B_1}. \]
Therefore
\begin{align*}
    \twonorm{\yhats^k} \geq \sqrt{s_k}  \left( 1 - \frac{B_2}{\kappa - B_1} \right).
\end{align*}

By simply starting with (\ref{eq:infty_pert_bound_loo_s5}), which is part of $E_1^{k,j}$, instead of $(\ref{eq:infty_pert_bound_orig_s5})$, we have the same bound for $\yhats^{k,j}$.
\end{proof}

\begin{lem}
\label{lem:relaxed_oracle_inner_prod_with_onevec_lower_bound}
Suppose $k$ satisfies the condition~\eqref{eq:critical_signal_lambda_relationship_s5}. Under the event $E_1 \cap \left\{\lambda_1\left(\Ashift^{(k)}\right) > \lambda\right\}$, we have 
\begin{align*}
    \eb^\top \yhats^k &\geq \sqrt{s_k}\left(1 -   \frac{B_9}{\kappa-B_1} \right).
\end{align*}
Also, under the event $E_1^{k,j} \cap \left\{\lambda_1\left(\Ashift^{k,j}\right) > \lambda^k\right\}$, we have
\begin{align*}
    \eb^\top \yhats^{k,j} &\geq \sqrt{s_k}\left(1 -   \frac{B_9}{\kappa-B_1} \right).
\end{align*}
\end{lem}
\begin{proof}
For convenience we abbreviate $y := \yhats^k$.
Since $\lambda_1\left(\Ashift^{(k)} - \lambda I\right)>0$,
\begin{align*}
    y^\top \left(\Ashift^{(k)} - \lambda I\right) y &= \frac{p-q}{2}(\onevec^\top y)^2 +  y^\top W y - \lambda \twonorm{y}^2 \geq 0
\end{align*}
so starting by rearranging this inequality, we have
\begin{align*}
    \frac{p-q}{2}(\onevec^\top y)^2 & \geq \lambda \twonorm{y}^2 -y^\top W^{(k)} y \\
    & \geq \lambda \twonorm{y}^2 - \opnorm{W^{(k)}}\twonorm{y}^2 \\
    & \geq \kappa \sqrt{pn \log m} \twonorm{y}^2 - B_1 \sqrt{p s_k} \twonorm{y}^2 \\
    & = \left(1 - \frac{B_1\sqrt{ps_k}}{\kappa \sqrt{pn \log m}} \right)\kappa \sqrt{pn \log m} \twonorm{y}^2 \\
    &\geq \left(1- \frac{B_1}{\kappa }\right)\kappa \sqrt{pn \log m} \twonorm{y}^2 \\
    &\geq \left(1- \frac{B_1}{\kappa }\right)\left(1 - \frac{B_6}{\kappa} \right) \frac{p-q}{2}s_k \twonorm{y}^2 \\
    & \geq \left(1- \frac{B_1 + B_6}{\kappa }\right) \frac{p-q}{2}s_k \twonorm{y}^2
\end{align*}
where we used the fact that under $E_1$ we have $\opnorm{W^{(k)}} \leq B_1\sqrt{ps_k}$, and then finally that
\[\kappa \sqrt{pn \log m} \geq \frac{p-q}{2}s_k \left(1 + \frac{B_6}{\kappa} \right)^{-1}  \geq \frac{p-q}{2}s_k\left(1 - \frac{B_6}{\kappa} \right)  \]
which follows from condition~\eqref{eq:critical_signal_lambda_relationship_s5} and the fact that $\frac{1}{1+x} \geq 1-x$.
Dividing both sides by $\frac{p-q}{2}s_k$, noting $ \eb = \onevec / \sqrt{s_k} $, and using Lemma \ref{lem:relaxed_oracle_norm_lower_bound} we have
\begin{align*}
    (\eb^\top y)^2 &\geq \left(1- \frac{B_1 + B_6}{\kappa }\right) \twonorm{y}^2 \\
    &\geq \left(1- \frac{B_1 + B_6}{\kappa }\right) \left(\sqrt{s_k}  \left( 1 - \frac{B_2}{\kappa - B_1} \right) \right)^2 \\
    & \geq s_k \left(1-\frac{B_1 + B_6 + 2B_2}{\kappa - B_1}\right) .
\end{align*}
Taking square roots and using the fact that $\sqrt{1-x} \geq 1-x$ we obtain $\eb^\top y \geq \sqrt{s_k}\left(1 -  \frac{B_1 + B_6 + 2B_2}{\kappa-B_1} \right)$.

A nearly identical argument works for $\twonorm{\yhats^{k,j}}$ (note $\opnorm{W^{k,j}} \leq B_1 \sqrt{ps_k}$ is guaranteed under $E_1^{k,j}$) except that $\lambda^k$ is slightly smaller than $\lambda$. Here we can simply bound very loosely
\[\lambda^k = \lambda - \frac{B_8 \sqrt{p \log m}}{\sqrt{s_k}} \geq \kappa \sqrt{pn \log m} - B_8 \sqrt{pn \log m} = \left(1 - \frac{B_8}{\kappa} \right) \kappa \sqrt{pn \log m},\]
which will simply add a term of $B_8$ to the numerator above, leading to the bound
\[\eb^\top \yhats^{k,j} \geq \sqrt{s_k}\left(1 -  \frac{B_8 + B_1 + B_6 + 2B_2}{\kappa-B_1} \right).\]
Lastly we set $B_9 = B_8 + B_1 + B_6 + 2B_2$.
\end{proof}

Finally we are able to show that the relaxed oracle SDP solution (\ref{eq:SDP_relaxed_oracle}) is all-positive, and thus it is also a solution of the (un-relaxed) oracle SDP (\ref{eq:SDP_regularized_oracle}). For notational convenience, we will only show that the first entry of $\yhats^{k}$ is positive, since the same argument applies to all other entries. Thus we will abbreviate the rank-one relaxed oracle SDP solution by $y := \yhats^k$ and the rank-one 1st LOO relaxed oracle solution by $\ytilde := \yhats^{k,1}$, and similarly we let $Y := \Yhats^k$, $\Ytilde := \Yhats^{k,1}, u:=u^k$, and $\utilde:=u^{k,j}$. We also let $\lambdatilde = \lambda^{k}$, $\Wtilde := W^{k,1}$, and $\wtilde := w^{k,1}$. Since we are working with a fixed $k$ we will also abbreviate $s := s_k$.

Now we can show one of the key properties of the LOO relaxed oracle solution, which is that $\ytilde_1$ is close to $1$.
\begin{lem}
\label{lem:loo_close_to_1}
Supposing cluster $k$ satisfies the condition~\eqref{eq:critical_signal_lambda_relationship_s5}, under the event $E_3 \cap \left\{\lambda_1\left(\Ashift^{(k)}\right) > \lambda\right\}$,
\begin{align*}
    \ytilde_1 \geq 1 - B_{10} \frac{1}{\kappa - 1}.
\end{align*}
\end{lem}
\begin{proof}
By the complementary slackness conditions~\eqref{eq:rank_one_LOO_relaxed_oracle_KKT_complementary_slackness_2} for the LOO relaxed oracle SDP,
\begin{align*}
    0 &= \utilde \circ \ytilde - \frac{p-q}{2} \onevec \onevec^\top \ytilde - \Wtilde \ytilde+\lambdatilde \ytilde
\end{align*}
so
\begin{align*}
    \ytilde_{1} & =\frac{1}{\lambdatilde}\left(\frac{p-q}{2}\onevec^{\top}\ytilde+ e_1^\top \Wtilde \ytilde -\utilde_{1}\ytilde_1\right)\\
 & =\frac{1}{\lambdatilde}\left(\frac{p-q}{2}\onevec^{\top}\ytilde-\utilde_{1}\ytilde_1\right).
\end{align*}
Now either $\ytilde_1 = 1$ and we are done, or otherwise $\utilde_{1}\ytilde_1 \leq 0$ by complementary slackness, in which case we have
\begin{align*}
    \ytilde_{1} & \geq \frac{1}{\lambdatilde}\frac{p-q}{2}\onevec^{\top}\ytilde \\
    &=  \frac{1}{\lambdatilde}\frac{p-q}{2}\sqrt{s}\eb^{\top}\ytilde\\
    &\geq \frac{1}{\kappa \sqrt{pn \log m}} \frac{p-q}{2}\sqrt{s} \sqrt{s}\left(1 -   \frac{B_9}{\kappa-B_1} \right) \\
    & \geq \left(1 -   \frac{B_1}{\kappa} \right) \left(1 -   \frac{B_9}{\kappa-B_1} \right) \\
    & \geq 1 - \frac{B_1 + B_9}{\kappa - B_1}
 \end{align*}
where we used Lemma \ref{lem:relaxed_oracle_inner_prod_with_onevec_lower_bound}, the fact that $\frac{1}{\lambdatilde} \geq \frac{1}{(\kappa+1) \sqrt{pn \log m}}$ since $\lambdatilde \leq \lambda$,
and then the fact that $\left(1 - \frac{B_1}{\kappa } \right) \kappa\sqrt{pn \log m} < \frac{p-q}{2}s$ from condition~\eqref{eq:critical_signal_lambda_relationship_s5}.
\end{proof}

Now let $y = \alpha \ytilde + \beta t$, where $t$ is a unit vector orthogonal to $\ytilde$. Also we may assume for convenience that $\beta \geq 0$ by possibly scaling $t$ by $-1$. If we can show that $\alpha$ is close to $1$ and $\beta$ is small, then in light of Lemma \ref{lem:loo_close_to_1}, we will succeed in showing that $y_1$ is also close to $1$ and thus positive.
First we show $\alpha$ is close to $1$.
\begin{lem}
\label{lem:two_sided_alpha_bound}
Supposing cluster $k$ satisfies the condition~\eqref{eq:critical_signal_lambda_relationship_s5}, under the event $E_3 \cap \left\{\lambda_1\left(\Ashift^{(k)}\right) > \lambda\right\}$,
\begin{align*}
    1 - \frac{B_{11}}{\kappa-B_1} \leq \alpha \leq 1 + \frac{B_{11}}{\kappa - B_1}.
\end{align*}
\end{lem}

\begin{proof}
By orthogonality of $t$ and $\ytilde$ we have
\begin{align*}
    \alpha &= \frac{\langle \ytilde, y \rangle}{\twonorm{\ytilde}^2}.
\end{align*}
Now we give upper and lower bounds on $\alpha$. First we prove the upper bound. Note $\langle y, \ytilde \rangle \leq s$ since each coordinate is bounded by $1$, and using Lemma \ref{lem:relaxed_oracle_norm_lower_bound}, we have
\begin{align*}
    \alpha &= \frac{\langle \ytilde, y \rangle}{\twonorm{\ytilde}^2} \leq \frac{s}{s\left(1 - \frac{B_2}{\kappa - B_1}\right)^2}  \leq \frac{1}{1 - \frac{2B_2}{\kappa - B_1} } \leq 1 + \frac{4B_2}{\kappa - B_1} 
\end{align*}
where we use the fact that $\frac{1}{1-a} \leq 1 + 2a$ for $a \in [0,0.5]$, and assume $\kappa-B_1 \geq 4 B_2$.

For the lower bound, write $\langle \ytilde, y \rangle = \frac{1}{2}\left(\twonorm{\ytilde}^2 + \twonorm{y}^2 - \twonorm{y - \ytilde}^2 \right)$. For the terms $\twonorm{\ytilde}^2$ and $\twonorm{y}^2$ we have lower bounds from Lemma \ref{lem:relaxed_oracle_norm_lower_bound}. Now we need to upper bound $\twonorm{y - \ytilde}$, which we do by showing that both $y$ and $\ytilde$ are close to $\onevec$ in $2$-norm and then using the triangle inequality. Using the Law of Cosines again we have
\begin{align*}
    \twonorm{y - \onevec}^2 &= \twonorm{y}^2 + \twonorm{\onevec}^2 - 2 \langle y, \onevec \rangle \\
    & \leq 2s - 2 \langle y, \onevec \rangle \\
    &= 2s - 2 \sqrt{s} \eb^\top y \\
    & \leq 2s - 2s\left(1 - \frac{B_9}{\kappa - B_1} \right) \\
    &= 2s  \frac{B_9}{\kappa - B_1} 
\end{align*}
where the inequality steps came from the fact that $\infnorm{y} \leq 1$ and Lemma \ref{lem:relaxed_oracle_inner_prod_with_onevec_lower_bound}. The same argument also shows that $\twonorm{\ytilde - \onevec}^2 \leq 2s  \frac{B_9}{\kappa - B_1}$. Then using the triangle inequality and AM-GM we have
\begin{align*}
    \twonorm{y - \ytilde}^2 &\leq 2 \left(\twonorm{y - \onevec}^2 + \twonorm{\ytilde - \onevec}^2 \right) \\
    &\leq 8s  \frac{B_9}{\kappa - B_1}.
\end{align*}

Combining with the above arguments we have that
\begin{align*}
    \langle \ytilde, y \rangle &= \frac{1}{2}\left(\twonorm{\ytilde}^2 + \twonorm{y}^2 - \twonorm{y - \ytilde}^2 \right) \\
    & \geq \frac{1}{2}\left(s\left(1 - \frac{B_2}{\kappa - B_1}\right)^2 + s\left(1 - \frac{B_2}{\kappa - B_1}\right)^2 - 8s  \frac{B_9}{\kappa - B_1} \right)\\
    & \geq \frac{1}{2}\left(s\left(1 - \frac{2B_2}{\kappa - B_1}\right) + s\left(1 - \frac{2B_2}{\kappa - B_1}\right) - 8s  \frac{B_9}{\kappa - B_1} \right)\\
    & s \left(1 - \frac{2B_2 + 4B_9}{\kappa - B_1}\right).
\end{align*}
Now using the bound $\twonorm{\ytilde}^2 \leq s$, we have
\begin{align*}
    \alpha &= \frac{\langle \ytilde, y \rangle}{\twonorm{\ytilde}^2} \\
    &\geq 1 - \frac{2B_2 + 4B_9}{\kappa - B_1}.
\end{align*}
\end{proof}

Now we derive a perturbation bound which will be key in showing that $\beta$ is small.
\begin{lem}
\label{lem:perturb_bound}
We have 
\begin{align}
\left\langle \diag(\utilde)-\left(\Atilde - \lambdatilde I \right), Y\right\rangle   & \le\left\langle \left(\Atilde - \lambdatilde I \right)-\left(\Ashift^{(k)} - \lambda I \right),\Ytilde-Y\right\rangle .\label{eq:perturb_bound2}
\end{align}
\end{lem}
\begin{proof}
    This essentially follows from the LOO relaxed oracle KKT conditions~\eqref{eq:rank_one_LOO_relaxed_oracle_KKT}.
    Condition~\eqref{eq:rank_one_LOO_relaxed_oracle_KKT_complementary_slackness_2} gives $\left( \diag(\utilde) - \left( \Atilde - \lambdatilde I \right)\right) \Ytilde = 0$ (since $\Ytilde = \yhats \yhats^\top$). Also by condition~\eqref{eq:rank_one_LOO_relaxed_oracle_KKT_complementary_slackness_1} and since $\utilde \geq 0$ from~\eqref{eq:rank_one_LOO_relaxed_oracle_KKT_dual_feasibility},
    \[
    \left \langle \diag(\utilde), Y - \Ytilde \right \rangle = \left \langle \diag(\utilde), Y - I \right \rangle = \sum_{i=1}^s \utilde_i \left(Y_{ii} -1 \right) \leq 0.
    \]
    Using these two facts we have
    \begin{align}
        \left\langle \diag(\utilde)-\left(\Atilde - \lambdatilde I \right), Y\right\rangle & = \left\langle \diag(\utilde)-\left(\Atilde - \lambdatilde I \right), Y - \Ytilde \right\rangle \nonumber\\
        &= \left\langle\Atilde - \lambdatilde I , \Ytilde - Y \right\rangle + \left \langle \diag(\utilde), Y - \Ytilde \right \rangle \nonumber\\
        & \leq \left\langle\Atilde - \lambdatilde I , \Ytilde - Y \right\rangle. \label{eq:pert_bound_halfway_step}
    \end{align}
    Now since $\Ytilde$ is feasible for the relaxed oracle SDP~\eqref{eq:SDP_relaxed_oracle} and $Y$ is optimal, we have
    \[\left\langle \Ashift^{(k)} - \lambda I ,  \Ytilde - Y \right\rangle \leq 0 \implies \left\langle - \left( \Ashift^{(k)} - \lambda I \right),  \Ytilde - Y \right\rangle \geq 0.\]
    Adding this to~\eqref{eq:pert_bound_halfway_step}, we obtain~\eqref{eq:perturb_bound2}.
\end{proof}

\begin{lem}
\label{lem:beta_bound}
Supposing cluster $k$ satisfies the condition~\eqref{eq:critical_signal_lambda_relationship_s5}, under the event $E_3 \cap \left\{\lambda_1\left(\Ashift^{(k)}\right) > \lambda\right\}$,
    \begin{align*}
        \beta \leq \frac{B_{14}}{\kappa-B_{15}}
    \end{align*}
\end{lem}
\begin{proof}
    First we relate the LHS of~\eqref{eq:perturb_bound2} to $\beta$. The event $E_3 \cap \left\{\lambda_1\left(\Ashift^{(k)}\right) > \lambda\right\}$ implies the event $E_1^{k,j} \cap \left\{\lambda_1\left(\Ashift^{k,j}\right) > \lambda^k\right\}$ by Lemma \ref{lem:eval_pert_bound_s5}, so we may use~\eqref{eq:rank_one_LOO_relaxed_oracle_KKT_complementary_slackness_2} which states $\left(\lambdatilde I  + \diag(\utilde) - \Atilde \right)\ytilde = 0$. Since $\ytilde$ is nonzero this implies that $\ytilde$ is an eigenvector with eigenvalue $0$.
    
    Also we use a similar calculation to that within the proof of Lemma \ref{lem:SDP_relaxed_oracle_rank_one} to lower bound the remaining eigenvalues of $\lambdatilde I  + \diag(\utilde) - \Atilde$. By Weyl's inequality for any $i \in [s_k]$ with $i < s_k$ we have
\begin{align}
    \lambda_i\left(\lambdatilde I  + \diag(\utilde) - \Atilde  \right) &= \lambda_i\left(\lambdatilde I + \diag(\utilde) -\Wtilde -\frac{p-q}{2}J_{s_k \times s_k}\right) \nonumber\\
    & \geq \lambda_i\left( -\frac{p-q}{2}J_{s_k \times s_k}\right) + \lambda_{s_k}\left(\lambdatilde I  + \diag(\utilde) -\Wtilde \right)\nonumber\\
    & =  \lambda_{s_k}\left(\lambdatilde I  + \diag(\utilde) -\Wtilde \right)\nonumber\\
    & \geq  \lambda_{s_k}\left(\lambdatilde I  -\Wtilde \right) \nonumber\\
    & \geq  \lambdatilde  -  \opnorm{\Wtilde} \nonumber\\
    & \geq \kappa \sqrt{pn \log m} - \frac{B_8 \sqrt{p \log m}}{\sqrt{s_k}} - B_1 \sqrt{ps_k}\nonumber\\
    & \geq \left(\kappa - B_8 - B_1 \right) \sqrt{pn \log m} \nonumber
\end{align}
where we use the fact that $\lambda_i\left( -\frac{p-q}{2}J_{s_k \times s_k}\right) = 0$ for $i < s_k$.

Therefore using $y = \alpha \ytilde + \beta t$ in the LHS of~\eqref{eq:perturb_bound2} we obtain
\begin{align}
    \left\langle \diag(\utilde)-\left(\Atilde - \lambdatilde I \right), Y\right\rangle &= y^\top\left(\lambdatilde I  + \diag(\utilde) - \Atilde  \right) y \nonumber\\
    &= \beta^2 t^\top \left(\lambdatilde I  + \diag(\utilde) - \Atilde  \right) t \nonumber\\
    & \geq \beta^2 \left(\kappa - B_8 - B_1 \right) \sqrt{pn \log m} \label{eq:pert_bound_inequality_LHS_lower_bound}
\end{align} 
using that $t$ is a unit vector orthogonal to $\ytilde$ as well as the above facts.

Next expanding the RHS of~\eqref{eq:perturb_bound2} we have
\begin{align*}
    \left\langle \left(\Atilde - \lambdatilde I \right)-\left(\Ashift^{(k)} - \lambda I \right),\Ytilde-Y\right\rangle &= \left\langle -\wtilde e_1^\top - e_1 \wtilde^\top + (\lambdatilde-\lambda)I,\ytilde \ytilde^\top-yy^\top\right\rangle \\
    &= 2y_1 \wtilde^\top y - 2\ytilde_1 \wtilde^\top \ytilde + \left\langle(\lambdatilde-\lambda)I,\ytilde \ytilde^\top-yy^\top\right\rangle \\
    & \leq 2y_1 \wtilde^\top y + 2 \left|\ytilde_1 \right| \left|\wtilde^\top \ytilde  \right| + \left|\lambdatilde - \lambda \right|\sum_{j=1}^s \left|y_j^2 - \ytilde_j^2 \right| \\
    & \leq 2y_1 \wtilde^\top y  + 2 \left|\wtilde^\top \ytilde  \right| + s\left|\lambdatilde - \lambda \right| \\
    & \leq 2y_1 \wtilde^\top y + 2 B_5\sqrt{ p s\log m} + s\frac{B_8 \sqrt{p \log m}}{\sqrt{s}} \\
    & = 2y_1 \wtilde^\top y + \left( 2 B_5 + B_8\right)\sqrt{ p s\log m}
\end{align*}
using Lemma \ref{lem:noise_LOO_soln_inner_product_concentration} to bound $\left|\wtilde^\top \ytilde  \right|$ and Lemma \ref{lem:eval_pert_bound_s5} to bound $\left|\lambdatilde - \lambda \right|$.

Now we bound the final term $2y_1 \wtilde^\top y$. Using $y = \alpha \ytilde + \beta t$ we split
\begin{align*}
    2y_1 \wtilde^\top y &= 2\langle \wtilde e_1^\top , yy^\top \rangle \\
    &= 2\langle \wtilde e_1^\top , (\alpha \ytilde + \beta t)(\alpha \ytilde + \beta t)^\top \rangle \\
    &= \underbrace{2\alpha^2 \ytilde_1 \wtilde^\top \ytilde}_{T_1}  + \underbrace{2 \alpha \beta t_1 \wtilde^\top \ytilde}_{T_2} + \underbrace{2 \alpha \beta \ytilde_1 \wtilde^\top t}_{T_3} + \underbrace{2\beta^2 t_1 \wtilde^\top t}_{T_4}.
\end{align*}
Now we bound each term. Note that from Lemma \ref{lem:two_sided_alpha_bound} we have that $|\alpha| \leq 2$ for sufficiently large $\kappa$. Also $\beta t = y - \alpha \ytilde$ and $\infnorm{y} \leq 1$, $\infnorm{\alpha \ytilde} \leq |\alpha|\infnorm{\ytilde} \leq 2$, so $\infnorm{\beta t} \leq 3$. Using these as well as the aforementioned bounds on $\left|\wtilde^\top \ytilde  \right|$ and the fact that $\infnorm{t} \leq 1$ because it is a unit vector, we have
\begin{align*}
    T_1 & \leq 2 (2)^2 (1) B_5\sqrt{ p s_k\log m} \leq 8B_5\sqrt{ p s_k\log m} \\
    T_2 & \leq 2 (2) (3) B_5\sqrt{ p s_k\log m} \leq 12B_5\sqrt{ p s_k\log m} \\
    T_3 & \leq 2 |\beta| (1) \twonorm{\wtilde} \leq 2B_1 |\beta| \sqrt{p s} \\
    T_4 & \leq 2 |\beta| \infnorm{\beta t} \twonorm{\wtilde} \leq 6B_1 |\beta| \sqrt{ps}.
\end{align*}
Therefore
\[2y_1 \wtilde^\top y \leq 20 B_5\sqrt{ p s\log m} + 8 B_1 |\beta| \sqrt{ps}\]
and so combining with the earlier bound we have
\begin{align*}
    \left\langle \left(\Atilde - \lambdatilde I \right)-\left(\Ashift^{(k)} - \lambda I \right),\Ytilde-Y\right\rangle 
    & \leq \left(22 B_5 + B_8 \right) \sqrt{ p s\log m} + 8 B_1 |\beta| \sqrt{ps}.
\end{align*}
Combining this with~\eqref{eq:perturb_bound2} and~\eqref{eq:pert_bound_inequality_LHS_lower_bound} we have
\begin{align*}
    \beta^2 \left(\kappa - B_8 - B_1 \right) \sqrt{pn \log m} &\leq \left\langle \diag(\utilde)-\left(\Atilde - \lambdatilde I \right), Y\right\rangle   \\
    &\leq\left\langle \left(\Atilde - \lambdatilde I \right)-\left(\Ashift^{(k)} - \lambda I \right),\Ytilde-Y\right\rangle \\
    & \leq \left(22 B_5 + B_8 \right) \sqrt{ p s\log m} + 8 B_1 |\beta| \sqrt{ps} \\
    & \leq \left(22 B_5 + B_8 + 8 B_1 |\beta|\right) \sqrt{pn \log m} \\
    &= (B_{12} + B_{13}|\beta|)\sqrt{pn \log m}
\end{align*}
which implies after rearranging and dividing by $\sqrt{pn \log m}$ that
\begin{align*}
    \beta^2 \left(\kappa - B_8 - B_1 \right) - \frac{B_{13}}{\kappa} |\beta| + \frac{B_{12}}{\kappa} \leq 0.
\end{align*}
By solving for the positive root in this quadratic of $|\beta|$ we obtain that
\begin{align*}
    |\beta| & \leq \frac{\frac{B_{13}}{\kappa} + \sqrt{\left( \frac{B_{13}}{\kappa} \right)^2 + \frac{4B_{12}}{\kappa}\left(\kappa - B_8 - B_1 \right)}}{2\left(\kappa - B_8 - B_1 \right)} \\
    & \leq  \frac{\frac{B_{13}}{\kappa} +  \frac{B_{13}}{\kappa} + \sqrt{\frac{4B_{12}}{\kappa}\left(\kappa - B_8 - B_1 \right)}}{2\left(\kappa - B_8 - B_1 \right)} \\
    & \leq \frac{\frac{2B_{13}}{\kappa}  + \sqrt{\frac{4B_{12}}{\kappa}\kappa }}{2\left(\kappa - B_8 - B_1 \right)} \\
    & \leq \frac{B_{13} + B_{12}}{\kappa - B_8 - B_1 }
\end{align*}
as desired.
\end{proof}

Now combining Lemma \ref{lem:loo_close_to_1} (which guaranteed $\ytilde_1$ is close to $1$), Lemma \ref{lem:two_sided_alpha_bound} which bounded $\alpha$, and the above Lemma \ref{lem:beta_bound} which bounds $\beta$, we have
\begin{align*}
    y_1 &= \alpha \ytilde_1 + \beta t_1 \\
    & \geq \left( 1 - \frac{B_{11}}{\kappa-B_1} \right) \left( 1 - \frac{B_{10}}{\kappa - 1}\right) - \frac{B_{14}}{\kappa - B_{15}} \\
    & \geq 1 - \frac{B_{11}}{\kappa-B_1} - \frac{B_{10}}{\kappa - 1}- \frac{B_{14}}{\kappa - B_{15}}
\end{align*}
where for the second line we used the triangle inequality, the fact that $\beta$ is defined to be positive, and the fact that $t$ is a unit vector so $|t_1| \leq 1$.
Now by choosing $\kappa$ sufficiently large, we are guaranteed that $y_1 \geq \frac{1}{2}$. Again, since the identical argument works for all entries of $y$, we conclude $y \geq \frac{1}{2}$ elementwise.

Now that we have shown the $k$th relaxed oracle SDP~\eqref{eq:SDP_relaxed_oracle} has rank one and nonnegative solution $y^k {y^k}^\top$, we can easily obtain the desired conclusions about the $k$th oracle SDP~\eqref{eq:SDP_regularized_oracle}. First, since all feasible points for the $k$th oracle SDP are also feasible for the $k$th relaxed oracle SDP, and also we have shown that the $k$th relaxed oracle SDP has an optimal solution $y^k {y^k}^\top$ which is feasible for the $k$th oracle SDP, we have that $y^k {y^k}^\top$ must also be optimal for the $k$th oracle SDP. Furthermore, we will now show that $y^k {y^k}^\top$ is the unique optimal solution of the $k$th relaxed oracle SDP, which by this same reasoning implies it must be the unique optimal solution of the $k$th oracle SDP.

We know from Lemma \ref{lem:SDP_relaxed_oracle_rank_one} that any optimal solution to the relaxed oracle SDP~\eqref{eq:SDP_relaxed_oracle} must be rank-one, so let one be $y' {y'}^\top$. 
Now by repeating the proof of the perturbation bound Lemma \ref{lem:perturb_bound} up until~\eqref{eq:pert_bound_halfway_step}, but replacing the LOO quantities $\utilde, \lambdatilde, \Atilde, \Ytilde$ with $u^k, \lambda, \Ashift^{(k)}, \Yhat^k$ (which satisfy the analogous conditions~\eqref{eq:rank_one_relaxed_oracle_KKT}), and setting $Y = y' {y'}^\top$, we obtain
\[
\left\langle \diag(u^k)-\left(\Ashift^{(k)} - \lambda I \right), y' {y'}^\top \right\rangle \leq \left\langle\Ashift^{(k)} - \lambda I , \Yhat^k - y' {y'}^\top \right\rangle.
\]
Since $y' {y'}^\top$ is optimal, we have $\left\langle\Ashift^{(k)} - \lambda I , \Yhat^k - y' {y'}^\top \right\rangle = 0$. Now write $y' = \gamma y + z$, where $z$ is orthogonal to $y$. Combining the previous display with the fact that $\left(\lambda I + \diag(u)-\Ashift^{(k)}  \right)y=0$ from~\eqref{eq:rank_one_LOO_relaxed_oracle_KKT_complementary_slackness_2}, we have that
\begin{align*}
    z^\top \left(\lambda I + \diag(u)-\Ashift^{(k)}  \right) z \leq 0
\end{align*}
but we already know from the proof of Lemma \ref{lem:SDP_relaxed_oracle_rank_one} that $\lambda_i\left(\lambda I + \diag(u)-\Ashift^{(k)}  \right) > 0$ for all $i < s_k$ and that the smallest eigenvalue has eigenvector $y$, so we must have $z = 0$. Thus $y' = \gamma y^k$, and $|\gamma| > 1$ would cause $y' {y'}^\top$ to be infeasible while $|\gamma| < 1$ would cause $y' {y'}^\top$ to be suboptimal (since $\lambda_1 \left(\Ashift^{(k)} - \lambda I\right) > 0$), so we must have $\gamma \in \{-1,1\}$ and thus $y' {y'}^\top = \Yhat^k$. Thus $\Yhat^k = y^k {y^k}^\top$ is the unique optimal solution to the $k$th relaxed oracle SDP, and thus by the reasoning in the previous paragraph it is also the unique optimal solution to the $k$th oracle SDP.

Finally, we set $L^k = 0_{s_k \times s_k}$ and $U^k = \diag(u^k)$. We have $U^k \succcurlyeq 0$ as desired since $u^k \geq 0$ from~\eqref{eq:rank_one_relaxed_oracle_KKT_dual_feasibility}, and from the conditions~\eqref{eq:rank_one_relaxed_oracle_KKT} as well as the fact that $y^k \geq 0$ elementwise, we immediately have that these choices of $L^k, U^k$ (and $\Yhat^k = y^k {y^k}^\top$) satisfy the oracle SDP KKT conditions~\eqref{eq:oracle_KKT}.

In summary, all conclusions of all cases of Lemma \ref{lem:oracle_SDP_solns_main_lemma} hold under the event $E_3$ (which is contained in the event $E_1$ which sufficed for the non-mid-size cases as shown in Lemmas \ref{lem:oracle_zero_soln_case} and \ref{lem:oracle_zero_ones_case}), and we checked that $\P(E_3) \geq 1 - O(m^{-3})$ as desired.
To complete our proof of Lemma \ref{lem:oracle_SDP_solns_main_lemma}, note we can ensure that the constants $\underline{B}, \overline{B}$ can be taken to be $> 2$ by requiring $\kappa$ sufficiently large.

\subsection{Recovery SDP Solution}
\label{sec:recovery_SDP_soln}
Now we complete the proof of the main Theorem \ref{thm:main_theorem_recovery_sdp}.
First we assume we are in the event $E_3$ defined in the previous subsection, under which the main Lemma \ref{lem:oracle_SDP_solns_main_lemma} from the previous subsection holds. Now we define a further subset of this event, $E_4$, which will be the intersection of $E_3$ with the events described in the following lemmas.

For each cluster $k$ such that $\frac{p-q}{2}s_k \geq B_4\sqrt{pn \log m}$ if the $k$th oracle solution $\Yhat^k = y^k {y^k}^\top$ is nonzero we define $x^k = \frac{y^k}{\twonorm{y^k}}$, and otherwise if $y^k = 0$ we define $x^k = v^k$ (a top eigenvector of $\Ashift^{(k)}$, with properties described in Lemma \ref{lem:infty_pert_bound_s5}). Our event $E_4$ will involve these $x^k$ so we establish a basic fact about them.
\begin{lem}
\label{lem:x_infnorm_bound}
    Under the event $E_3$, for all clusters $k$ such that $\frac{p-q}{2}s_k \geq B_4\sqrt{pn \log m}$, we have
    \begin{align*}
        \infnorm{x^k} \leq \frac{2}{\sqrt{s_k}}.
    \end{align*}
\end{lem}
\begin{proof}
    The case that $y^k = 0$ and $x^k = v^k$ is established by Lemma \ref{lem:LOO_evec_noise_inner_prod_concentration}. 
    By Lemma \ref{lem:relaxed_oracle_norm_lower_bound} if $\frac{p-q}{2}s_k < \left(1+\frac{B_6}{\kappa}\right)\kappa \sqrt{pn \log m}$ then we will have
    \begin{align*}
        \twonorm{y^k} \geq \sqrt{s_k}  \left( 1 - \frac{B_2}{\kappa - B_1} \right) \geq \frac{\sqrt{s_k} }{2}
    \end{align*}
    (by our assumptions on the size of $\kappa$)
    and otherwise when $\frac{p-q}{2}s_k \geq \left(1+\frac{B_6}{\kappa}\right)\kappa \sqrt{pn \log m}$ by Lemma \ref{lem:oracle_zero_ones_case} we have $y^k = \onevec$ so $x^k = \frac{1}{\sqrt{s_k}}\onevec$ which clearly satisfies the conclusion.
\end{proof}

\begin{lem}
\label{lem:offdiagonal_noise_concentration}
With probability at least $1 - O(m^{-3})$, for all clusters $k$ such that $\frac{p-q}{2}s_k \geq B_4\sqrt{pn \log m}$, for all clusters $j \neq k$, we have
\begin{align}
    \infnorm{W^{(jk)}x^k} \leq B_5 \sqrt{p \log m}.
\end{align}
\end{lem}

\begin{lem}
\label{lem:op_norm_concentration_entire_noise}
With probability at least $1-O(m^{-3})$,
\begin{align}
    \opnorm{W} &\leq B_1 \sqrt{pn}. \label{eq:noise_op_norm_bound_overall}
\end{align}
\end{lem}
Lemmas \ref{lem:offdiagonal_noise_concentration} and \ref{lem:op_norm_concentration_entire_noise} are proven in Subsection \ref{sec:concentration_results}.

We let the event $E_4$ be the intersection of the event $E_3$ (under which Lemma \ref{lem:oracle_SDP_solns_main_lemma} on the oracle SDP solutions holds) and the events described in Lemmas \ref{lem:offdiagonal_noise_concentration} and \ref{lem:op_norm_concentration_entire_noise}. By the union bound (recalling from the previous subsection that $E_3$ holds with probability at least $1 - O(m^{-3})$) we have that $\P(E_4) \geq 1 - O(m^{-3})$.
Now we will construct a solution to the KKT equations~\eqref{eq:original_KKT} of the recovery SDP~\eqref{eq:SDP_regularized_original}. The rest of the proof will be deterministic, assuming the event $E_4$ holds.

We define
\begin{gather*}
    \Yhat = \diag \left(y^1{y^1}^\top, \dots, y_K {y_K}^\top  \right)\\
    U = \diag \left(U^1, \dots, U^K  \right).
\end{gather*}
Note that it is immediate that $\Yhat$ satisfies the conditions in~\eqref{eq:recovery_KKT_primal_feasibility}, since each block $y^i{y^i}^\top$ is feasible for the oracle SDP (\ref{eq:SDP_regularized_oracle}) and thus satisfies $0 \leq y^i{y^i}^\top \leq 1$ and $y^i{y^i}^\top \succcurlyeq 0$, and a block-diagonal matrix of PSD blocks is PSD. Also from the oracle KKT conditions~\eqref{eq:oracle_KKT} we have that $U^i \geq 0$ and $(U^i_{jk}-1)\Yhat^i_{jk} = 0$ $\forall j,k \in [n]$, which ensure $U \geq 0$ and the condition $(U_{jk}-1)\Yhat_{jk} = 0$ $\forall j,k \in [n]$ from~\eqref{eq:recovery_KKT_dual_feasibility} and~\eqref{eq:recovery_KKT_complementary_slackness_1}.

Finally we describe our choice of $L$. First we set the diagonal blocks $L^{(i)} = L^i =0$ for all $i$. This guarantees that the condition $L_{jk}\Yhat_{jk} =0 $ $\forall j,k \in [n]$ from line~\eqref{eq:recovery_KKT_complementary_slackness_1} is satisfied, since the diagonal blocks of $L$ are $0$ and the off-diagonal blocks of $\Yhat$ are 0.

Now we describe the off-diagonal blocks $L^{(ij)}$, which have different forms depending on the sizes of $s_i$ and $s_j$, so we split into 3 cases. We only describe the cases where $i > j$, because we set $L^{(ji)} = {L^{(ij)}}^\top$. We assume that $\kappa \geq 2 B_4$, so that the conclusions of Lemma \ref{lem:offdiagonal_noise_concentration} hold for clusters $k$ such that $\frac{p-q}{2}s_k \geq \frac{\kappa}{2}\sqrt{pn \log m}$.
\subsubsection{Both clusters large}
\label{sec:L_blocks_both_large}
If $s_i, s_j \geq \left(\frac{p-q}{2}\right)^{-1} \frac{\kappa \sqrt{pn \log m}}{2}$, we use 
\begin{align*}
    L^{(ij)} &= -\E \Ashift^{(ij)} - x_i x_i^\top W^{(ij)} - W^{(ij)} x_j x_j^\top + x_i x_i^\top W^{(ij)} x_j x_j^\top 
\end{align*}
This is chosen so that we have
\begin{align*}
    \Ashift^{(ij)} + L^{(ij)} &= W^{(ij)} + \E \Ashift^{(ij)} + L^{(ij)} \\
    &= W^{(ij)} -x_i x_i^\top W^{(ij)} - W^{(ij)} x_j x_j^\top + x_i x_i^\top W^{(ij)} x_j \\
    &= (I - x_i x_i^\top) W^{(ij)} (I - x_j x_j^\top).
\end{align*}

Now we check non-negativity of this choice of $L^{(ij)}$. Since $\infnorm{x_i} \leq \frac{2}{\sqrt{s_i}}$ and $\infnorm{x_j} \leq \frac{2}{\sqrt{s_j}}$ by Lemma \ref{lem:x_infnorm_bound}, and since $\infnorm{ W^{(ji)}x_i},\infnorm{ W^{(ij)} x_j} \leq B_5 \sqrt{p \log m}$ by Lemma \ref{lem:offdiagonal_noise_concentration},
we can verify
\begin{align*}
    &\infnorm{- x_i x_i^\top W^{(ij)} - W^{(ij)} x_j x_j^\top + x_i x_i^\top W^{(ij)} x_j x_j^\top} \\
    &\leq \infnorm{ x_i x_i^\top W^{(ij)}} + \infnorm{ W^{(ij)} x_j x_j^\top } + \infnorm{ x_i x_i^\top W^{(ij)} x_j x_j^\top } \\
    & \leq \infnorm{ x_i} \infnorm{ W^{(ji)}x_i} + \infnorm{ W^{(ij)} x_j} \infnorm{x_j} + \infnorm{x_i} \infnorm{x_j} \left|x_i^\top W^{(ij)} x_j\right|\\
    & \leq \frac{2}{\sqrt{s_i}} B_5\sqrt{ p \log m} + \frac{2}{\sqrt{s_j}} B_5\sqrt{ p \log m} + \infnorm{x_i} \infnorm{x_j} \onenorm{x_i}\infnorm{W^{(ij)} x_j} \\
    & \leq \frac{2}{\sqrt{s_i}} B_5\sqrt{ p \log m} + \frac{2}{\sqrt{s_j}} B_5\sqrt{p \log m} + s_i\infnorm{x_i}^2 \infnorm{x_j} \infnorm{W^{(ij)} x_j} \\
    & \leq \frac{2}{\sqrt{s_i}} B_5\sqrt{p \log m} + \frac{2}{\sqrt{s_j}} B_5\sqrt{ p \log m} + s_i\frac{4}{s_i} \frac{2}{\sqrt{s_j}} B_5\sqrt{p \log m} \\
    & \leq \left( \frac{2}{\sqrt{s_i}} + \frac{10}{\sqrt{s_j}}\right)B_5 \sqrt{p \log m}
\end{align*}
which is $\leq \frac{p-q}{2}$ as long as $\kappa$ is sufficiently large, since by assumption $s_i, s_j$ are each sufficiently large so that
\begin{align*}
    \frac{p-q}{2}s_i \geq \frac{\kappa \sqrt{pn \log m}}{2} \implies \frac{p-q}{2} \geq  \frac{\kappa \sqrt{pn \log m}}{2s_i}= \frac{\kappa}{2} \frac{\sqrt{p \log m}}{\sqrt{s_i}} \frac{\sqrt{n}}{\sqrt{s_i}} \geq \frac{\kappa}{2} \frac{\sqrt{p \log m}}{\sqrt{s_i}}
\end{align*}
and likewise for $s_j$. Now noting that $-\E \Ashift^{(ij)} = \frac{p-q}{2}J_{s_i \times s_j}$, we have that all entries of $L^{(ij)}$ are positive.

\subsubsection{One cluster large, one cluster small}
\label{sec:L_blocks_one_large_one_small}
When $\frac{p-q}{2}s_i \geq \frac{\kappa \sqrt{pn \log m}}{2} > \frac{p-q}{2}s_j$, we use 
\begin{align*}
    L^{(ij)} &= -\E \Ashift^{(ij)} - x_ix_i^\top W^{(ij)}
\end{align*}
which ensures that
\begin{align*}
     \Ashift^{(ij)} + L^{(ij)} &= W^{(ij)} +\E \Ashift^{(ij)} + L^{(ij)} \\
     &= (I - x_i x_i^\top) W^{(ij)}.
\end{align*}
Now we check non-negativity of this choice of $L^{(ij)}$. Again using the facts that $\infnorm{x_i} \leq \frac{2}{\sqrt{s_i}}$ and $\infnorm{W^{(ji)}x_i} \leq B_5\sqrt{ p \log m}$, we have
\begin{align*}
    \infnorm{x_ix_i^\top W^{(ij)}}  \leq \infnorm{x_i} \infnorm{x_i^\top W^{(ij)}} \leq \frac{2}{\sqrt{s_i}} B_5 \sqrt{p \log n}
\end{align*}
which identically to the previous case is $\leq \frac{p-q}{2}$ for sufficiently large $\kappa$ due to the assumption on the size of $s_i$.

\subsubsection{Both clusters small}
\label{sec:L_blocks_both_small}
This is the simplest case. If $s_i, s_j < \left(\frac{p-q}{2}\right)^{-1} \frac{\lambda}{2}$, we set $L^{(ij)} = -\E \Ashift^{(ij)} = \frac{p-q}{2}J_{s_i \times s_j}$ which is obviously all non-negative. Note in this case we have
\begin{align*}
    \Ashift^{(ij)} + L^{(ij)} = \E \Ashift^{(ij)} + W^{(ij)} -\E \Ashift^{(ij)} = W^{(ij)}.
\end{align*}

\subsubsection{Checking KKT conditions}
It remains to check that $\Ashift - \lambda I - U + L \preccurlyeq 0$ and $\left( \Ashift - \lambda I - U + L \right) \Yhat = 0_{n \times n}$.

First we check that $\left( \Ashift - \lambda I - U + L \right) \Yhat = 0_{n \times n}$ using blockwise matrix multiplication. We have
\begin{align*}
    \left( \left( \Ashift - \lambda I - U + L \right) \Yhat\right)^{(ij)} &= \sum_{k=1}^K \left( \Ashift - \lambda I - U + L \right)^{(ik)} \Yhat^{(kj)} \\
    &= \left( \Ashift - \lambda I - U + L \right)^{(ij)} \Yhat^{(jj)}
\end{align*}
using the fact that $\Yhat^{(kj)} = 0$ for $k \neq j$.

If $i \neq j$, then $\left( \Ashift - \lambda I - U + L \right)^{(ij)} = \Ashift^{(ij)} + L^{(ij)}$. If $\frac{p-q}{2}s_j < \frac{\kappa \sqrt{pn \log m}}{2}$, then since $\frac{\kappa \sqrt{pn \log m}}{2} \leq \left(1- \frac{B_1}{\kappa} \right)\kappa \sqrt{pn \log m}$, by Lemma \ref{lem:oracle_SDP_solns_main_lemma} $ \Yhat^{(jj)} = \Yhat^j = 0$ so 
$\left( \Ashift - \lambda I - U + L \right)^{(ij)} \Yhat^{(jj)} = 0$ as desired. If $\frac{p-q}{2}s_j \geq \frac{\kappa \sqrt{pn \log m}}{2}$, if still $\Yhat^{(jj)} = 0$ then again we are done, and otherwise in this case by construction of $L^{(ij)}$ we have
$\Ashift^{(ij)} + L^{(ij)} = M(I - x^j{x^j}^\top)$ 
for some $M$, so then 
\begin{align*}
    \left( \Ashift^{(ij)} + L^{(ij)}\right)\Yhat^{(jj)} = M (I - x^j{x^j}^\top) y^j {y^j}^\top = M (y^j {y^j}^\top - y^j {y^j}^\top) = 0.
\end{align*}

Next if $i = j$, we have $\left( \Ashift - \lambda I - U + L \right)^{(ij)} = \left(\Ashift^{(i)} - \lambda I - U^i +L^i \right)$. Again if $\Yhat^{(jj)} = \Yhat^j = 0$ then we have $\left( \Ashift - \lambda I - U + L \right)^{(ij)} \Yhat^{(jj)} = 0$ as desired. If $\Yhat^j \neq 0$, by the KKT conditions for the oracle SDP (\ref{eq:oracle_KKT_complementary_slackness_2}) we have $\left( \Ashift^{(i)} - \lambda I - U^i + L^i \right) \Yhat^i = 0$ so
\begin{align*}
    \left( \Ashift - \lambda I - U + L \right)^{(ij)} \Yhat^{(jj)} = \left(\Ashift^{(i)} - \lambda I - U^i +L^i\right)\Yhat^{j} =0
\end{align*}
as desired.

Finally we verify that $\Ashift - \lambda I - U + L \preccurlyeq 0$. Equivalently we check that $\Ashift  - U + L \preccurlyeq \lambda I$.
Let $\overline{x^i}$ be a zero-padded version of $x^i$ so that $\overline{x^i}^{(i)} = x^i$ (recall that for $z \in \R^n$, $z^{(i)}\in \R^{s_i}$ extracts the entries corresponding to cluster $i$). Let $Q\in \R^{n \times r} $ be a matrix with (orthonormal) columns equal to the $\overline{x^1}, \dots, \overline{x^r}$ corresponding to the clusters large enough so that $s_1, \dots, s_r \geq \left( \frac{p-q}{2}\right)^{-1}\frac{\kappa \sqrt{pn \log m}}{2}$. Let $P = QQ^\top = \diag\left(x^1{x^1}^\top , \dots, x^r {x^r}^\top, 0_{s_{r+1} \times s_{r+1} }, \dots, 0_{s_K \times s_K} \right)$.

\begin{lem}
\label{lem:eigenvectors_of_M}
For each $i = 1, \dots, r$, we have
\begin{align*}
    \left( \Ashift  - U + L \right) \overline{x^i} &= \alpha^i \overline{x^i}
\end{align*}
for some $\alpha^i \leq \lambda$.
\end{lem}
\begin{proof}
First, by the zero-padding of $\overline{x^i}$, we have
\begin{align*}
    \left(\left( \Ashift  - U + L \right) \overline{x^i}\right)^{(j)} &= \sum_{k=1}^K\left( \Ashift  - U + L \right)^{(jk)}\overline{x^i}^{(k)} \\
    &= \left( \Ashift  - U + L \right)^{(ji)}\overline{x^i}^{(i)} \\
    &= \left( \Ashift^{(ji)}  - U^{(ji)} + L^{(ji)} \right) x^i
\end{align*}
so it remains to show that the above is equal to $\alpha^i x^i$ for some $\alpha^i \leq \lambda$.
If $i \neq j$ then $U^{(ji)} = 0$ and $\Ashift^{(ji)}  + L^{(ji)}$ has the form $\Ashift^{(ji)}  + L^{(ji)} = M (I - x^i {x^i}^\top)$, so in this case $\left( \Ashift^{(ji)}  - U^{(ji)} + L^{(ji)} \right) x^i = 0$. If $i = j$, then we have to consider the cases that $y^i = 0$ and $y^i \neq 0$ separately. 

When $y^i = 0$ we define $x^i$ as a top eigenvector $v^i$ of $\Ashift^{(i)} = \Ashift^{(i)} - U^i + L^i$ (where we are using the fact from Lemma \ref{lem:oracle_SDP_solns_main_lemma} that $U^i = L^i = 0$ in this case). Thus there exists some $\alpha^i$ such that $\left( \Ashift^{(i)} - U^i + L^i\right)x^i = \alpha^i x^i$, and furthermore by condition~\eqref{eq:oracle_KKT_dual_feasibility} from the KKT conditions for the oracle SDP we have that $\alpha^i \leq \lambda$.

When $y^i \neq 0$, then by combining the condition~\eqref{eq:oracle_KKT_complementary_slackness_2} from the oracle SDP KKT conditions with the fact that $\Yhat^i = y^i {y^i}^\top$, we obtain (by right-multiplying by $y^i$ and dividing by $\twonorm{y^i}^2 \neq 0$) that
\[\left( \Ashift^{(i)} - \lambda I - U^i + L^i \right) y^i = 0,\]
so now dividing by $\twonorm{y^i}$ and finally recall that in this case we have defined $x^i = \frac{y^i}{\twonorm{y^i}}$.
\end{proof}

\begin{lem}
\label{lem:psd_bound_top_evecs}
We have
\begin{gather*}
    (I-P) \left( \Ashift  - U + L \right) P = 0,\\
    P \left( \Ashift  - U + L \right) P \preccurlyeq \lambda I.
\end{gather*}
\end{lem}
\begin{proof}
Immediate from the form of $P$ and Lemma \ref{lem:eigenvectors_of_M}.
\end{proof}

\begin{lem}
\label{lem:psd_bound_bottom_evecs}
We have
\[(I-P) \left( \Ashift  - U + L \right) (I-P) \preccurlyeq \lambda I.\]
\end{lem}
\begin{proof}
Let $D = \diag(\Ashift^{(1)} - U^{(1)}, \dots, \Ashift^{(r)} - U^{(r)}, \frac{p-q}{2}J_{s_{r+1}}, \dots, \frac{p-q}{2}J_{s_{K}}  )$. First we check that $(I - P)D (I-P) \preccurlyeq \frac{\lambda}{2}$. Notice that $(I - P)D (I-P)$ is also block-diagonal, with
\begin{align*}
    \left( (I - P)D (I-P) \right)^{(i)} &= \begin{cases}
        (I - x^i{x^i}^\top)\left(\Ashift^{(i)} - U^{(i)} \right)(I - x^i{x^i}^\top) & i \leq r \\
        \frac{p-q}{2}J_{s_{i} \times s_i} & i > r
    \end{cases}
\end{align*}
so $\lambda_1 \left( (I - P)D (I-P)\right)$ is bounded by the maximum eigenvalue of the blocks. We have chosen $r$ so that for all $i > r$, \[\lambda_1 \left(  \frac{p-q}{2}J_{s_{i} \times s_i} \right) = \frac{p-q}{2}s_i \leq \frac{\kappa \sqrt{pn \log m}}{2} \leq \frac{\lambda}{2}.\]

For all $i \leq r$, we know $x^i$ is an eigenvector of $\Ashift^{(i)} - U^{(i)} + L^{(i)} = \Ashift^{(i)} - U^{(i)}$ with eigenvalue $\lambda$ from the proof of Lemma \ref{lem:eigenvectors_of_M}. By Weyl's inequality we have for all $j > 1$ that
\begin{align*}
    \lambda_j \left( \Ashift^{(i)} - U^{(i)}\right) &= \lambda_j \left( \frac{p-q}{2}J_{s_i \times s_i} + W^{(i)} - U^{(i)}\right) \\
    &\leq \lambda_j \left( \frac{p-q}{2}J_{s_i \times s_i}\right) + \lambda_1 \left(  W^{(i)} - U^{(i)}\right) \\
    & \leq 0 + \opnorm{W^{(i)}} \\
    & \leq B_1 \sqrt{pn \log m} \\
    & \leq \frac{\kappa \sqrt{pn \log m}}{2} \\
    & < \lambda
\end{align*}
where we use the facts that $\lambda_j \left( \frac{p-q}{2}J_{s_i \times s_i}\right) = 0$ for $j > 1$, that $U^{(i)} = U^i \succcurlyeq 0$ from Lemma \ref{lem:oracle_SDP_solns_main_lemma}, and assume $\kappa \geq 2B_1$. Therefore $x^i$ is a top eigenvector of $\Ashift^{(i)} - U^{(i)}$ and all other eigenvalues are $\leq \frac{\kappa \sqrt{pn \log m}}{2}$. Now we can calculate that
\begin{align*}
    \lambda_1 \left( (I - x^i{x^i}^\top)\left(\Ashift^{(i)} - U^{(i)} \right)(I - x^i{x^i}^\top) \right) &= \sup_{z: \twonorm{z}\leq 1} z^\top \left( (I - x^i{x^i}^\top)\left(\Ashift^{(i)} - U^{(i)} \right)(I - x^i{x^i}^\top) \right) z \\
    &= \sup_{z: \twonorm{z}\leq 1, z \perp x^i} z^\top \left(\Ashift^{(i)} - U^{(i)} \right) z \\
    &= \lambda_2 \left( \Ashift^{(i)} - U^{(i)} \right) \\
    & \leq \frac{\kappa \sqrt{pn \log m}}{2}
\end{align*}
as desired, where we used the fact that $(I - x^i{x^i}^\top) z$ is perpendicular to $x^i$ for any $z$ and also always has norm $\leq \twonorm{z}$ since it is an orthogonal projection. Therefore we have checked that $(I - P)D (I-P) \preccurlyeq \frac{\lambda}{2}$.

Now we check that by our construction of $L$ we have
\begin{align*}
    &(I-P) \left(\Ashift + L - U - D\right)(I-P) \\
    &= (I-P) \left(W - \diag \left(W^{(1)}, \dots, W^{(r)}, 0_{s_r \times s_r}, \dots, 0_{s_K \times s_K} \right)\right)(I-P) \\
    & \preccurlyeq 2 \opnorm{W} I \\
    & \preccurlyeq 2 B_1 \sqrt{pn} I\\
    & \preccurlyeq \frac{\kappa \sqrt{pn \log m} }{2}I.
\end{align*}
This completes the proof of the lemma.
\end{proof}

Now we combine the results of Lemmas \ref{lem:psd_bound_top_evecs} and \ref{lem:psd_bound_bottom_evecs} with the conjugation rule and the fact that $P$ is an orthogonal projection to obtain that 
\begin{gather*}
    P \left( \Ashift  - U + L \right) P = P^2 \left( \Ashift  - U + L \right) P^2\preccurlyeq \lambda P^2 = \lambda P \\
    (I-P) \left( \Ashift  - U + L \right) (I-P) = (I-P)^2 \left( \Ashift  - U + L \right) (I-P)^2\preccurlyeq \lambda (I-P)^2 = \lambda (I-P).
\end{gather*}
Finally we add these two inequalities and use the fact that $(I-P) \left( \Ashift  - U + L \right) P = 0$ from Lemma \ref{lem:psd_bound_top_evecs} to obtain that $\Ashift  - U + L \preccurlyeq \lambda I$ as desired.

Thus we have checked all KKT conditions~\eqref{eq:original_KKT} so we conclude that $\Yhat$ is a solution to the recovery SDP~\eqref{eq:SDP_regularized_original}. Finally we need to show that $\Yhat$ is the unique solution. We use an argument similar to the uniqueness proof for the mid-size oracle SDP solutions. We start with bound similar to the perturbation bound Lemma \ref{lem:perturb_bound}.
\begin{lem}
\label{lem:perturb_bound_for_recovery_uniqueness}
    Suppose that $Y'$ is feasible for the recovery SDP~\eqref{eq:SDP_regularized_original} and $\left \langle \Ashift - \lambda I, Y' - Y \right \rangle = 0$. Then
    \begin{align}
        \left \langle L, Y' \right \rangle = 0.
    \end{align}
\end{lem}
\begin{proof}
    The key ingredients will be the recovery SDP KKT conditions~\eqref{eq:original_KKT}. From~\eqref{eq:recovery_KKT_complementary_slackness_1}, we have the complementarity condition $\left \langle L, \Yhat \right \rangle = 0$. Also from~\eqref{eq:recovery_KKT_complementary_slackness_2} we have $\left( \Ashift - \lambda I - U + L \right) \Yhat = 0_{n \times n}$, from~\eqref{eq:recovery_KKT_dual_feasibility} we have $\Ashift - \lambda I - U + L \preccurlyeq 0$, and since $Y'$ is feasible we have $Y' \succcurlyeq 0$, so combining these we have
    \[\left \langle \Ashift - \lambda I - U + L, Y' - \Yhat \right \rangle = \left \langle \Ashift - \lambda I - U + L, Y'  \right \rangle \leq 0. \]
    Finally, from~\eqref{eq:recovery_KKT_complementary_slackness_2} we have $\left\langle U, \Yhat \right\rangle = \left\langle U, \Yhat \right\rangle$ and we know $U \geq 0$ (from~\eqref{eq:recovery_KKT_dual_feasibility}) and that $Y' \leq 1$ since it is feasible, so
    \[
    \left \langle U, Y'-\Yhat \right\rangle = \left \langle U, Y'-J \right\rangle \leq 0.
    \]

    Now combining all these facts, we can obtain that
    \begin{align*}
        \left \langle L, Y' \right \rangle &= \left \langle L, Y' - \Yhat \right \rangle \\
        &= \left \langle \Ashift - \lambda I - U + L, Y' - \Yhat \right \rangle + \left \langle U, Y'-\Yhat \right\rangle - \left \langle \Ashift - \lambda I, Y' - Y \right \rangle\\
        & \leq  -\left \langle \Ashift - \lambda I, Y' - Y \right \rangle \\
        &= 0.
    \end{align*}
    Now since $Y' \geq 0$ by feasibility and $L \geq 0$ from~\eqref{eq:recovery_KKT_dual_feasibility}, $\left \langle L, Y' \right \rangle \geq 0$, so we must have $\left \langle L, Y' \right \rangle = 0$.
\end{proof}

From the construction for $L$ given in this subsection, all non-block-diagonal entries of $L$ are strictly positive. Therefore if $Y'$ is optimal for the recovery SDP, since it must have $Y' \geq 0$, we can apply Lemma \ref{lem:perturb_bound_for_recovery_uniqueness} to immediately obtain that $Y'$ is zero for all non-diagonal blocks (that is, ${Y'}^{(ij)} = 0_{s_i \times s_j}$ whenever $i \neq j$). Now since Lemma \ref{lem:oracle_SDP_solns_main_lemma} states that the oracle diagonal blocks have unique solutions, we must have $Y' = \Yhat$, and thus $\Yhat$ is the unique solution to the recovery SDP. This completes the proof of Theorem \ref{thm:main_theorem_recovery_sdp}.

\subsection{Clustering With a Gap Proofs}
\label{sec:clustering_with_gap}
In this subsection we prove our gap-dependent clustering result Theorem \ref{thm:clustering_with_a_gap}, and we reuse the ideas to prove Lemma \ref{lem:oracle_zero_ones_case} which handles the case that the oracle SDP will have an all-one solution. Specifically, we first prove Lemmas \ref{lem:gap_exact_recovery_trace_term_bound}, \ref{lem:gap_exact_recovery_signal_term_bound}, and \ref{lem:gap_exact_recovery_noise_term_bound}, which can then be applied to easily prove Theorem \ref{thm:clustering_with_a_gap} and Lemma \ref{lem:oracle_zero_ones_case}.

Before beginning, we introduce some notation used in this subsection. We let $\sb > \ss$ be consecutive cluster sizes, that is there exists some $r$ such that $\sb = s_r, \ss = s_{r + 1}$). $\sb$ will be the smallest cluster recovered in Theorem \ref{thm:clustering_with_a_gap}. We define $\Yb$ to be a version of the ground truth $\Ystar$ which only contains clusters at least the size of $\sb$, that is,
\begin{align*}
    \Yb &= \diag \left(J_{s_1 \times s_1}, \dots, J_{\sb \times \sb}, 0_{\ss \times \ss}, \dots, 0_{s_K \times s_K} \right).
\end{align*}
We also define a version for only clusters at most the size of $\ss$,
\begin{align*}
    \Ys = \Ystar - \Yb = \diag \left(0_{s_1 \times s_1}, \dots, 0_{\sb \times \sb}, J_{\ss \times \ss}, \dots, J_{s_K \times s_K} \right).
\end{align*}
Letting $m = \sum_{k = 1}^r s_k$ be the number of nodes in the ``big'' clusters (at least size $\sb$), we define $\Js = \diag(0_{m \times m}, J_{(n-m) \times (n-m)})$ and $\Jb = J_{n \times n} - \Js$, that is,
\[
\Jb=\begin{bmatrix}J_{m\times m} & J_{m\times(n-m)}\\
J_{(n-m)\times m} & 0_{(n-m)\times(n-m)}
\end{bmatrix},\qquad\Js=\begin{bmatrix}0_{m\times m} & 0_{m\times(n-m)}\\
0_{(n-m)\times m} & J_{(n-m)\times(n-m)}
\end{bmatrix}.
\]
We let $Q \in \R^{n \times r}$ be a matrix whose (orthonormal) columns are the $r$ singular vectors of $\Yb$. Then $QQ^\top = \diag\left( \frac{1}{s_1}J_{s_1 \times s_1}, \dots, \frac{1}{\sb}J_{\sb \times \sb}, 0_{\ss \times \ss}, \dots, 0_{s_K \times s_K}\right)$. Finally we define the projection $\PT:\real^{n\times n}\to\real^{n\times n}$
and its complement $\PP$ by 
\begin{align*}
\PT(M) & =QQ^{\top}M+MQQ^{\top}-QQ^{\top}MQQ^{\top},\\
\PP(M) & =M-\PT(M)=(I-QQ^{\top})M(I-QQ^{\top}).
\end{align*}

\begin{lem}
\label{lem:gap_exact_recovery_trace_term_bound}
    For any $Y$ which is feasible to the recovery SDP~\eqref{eq:SDP_regularized_original},
    \begin{align*}
        &\nucnorm{\PP(Y - \Yb)} \leq \tr(Y - \Yb) + \frac{1}{\sb} \onenorm{\Jb \circ (Y - \Yb)} \\
        \text{or equivalently } & - \tr(\Yb - Y) \geq \nucnorm{\PP(Y - \Yb)} - \frac{1}{\sb} \onenorm{\Jb \circ (Y - \Yb)}.
    \end{align*}
\end{lem}
\begin{proof}
    Let $D:=Y-\Yb$. We have the following chain of inequalities:
\begin{align*}
\nucnorm{\PP D} & =\tr(\PP D) &  & \PP(Y-\Yb)=\PP(Y)=(I-QQ^{\top})Y(I-QQ^{\top})\text{ is psd}\\
 & =\tr\left((I-QQ^{\top})D(I-QQ^{\top})\right)\\
 & =\tr\left((I-QQ^{\top})D\right) &  & \text{cyclic property of trace; }(I-QQ^{\top})^{2}=I-QQ^{\top}\\
 & =\tr(D)-\tr(QQ^{\top}D)\\
 & = \tr(D)-\left\langle QQ^{\top},D\right\rangle\\
 & = \tr(D)-\left\langle QQ^{\top},\Jb D\right\rangle \\
 & \le\tr(D) + \left\Vert QQ^{\top}\right\Vert _{\infty}\onenorm{\Jb D} \\
 & \le\tr(D) + (1/\sb)\onenorm{\Jb D}.
\end{align*}
\end{proof}

\begin{lem}
\label{lem:gap_exact_recovery_signal_term_bound}
    For any $Y$ which is feasible to the recovery SDP~\eqref{eq:SDP_regularized_original},
    \begin{align*}
        \left\langle \Yb - Y, \E A - \frac{p+q}{2}J \right\rangle \geq \frac{p-q}{2} \onenorm{\Jb \circ (Y - \Yb)} - \frac{p-q}{2}\ss \nucnorm{\PP(Y - \Yb)}.
    \end{align*}
\end{lem}
\begin{proof}
    Let $Y$ be an arbitrary feasible solution of (\ref{eq:SDP_regularized_original}).
Let us write
\begin{align*}
\left\langle \Yb-Y,\EE A-\frac{p+q}{2}J\right\rangle  & =\left\langle (\PT+\PP)(\Yb-Y),\EE A-\frac{p+q}{2}J\right\rangle \\
 & =\underbrace{\left\langle \Yb-Y,\PT\left(\EE A-\frac{p+q}{2}J\right)\right\rangle }_{F_{1}}+\underbrace{\left\langle \PP(\Yb-Y),\PP\left(\EE A-\frac{p+q}{2}J\right)\right\rangle }_{F_{2}}.
\end{align*}
We calculate that
\begin{align*}
    \PT J &=QQ^\top J + JQQ^\top - QQ^\top J QQ^\top \\
    &= \begin{bmatrix}J_{m\times m} & J_{m\times(n-m)}\\ 0_{(n-m)\times m} & 0_{(n-m)\times(n-m)} \end{bmatrix}+
\begin{bmatrix}J_{m\times m} & 0_{m\times(n-m)}\\ J_{(n-m)\times m} & 0_{(n-m)\times(n-m)} \end{bmatrix} - 
\begin{bmatrix}J_{m\times m} & 0_{m\times(n-m)}\\0_{(n-m)\times m} & 0_{(n-m)\times(n-m)}\end{bmatrix} \\
    &=\Jb.
\end{align*}
It follows that $\PP J=J-\PT J=\Js$.
Also note that $\E A=(p-q)Y^{*}+qJ$. Hence 
\begin{equation}
\PT\left(\E A-\frac{p+q}{2}J\right)=\PT\left((p-q)Y^{*}-\frac{p-q}{2}J\right)=(p-q)\Yb-\frac{p-q}{2}\Jb.\label{eq:PT_coeff}
\end{equation}
and 
\begin{equation}
\PP\left(\E A-\frac{p+q}{2}J\right)=\PP\left((p-q)Y^{*}-\frac{p-q}{2}J\right)=(p-q)\Ys-\frac{p-q}{2}\Js.\label{eq:PP_coeff}
\end{equation}

Using (\ref{eq:PT_coeff}), the $F_{1}$ term can be written explicitly
as 
\begin{align*}
F_{1} & =\left\langle \Yb-Y,(p-q)\Yb-\frac{p-q}{2}\Jb\right\rangle \\
 & =\sum_{i,j}\left(\Yb_{ij}-Y_{ij}\right)\left((p-q)\Yb_{ij}-\frac{p-q}{2}\Jb_{ij}\right).
\end{align*}
For each $(i,j)\in[n]\times[n]$, observe that 
\begin{itemize}
\item If $\Yb_{ij}=1$, then $\Yb_{ij}-Y_{ij}\ge0$ since $Y_{ij}\le1$,
and $\Jb_{ij}=1$. Moreover, $(p-q)\Yb_{ij}-\frac{p-q}{2}\Jb_{ij}=(p-q)-\frac{p-q}{2}=\frac{p-q}{2}\Jb_{ij}$.
It follows that 
\[
\left(\Yb_{ij}-Y_{ij}\right)\left((p-q)\Yb_{ij}-\frac{p-q}{2}\Jb_{ij}\right)=\frac{p-q}{2}\Jb_{ij}\left|\Yb_{ij}-Y_{ij}\right|.
\]
\item If $\Yb_{ij}=0$, then $\Yb_{ij}-Y_{ij}\le0$ since $Y_{ij}\ge0$.
Moreover, $(p-q)\Yb_{ij}-\frac{p-q}{2}\Jb_{ij}=-\frac{p-q}{2}\Jb_{ij}$.
It follows that 
\[
\left(\Yb_{ij}-Y_{ij}\right)\left(\mathbb{E}A_{ij}-\frac{p+q}{2}\right)=\frac{p-q}{2}\Jb_{ij}\left|Y_{ij}^{*}-Y_{ij}\right|.
\]
\end{itemize}
Combining, we see that 
\[
F_{1}=\frac{p-q}{2}\sum_{i,j}\left|Y_{ij}^{*}-Y_{ij}\right|\cdot\Jb_{ij}=\frac{p-q}{2}\onenorm{\Jb\circ(\Yb-Y)}.
\]

Turning to the $F_{2}$ term, we note that $\PP\left(\EE A-\frac{p+q}{2}J\right)$
is supported on $\supp(\Js)$ thanks to (\ref{eq:PP_coeff}), and
$\PP\left(\Js\circ B\right)=\Js\circ B$ for any matrix $B$. Hence
\begin{align*}
F_{2} & =\left\langle \Yb-Y,\PP\left(\EE A-\frac{p+q}{2}J\right)\right\rangle \\
 & =\left\langle \Yb-Y,J\circ\PP\left(\EE A-\frac{p+q}{2}J\right)\right\rangle \\
 & =\left\langle J\circ(\Yb-Y),\PP\left(\EE A-\frac{p+q}{2}J\right)\right\rangle \\
 & =\left\langle \PP\left(J\circ(\Yb-Y)\right),\EE A-\frac{p+q}{2}J\right\rangle \\
 & =\left\langle J\circ(\Yb-Y),\EE A-\frac{p+q}{2}J\right\rangle \\
 & =\left\langle \Js\circ(\Yb-\Yhat),(p-q)\Ys-\frac{p-q}{2}\Js\right\rangle ,
\end{align*}
where the last step follows from (\ref{eq:PP_coeff}). Further observe
that $\Js\circ(\Yb-Y)$ is non-positive on all entries, and $(p-q)\Ys-\frac{p-q}{2}\Js$
is positive on $\supp(\Ys)$ and non-positive elsewhere. Ignoring
the non-negative contribution from outside $\supp(\Ys)$, we obtain
that
\begin{align*}
F_{2} & \ge\left\langle \Js\circ(\Yb-Y),\Ys\circ\left((p-q)\Ys-\frac{p-q}{2}\Js\right)\right\rangle \\
 & =\left\langle \Yb-Y,\Ys\circ\left((p-q)\Ys-\frac{p-q}{2}\Js\right)\right\rangle \\
 & =\left\langle \Yb-Y,\frac{p-q}{2}\Ys\right\rangle \\
 & =\left\langle \Yb-Y,\PP\left(\frac{p-q}{2}\Ys\right)\right\rangle \\
 & =\left\langle \PP(\Yb-Y),\frac{p-q}{2}\Ys\right\rangle \\
 & \ge-\nucnorm{\PP(\Yb-Y)}\opnorm{\frac{p-q}{2}\Ys} \\
 &= -\nucnorm{\PP(\Yb-Y)}\frac{p-q}{2}\ss.
\end{align*}
Combining our calculations for the $F_1$ and $F_2$ terms completes the proof.
\end{proof}

\begin{lem}
\label{lem:gap_exact_recovery_noise_term_bound}
For all $Y$ feasible to the recovery SDP~\eqref{eq:SDP_regularized_original},
\begin{align*}
        \left |   \left\langle \Yb - Y, W \right\rangle \right | \leq \infnorm{\PT W} \onenorm{\Jb \circ (Y - \Yb)} + \opnorm{W}\nucnorm{\PP(Y - \Yb)}.
\end{align*}
\end{lem}
\begin{proof}
    Let $D:=Y-Y^{*}$. Since $\PT W$ is supported on $\supp(\Jb$), we
have 
\begin{align*}
\left\langle \PT D,W\right\rangle  & =\left\langle D,\PT W\right\rangle \\
 & =\left\langle D,\Jb\circ\left(\PT W\right)\right\rangle \\
 & =\left\langle \Jb\circ D,\PT W\right\rangle ,
\end{align*}
hence 
\begin{align}
\left|\left\langle Y-Y^{*},W\right\rangle \right| & =\left|\left\langle \PT D,W\right\rangle +\left\langle \PP D,W\right\rangle \right|\nonumber \\
 & \le\onenorm{\Jb\circ D}\infnorm{\PT W}+\nucnorm{\PP D}\opnorm W,\label{eq:dual_norm_bound}
\end{align}
where the last step follows from the generalized Holder's inequality
since $\onenorm{\cdot}$ ($\nucnorm{\cdot}$, resp.) is the dual norm
of $\infnorm{\cdot}$ ($\opnorm{\cdot}$, resp.).
\end{proof}

\begin{proof}[Proof of Theorem \ref{thm:clustering_with_a_gap}]
First we bound $\opnorm{W}$ and $\infnorm{\PT W}$. 
By Lemma \ref{lem:op_norm_concentration_entire_noise} we have that $\opnorm{W} \leq B_1 \sqrt{pn}$ with probability at least $1-O(n^{-3})$.

For $\infnorm{\PT W}$, first fix $j$ and $i$ and let $w$ be the $i$th row of $W$. First note by definition of $Q$ that if $j \in V_\ell$ for some $\ell$ such that $s_\ell < \sb$, then $\left(WQQ^\top \right)_{ij} = 0$. Otherwise, letting $j \in V_\ell$ where $s_\ell \geq \sb$,
\begin{align*}
    \left(WQQ^\top \right)_{ij} &= \left\langle w^{(\ell)}, \frac{1}{s_\ell} \onevec_{s_\ell}\right \rangle = \frac{1}{\sqrt{s_\ell}} \left\langle w^{(\ell)}, \frac{1}{\sqrt{s_\ell}} \onevec_{s_\ell}\right \rangle.
\end{align*}
By Bernstein's inequality, since all entries of $w$ have variance $\leq p$, mean zero, and are bounded in magnitude by $1$, we have
\begin{align}
    \P \left( \left|\left\langle w^{(\ell)}, \frac{1}{\sqrt{s_\ell}} \onevec_{s_\ell}\right \rangle \right| \geq t \right) & \leq 2 \exp \left( - \frac{t^2/2}{p \twonorm{\frac{1}{\sqrt{s_\ell}} \onevec_{s_\ell}}^2 + \frac{t}{3}\infnorm{\frac{1}{\sqrt{s_\ell}} \onevec_{s_\ell}}}\right) \nonumber\\
    & \leq 2 \exp \left( - \frac{t^2/2}{p  + \frac{t}{3\sqrt{s_\ell}}}\right) \nonumber\\
    & = 2 \exp \left( - \frac{pB_5^2 (\log n)/2}{p  + \frac{B_5 \sqrt{p \log n}}{3\sqrt{s_\ell}}}\right) \label{eq:onevec_ip_bernstein_failure_prob}
\end{align}
where in the last line we choose $t = B_5\sqrt{p \log n}$ 
(similarly to the proof of Lemma \ref{lem:LOO_evec_noise_inner_prod_concentration}).
We have $s_\ell \geq \sb$, and by assumption we have $\frac{p-q}{2}\left(\sb - \ss \right) \geq \sqrt{pn \log n}$, which implies that
\begin{align*}
    p s_\ell \geq p \sb \geq \frac{p-q}{2}\left(\sb - \ss \right) \geq \sqrt{pn \log n} \geq \sqrt{ps_\ell \log n}\implies \frac{ \sqrt{p \log n}}{\sqrt{s_\ell}} \leq p
\end{align*}
so the leading term of the denominator is $p$, and therefore by taking $B_5$ large enough (the same value as in Lemma \ref{lem:LOO_evec_noise_inner_prod_concentration}) we have that the expression~\eqref{eq:onevec_ip_bernstein_failure_prob} is $O(n^{-5})$. Now taking a union bound over all $i$ and $j$ ($O(n^2)$ pairs), we have that
\[\infnorm{WQQ^\top} \leq \sup_{s_\ell \geq \sb} \frac{1}{\sqrt{s_\ell}} B_5\sqrt{p \log n} \leq B_5 \sqrt{\frac{p \log n}{\sb}}\]
with probability at least $1 - O(n^{-3})$.

Notice that $\infnorm{QQ^\top W QQ^\top} \leq \infnorm{ W QQ^\top}$ since
\[QQ^\top = \diag \left(\frac{1}{s_1}J_{s_1 \times s_1}, \dots, \frac{1}{\sb}J_{\sb \times \sb},0_{\ss \times \ss},\dots,0_{s_K \times s_K}\right)\]
which implies that $(QQ^\top W QQ^\top)_{ij}$ is an average of some entries from column $j$ of $W QQ^\top$ (or $0$). Using triangle inequality and then this fact, we have
\[
\infnorm{\PT{W}} \leq \infnorm{ W QQ^\top} + \infnorm{  QQ^\top W} + \infnorm{QQ^\top W QQ^\top} \leq 3 \infnorm{WQQ^\top} \leq 3B_5 \sqrt{\frac{p \log n}{\sb}}.
\]

We now assume that we are in the event that the bounds $\opnorm{W} \leq B_1 \sqrt{pn}$ and $\infnorm{\PT{W}} \leq 3B_5 \sqrt{\frac{p \log n}{\sb}}$, which by the union bound occurs with probability at least $1-O(n^{-3})$.

Starting by using the optimality of the SDP solution $\Yhat$, we have
\begin{align}
    0 &\geq \left\langle \Yb - \Yhat,  A - \frac{p+q}{2}J \right\rangle - \lambda \tr \left( \Yb - \Yhat \right) \\
    & = \left\langle \Yb - \Yhat, \E A - \frac{p+q}{2}J \right\rangle +   \left\langle \Yb - \Yhat, W \right\rangle   - \lambda \tr(\Yb - \Yhat) \label{eq:gap_exact_recovery_basic_ineq}
\end{align}
where recall we define $W = A - \E A$. Our first objective will be to lower bound the final expression by a multiple of $\onenorm{\Jb \circ (\Yhat - \Yb)}$, which will then imply that $\onenorm{\Jb \circ (\Yhat - \Yb)} = 0$ and thus $\Jb \circ \Yhat = \Yb$. We accomplish this by using Lemmas \ref{lem:gap_exact_recovery_trace_term_bound}, \ref{lem:gap_exact_recovery_signal_term_bound}, and \ref{lem:gap_exact_recovery_noise_term_bound}. Using their bounds in (\ref{eq:gap_exact_recovery_basic_ineq}) and combining terms we have
\begin{align*}
    0 &\geq \left\langle \Yb - \Yhat, \E A - \frac{p+q}{2}J \right\rangle +   \left\langle \Yb - \Yhat, W \right\rangle   - \lambda \tr(\Yb - \Yhat)\\
    & \geq \left(\frac{p-q}{2} - \infnorm{\PT W} - \frac{\lambda}{\sb}\right)\onenorm{\Jb \circ (\Yhat - \Yb)} + \left( \lambda - \opnorm{W} - \frac{p-q}{2}\ss \right)\nucnorm{\PP(\Yhat - \Yb)} \\
    & \geq  \left(\frac{p-q}{2} -3B_5 \sqrt{\frac{p \log n}{\sb}} - \frac{\lambda}{\sb}\right)\onenorm{\Jb \circ (\Yhat - \Yb)} 
\end{align*}
using the fact that $\lambda = B_1 \sqrt{pn}+\frac{p-q}{2}\ss> \opnorm{W} + \frac{p-q}{2}\ss$. Therefore if
$\frac{p-q}{2}\sb > 3B_5\sqrt{p \sb \log n} + \lambda$
we can conclude that $\Jb \circ \Yhat = \Yb$. 

Now we show that also $\Js \circ \Yhat = 0$ or equivalently that $\Yhat = \Yb$. Since we know $\Jb \circ \Yhat = \Yb$, it suffices to show that for any feasible $Y$ with $\Jb \circ Y = \Yb$, we have $\left \langle Y - \Yb, A - \frac{p+q}{2}J \right \rangle - \lambda \tr(Y - \Yb) \leq 0$. Starting with the fact that $\Jb \circ Y = \Yb$,
\begin{align*}
    \left \langle Y - \Yb, A - \frac{p+q}{2}J \right \rangle - \lambda \tr(Y - \Yb) &= \left \langle \Js \circ Y , \Js \circ \left( A - \frac{p+q}{2}J \right)\right \rangle - \lambda \tr (\Js \circ Y) \\
    &=  \left \langle \Js \circ Y , \Js \circ  \frac{p-q}{2}\Js \right \rangle + \left \langle \Js \circ Y , \Js \circ  W \right \rangle - \lambda \tr (\Js \circ Y) \\
    & \leq \nucnorm{\Js \circ Y} \opnorm{\frac{p-q}{2}\Js} + \nucnorm{\Js \circ Y} \opnorm{W} - \lambda \tr (\Js \circ Y) \\
    &= \left(\opnorm{\frac{p-q}{2}\Js} + \opnorm{W} - \lambda \right) \tr (\Js \circ Y) \\
    &= \left(\frac{p-q}{2}\ss + \opnorm{W} - \lambda \right) \tr (\Js \circ Y) \numberthis \label{eq:gap_exact_recovery_lower_block_ineq}
\end{align*}
where we then used the generalized Holder's inequality and then fact that $\nucnorm{\Js \circ Y} = \tr (\Js \circ Y)$ since $\Js \circ Y \succcurlyeq 0$, which follows from Schur's Lemma since $Y, \Js \succcurlyeq 0$. Because $\lambda > \opnorm{W} - \frac{p-q}{2}\ss$, (\ref{eq:gap_exact_recovery_lower_block_ineq}) establishes that $\left \langle Y - \Yb, A - \frac{p+q}{2}J \right \rangle - \lambda \tr(Y - \Yb) \leq 0$, and furthermore if $\Js \circ \Yhat \neq 0$ then we will have $\left \langle Y - \Yb, A - \frac{p+q}{2}J \right \rangle - \lambda \tr(Y - \Yb) < 0$ (since then $\tr (\Js \circ Y) = \nucnorm{\Js \circ Y} > 0$), establishing uniqueness.

Finally we establish the claimed recovery condition by noting that if $\frac{p-q}{2}\left(\sb-\ss\right) \geq (3B_5 + B_1)\sqrt{pn \log n}$, then we will have
\begin{align*}
    \frac{p-q}{2}\sb \geq \frac{p-q}{2} \ss + (3B_5 + B_1)\sqrt{pn \log n} > 3B_5\sqrt{p \sb \log n} + \lambda
\end{align*}
which was shown above to be sufficient.
\end{proof}

\begin{proof}[Proof of Lemma \ref{lem:oracle_zero_ones_case}]
    Lemma \ref{lem:oracle_zero_ones_case} is essentially a special case of Theorem \ref{thm:clustering_with_a_gap}, but since we have a slightly different choice of $\lambda$, we must repeat the proof. 

    Fix a cluster $k$ such that $\frac{p-q}{2}s \geq \left(1+\frac{B_6}{\kappa}\right)\kappa \sqrt{pn \log m}$, where we set $B_6 = 2 + 3B_5$.
    Since we treat $A^{(k)}$ as the entire adjacency matrix for the purposes of this proof, the conclusions of Lemmas \ref{lem:gap_exact_recovery_trace_term_bound}, \ref{lem:gap_exact_recovery_signal_term_bound}, and \ref{lem:gap_exact_recovery_noise_term_bound} become that for all $Y$ feasible to the $k$th oracle SDP~\eqref{eq:SDP_regularized_oracle}, we have
    \begin{gather}
        - \tr(J_{s_k \times s_k} - Y) \geq \nucnorm{\PP\left(Y - J_{s_k \times s_k}\right)} - \frac{1}{s_k} \onenorm{Y - J_{s_k \times s_k}} \label{eq:all_ones_case_trace_term_bound}\\
        \left\langle J_{s_k \times s_k} - Y, \E A^{(k)} - \frac{p+q}{2}J_{s_k \times s_k} \right\rangle \geq \frac{p-q}{2} \onenorm{Y - J_{s_k \times s_k}} \label{eq:all_ones_case_signal_term_bound} \\
         \left |   \left\langle J_{s_k \times s_k} - Y, W^{(k)} \right\rangle \right | \leq \infnorm{\PT W^{(k)}} \onenorm{Y - J_{s_k \times s_k}} + \opnorm{W^{(k)}}\nucnorm{\PP\left(Y - J_{s_k \times s_k}\right)} \label{eq:all_ones_case_noise_term_bound}
    \end{gather}
    because now $\Yb = \Jb = J_{s_k \times s_k}$, $\sb = s_k$, and there are no other clusters (so $\ss = 0$).

    Since we still assume we are in the event $E_1$, we have $\opnorm{W^{(k)}} \leq B_1 \sqrt{ps_k} \leq \lambda$ (we assume $\kappa \geq 1$ which ensures that $s_k$ meets the condition in~\eqref{eq:noise_op_norm_bound_cluster}, and also that $\kappa \geq B_1$). We also have by Lemma \ref{lem:onevec_noise_inner_prod_concentration_all_ones_case} that $\infnorm{  W^{(k)} \frac{1}{\sqrt{s_k}} \onevec_{s_k} } \leq B_5 \sqrt{p \log m}$ (since we also assume $\kappa\geq B_4$). Then since
    \[W^{(k)}VV^\top = W^{(k)} \frac{1}{s_k}J_{s_k \times s_k} = \frac{1}{\sqrt{s_k}} W^{(k)} \frac{1}{\sqrt{s_k}}\onevec_{s_k}\onevec_{s_k}^\top\]
    we have $\infnorm{W^{(k)} \frac{1}{s_k} J_{s_k \times s_k}} \leq \frac{B_5\sqrt{p \log m}}{\sqrt{s_k}}$.
    Identically to the proof of \ref{thm:clustering_with_a_gap} we have $\infnorm{VV^\top W^{(k)} VV^\top} \leq \infnorm{W^{(k)} VV^\top}$, so we conclude that
    \[\infnorm{\PT W^{(k)}} \leq 3\infnorm{W^{(k)} \frac{1}{s_k} J_{s_k \times s_k}} \leq \frac{3B_5\sqrt{p \log m}}{\sqrt{s_k}}. \]
    
    Now like in the proof of Theorem \ref{thm:clustering_with_a_gap}, using the optimality of the $k$th oracle SDP solution $\Yhat^k$ and then plugging in the bounds~\eqref{eq:all_ones_case_trace_term_bound},~\eqref{eq:all_ones_case_signal_term_bound}, and~\eqref{eq:all_ones_case_noise_term_bound}, we have
    \begin{align}
    0 &\geq \left\langle J_{s_k \times s_k} - \Yhat^k,  A^{(k)} - \frac{p+q}{2}J_{s_k \times s_k} \right\rangle - \lambda \tr \left( J_{s_k \times s_k} - \Yhat^k \right) \nonumber \\
    & = \left\langle J_{s_k \times s_k} - \Yhat^k, \E A^{(k)} - \frac{p+q}{2}J_{s_k \times s_k} \right\rangle +   \left\langle J_{s_k \times s_k} - \Yhat^k, W^{(k)} \right\rangle   - \lambda \tr(J_{s_k \times s_k} - \Yhat^k)\nonumber \\
    & \geq \left( \lambda - \opnorm{W^{(k)}}\right) \nucnorm{\PP\left(\Yhat^k - J_{s_k \times s_k}\right)}  + \left( \frac{p+q}{2}-\infnorm{\PT W^{(k)}}- \frac{\lambda}{s_k}\right) \onenorm{\Yhat^k - J_{s_k \times s_k}} \nonumber\\
    & \geq \left( \frac{p+q}{2}-\frac{3B_5\sqrt{p \log m}}{\sqrt{s_k}}- \frac{\lambda}{s_k}\right) \onenorm{\Yhat^k - J_{s_k \times s_k}} \label{eq:key_inequality_oracle_all_ones_case}
    \end{align}
    using the fact that $\lambda \geq \opnorm{W^{(k)}}$. Now since $\frac{p+q}{2}s_k \geq \left(1+\frac{B_6}{\kappa}\right)\kappa \sqrt{pn \log m} = \kappa \sqrt{pn \log m}+ (2 + 3B_5)\sqrt{pn \log m}$, using the fact that we assumed $\lambda \leq (\kappa+1)\sqrt{pn \log m}$ almost surely, we have
    \begin{align*}
        \frac{p-q}{2}s_k &\geq (\kappa+1)\sqrt{pn \log m} + 3B_5\sqrt{pn \log m} + \sqrt{pn \log m} \\
        & > \lambda + 3B_5\sqrt{ps_k \log m}
    \end{align*}
    and plugging this back in to~\eqref{eq:key_inequality_oracle_all_ones_case} implies that we must have $\onenorm{\Yhat^k - J_{s_k \times s_k}} = 0$ and thus $\Yhat^k = J_{s_k \times s_k}$.

    The fact that there exist some $U^k, L^k$ which satisfy the oracle SDP KKT conditions~\eqref{eq:oracle_KKT} follows from Lemma \ref{lem:oracle_KKT_and_recovery_KKT} (as a solution of the oracle SDP KKT conditions~\eqref{eq:oracle_KKT} is necessary for the optimality of $\Yhat^k$). From the complementary slackness conditions~\eqref{eq:oracle_KKT_complementary_slackness_1} and the fact that $\Yhat^k_{ij} = 1$ for all $i,j$, we must have $L^k = 0$. Finally, it is not immediate that $U^k \succcurlyeq 0$, but we ensure this by finding $u \in \R^{s_k}$ such that $\diag(u) \succcurlyeq 0$ and such that $\diag(u)$ also satisfies the oracle SDP KKT conditions (still with $L^k = 0$). Define $u$ as the unique solution to $U^k \onevec_{s_k} = \diag(u)\onevec_{s_k}$. Clearly $u \geq 0$ since $U^k \geq 0$, so we have $\diag(u) \succcurlyeq 0$ and $\diag(u) \geq 0$. Also by definition of $u$ the complementarity condition~\eqref{eq:oracle_KKT_complementary_slackness_2} still holds with $\diag(u)$ in place of $U^k$, since $\Yhat^k = \onevec \onevec^\top$. This leaves us only to check the dual feasibility condition~\eqref{eq:oracle_KKT_dual_feasibility} that $\Ashift^{(k)} - \lambda I - \diag(u) + L^k \preccurlyeq 0$. Again since~\eqref{eq:oracle_KKT_complementary_slackness_2} holds with $\diag(u)$ in place of $U^k$, we have that $\onevec$ is an eigenvector of $\Ashift^{(k)} - \lambda I - \diag(u) + L^k$ (with eigenvalue $0$), and for all vectors $v$ orthogonal to $\onevec$, we have
    \begin{align*}
        v^\top \left( \Ashift^{(k)} - \lambda I - \diag(u) + L^k \right)v &= v^\top \left( \frac{p-q}{2}\onevec\onevec^\top + W^{(k)} - \lambda I - \diag(u) +L^k \right)v \\
        &= v^\top \left(W^{(k)} - \lambda I - \diag(u) \right)v\\
        & \leq v^\top \left(W^{(k)} - \lambda I \right)v\\
        & < 0
    \end{align*}
    using the facts that $v^\top \onevec = 0$, $L^k = 0$, and that $\lambda > \opnorm{W^{(k)}}$. Therefore $\Ashift^{(k)} - \lambda I - \diag(u) + L^k \preccurlyeq 0$ as desired.
\end{proof}

\subsection{Concentration Inequalities}
\label{sec:concentration_results}

\begin{thm}[Bernstein's Inequality \citep{vershynin_high-dimensional_2018}]
    Suppose $(X_i)_{i=1}^n$ are independent zero-mean random variables such that $|X_i|\leq M$ almost surely. Then for any $t \geq 0$,
    \begin{align*}
        \P\left( \sum_{i=1}^n X_i \geq t\right) \leq \exp \left( -\frac{ \frac{1}{2}t^2}{\sum_{i=1}^n \E[X_i^2] + \frac{1}{3}Mt}\right).
    \end{align*}
\end{thm}

\begin{thm}[\citep{bandeira2016sharp}]
\label{thm:bandeira_opnorm_bound}
Suppose $X$ is a symmetric $n \times n$ matrix with independent entries such that $\EE{X_{ij}}=0$ and $|X_{ij}|\leq B$ for all $1 \leq i,j \leq n$. Let $\nu = \max_i \sum_j \EE{X_{ij}^2}$. Then there exists an absolute constant $c > 0$ such that for any $t \geq 0$,
\begin{align*}
    \P\left(\opnorm{X} \geq 4 \sqrt{\nu }+t \right) \leq n\exp \left( -\frac{t^2}{cB^2}\right).
\end{align*}
\end{thm}

\begin{proof}[Proof of Lemma \ref{lem:op_norm_concentration}]
We simply check the conditions of Theorem \ref{thm:bandeira_opnorm_bound}. First we establish~\eqref{eq:noise_op_norm_bound_all_oracle_blocks}, which we will actually establish by showing that $\opnorm{W}\leq B_1 \sqrt{pn}$, which immediately implies~\eqref{eq:noise_op_norm_bound_all_oracle_blocks} (although we could simply state this bound on $\opnorm{W}$ in Lemma \ref{lem:op_norm_concentration}, we want the event in the lemma to be independent of the non-block-diagonal noise entries). We have $\EE{W_{ij}}=0$, $|W_{ij}| \leq \max\{p, 1-p\} = 1-p \leq 1$, and $\EE{W_{ij}^2} = p(1-p) \leq p$, so with $\nu = np$ and $B=1$ in Theorem \ref{thm:bandeira_opnorm_bound} we have 
\begin{align*}
    \P\left(\opnorm{W} \geq 4 \sqrt{pn }+t \right) \leq n\exp \left( -\frac{t^2}{c}\right).
\end{align*}
By our assumption that $p \geq \frac{\log m}{n}$, $\sqrt{pn} \geq \sqrt{\log m}$, so setting $t = 3\sqrt{c \log m}$ the above probability is $\leq n \exp(-9 \log m) = nm^{-9} \leq m^{-8}$ (since $n \leq m$). Now take $B_1 = 4+3\sqrt{c}$ to obtain~\eqref{eq:noise_op_norm_bound_all_oracle_blocks}.

To establish~\eqref{eq:noise_op_norm_bound_cluster}, following the same steps for a fixed cluster $k$ such that $\frac{p-q}{2}s_k \geq \sqrt{pn \log m}$ we will obtain
\begin{align*}
    \P\left(\opnorm{W^{(k)}} \geq 4 \sqrt{ps_k }+t \right) \leq s_k\exp \left( -\frac{t^2}{c}\right).
\end{align*}
Now we use the assumed size of the cluster to show that $\sqrt{ps_k }$ is the dominant term when we set $t = \Theta(\sqrt{\log m})$. We have \[ps_k \geq \frac{p-q}{2}s_k \geq \sqrt{pn \log m} \geq \log m\]
where the final inequality again uses that $p \geq \frac{\log m}{n}$, so taking square roots $\sqrt{ps_k} \geq \sqrt{\log m}$ and therefore we may again set $t = 3\sqrt{c \log m}$ and obtain a failure probability of $m^{-8}$.

To establish~\eqref{eq:noise_two_norm_bound} note $\twonorm{W^{(k)}_{:,j}} \leq \opnorm{W^{(k)}}$. Finally by taking a union bound, there are obviously fewer than $n\leq m$ clusters so the overall failure probability is $O(m^{-7})$.
\end{proof}

\begin{proof}[Proof of Lemma \ref{lem:infty_pert_bound_s5}]
A bound of this form first appeared in \cite{abbe2017entrywise}, but we use the version from \cite[Theorem 4.2]{chen2021spectral}. One technical issue is that \cite[Theorem 4.2]{chen2021spectral} uses a failure probability which is inverse polynomial in the size of the matrix in the theorem, which in our case will be $s_i$, but we seek a failure probability which is inverse polynomial in $m$. However, by inspecting the proof, thanks to the fact that $m \geq n \geq s_i$, the exact same argument goes through if we replace $\log s_i$ by $\log m$ in the conclusion as well as the conditions of \cite[Theorem 4.2]{chen2021spectral}.

Now fix a cluster $i$ such that $\frac{p-q}{2}s_i \geq \sqrt{pn \log m}$. 
As checked in the proof of Lemma \ref{lem:op_norm_concentration}, this implies that $\sqrt{ps_i} \geq \sqrt{\log m}$. Now we check the conditions of \cite[Theorem 4.2]{chen2021spectral}. Note $A^{(i)} - \frac{p+q}{2}J_{s_i \times s_i} = W^{(i)} + \frac{p-q}{2}J_{s_i \times s_i}$, and $\frac{p-q}{2}J_{s_i \times s_i}$ is rank $1$ and has top eigenvector $\frac{1}{\sqrt{s_i}} \onevec$, which has incoherence parameter $s_i \infnorm{\frac{1}{\sqrt{s_i}} \onevec}^2 = 1$ and top eigenvalue $\frac{p-q}{2}s_i$. Each entry of $W^{(i)}$ is independent with mean zero, variance $\leq p(1-p) \leq p$, and is almost surely bounded by $1$. \cite[Theorem 4.2]{chen2021spectral} requires that $\frac{1}{\sqrt{p s_i / \log m}} = O(1)$, which is satisfied due to our assumption on the size of $s_i$, since as mentioned this assumption implies $\sqrt{ps_i} \geq \sqrt{\log m}$.

Therefore by \cite[Theorem 4.2]{chen2021spectral} there exists an absolute constant $c>0$ such that if $\sqrt{p s_i \log m} \leq c \frac{p-q}{2}s_i$, there exists a constant $B_2$ such that with probability at least $1 - O(m^{-5})$ we have
\[\infnorm{v^i - \frac{1}{\sqrt{s_i}}\onevec} \leq 2B_2 \frac{\sqrt{p \log m}}{(p-q)s_i}.\]
We now let $B_3 = \frac{1}{c}$ and note that
\[\frac{p-q}{2}s_i \geq B_3 \sqrt{pn \log m} \implies \sqrt{p s_i \log m} \leq \sqrt{p n \log m} \leq \frac{1}{c} \frac{p-q}{2}s_i.\]
Furthermore by examining the proof of \cite[Theorem 4.2]{chen2021spectral} it is also established for all $j \in [s_i]$ that
\[\infnorm{v^{i,j} - \frac{1}{\sqrt{s_i}}\onevec} \leq 2B_2 \frac{\sqrt{p \log m}}{(p-q)s_i}\]
(in the same event where the bound on $v^i$ is shown to hold).

Now repeating this argument for each sufficiently large cluster and taking a union bound over all $O(n)\leq O(m)$ clusters, we obtain the desired conclusion.
\end{proof}

\begin{proof}[Proof of Lemma \ref{lem:LOO_evec_noise_inner_prod_concentration}]
We let $B_4 = \max\{B_2, B_3,1\}$.
Now we fix a cluster $k$ such that $\frac{p-q}{2}s_k\geq \max\{B_2, B_3\} \sqrt{pn \log m}$.
Under the event in Lemma \ref{lem:infty_pert_bound_s5}, for all $j \in [s_k]$ we have that
\begin{align*}
    \infnorm{v^{k,j} - \frac{1}{\sqrt{s_k}}\onevec} &\leq 2B_2 \frac{\sqrt{p \log m}}{(p-q)s_k} \\
    & = \frac{1}{\sqrt{s_k}} \frac{2}{(p-q)s_k} B_2 \sqrt{p s_k \log m} \\
    & \leq \frac{1}{\sqrt{s_k}}.
\end{align*}
Therefore we have that $\infnorm{v^{k,j}} \leq \frac{2}{\sqrt{s_k}}$. 

Now also fix $j$ and define $F^{k,j}$ to be the event where~\eqref{eq:LOO_evec_inf_norm_simple_bound} holds, and note that $F^{k,j}$ is independent of $w^{k,j}$. Therefore we can apply Bernstein's inequality conditionally on the event $F^{k,j}$ to obtain
\begin{align*}
    \P \left( \left| \left \langle w^{k,j}, v^{k,j}\right \rangle \right| \geq t \middle| F^{k,j} \right) \leq 2 \exp \left( \frac{-t^2 / 2}{p1 + \frac{t}{3} \frac{2}{\sqrt{s_k}} }\right)
\end{align*}
using the facts that $\twonorm{v^{k,j}}^2=1$, $\infnorm{v^{k,j}} \leq \frac{2}{\sqrt{s_k}} $ on $F^{k,j}$.
Now we can take $t = B_5 \sqrt{p \log m}$ and note that as we checked in the proof of Lemma \ref{lem:op_norm_concentration}, the assumption that $\frac{p-q}{2}s_k \geq B_4\sqrt{pn\log m} \geq \sqrt{pn \log m}$ implies $\sqrt{p s_k} \geq \sqrt{\log m}$ which ensures that
\[\frac{-t^2 / 2}{p + \frac{2t}{3 \sqrt{s_k}}} \leq \frac{-pB_5^2\log m /2 }{p + p\frac{2B_5\sqrt{\log m}}{3\sqrt{ps_k}}} \leq -\frac{B_5^2/2}{1 + \frac{2B_5}{3}}\log m\]
so we can make $\exp \left( \frac{-t^2 / 2}{p + \frac{2t}{3 \sqrt{s_k}} }\right) \leq m^{-5}$ by taking $B_5$ sufficiently large. Also, by inspecting the proof of Lemma \ref{lem:LOO_evec_noise_inner_prod_concentration}, the event where~\eqref{eq:infty_pert_bound_loo_s5} holds occurs with probability at least $1-O(m^{-4})$, and as shown above $F^{k,j}$ is guaranteed to hold when~\eqref{eq:infty_pert_bound_loo_s5} holds, so we have that $\P({F^{k,j}}^c) \leq O(m^{-4})$.

Therefore
\begin{align*}
    \P \left( \left| \left \langle w^{k,j}, v^{k,j}\right \rangle \right| \geq B_5 \sqrt{p \log m} \right) &\leq \P \left( \left| \left \langle w^{k,j}, v^{k,j}\right \rangle \right| \geq B_5 \sqrt{p \log m} \middle| F^{k,j} \right) + \P({F^{k,j}}^c) \\ 
    &\leq O(m^{-4})
\end{align*}
and now taking a union bound over at most $ n \leq m$ pairs $(k,j)$, we obtain the desired conclusion.
\end{proof}

\begin{proof}[Proof of Lemma \ref{lem:onevec_noise_inner_prod_concentration_all_ones_case}]
    Fix a cluster $k$ such that $\frac{p-q}{2}s_k \geq B_4\sqrt{pn \log m}$, and let $w$ be the $j$th row of $W^{(k)}$ for some fixed $j \in [s_k]$. Note that $\twonorm{\frac{1}{\sqrt{s_k}} \onevec_{s_k}} = 1$ and $\infnorm{\frac{1}{\sqrt{s_k}} \onevec_{s_k}} = \frac{1}{\sqrt{s_k}}$, so using Bernstein's inequality the same way as in the proof of Lemma \ref{lem:LOO_evec_noise_inner_prod_concentration} we obtain
    \begin{align*}
        \P \left( \left| \left \langle w, \frac{1}{\sqrt{s_k}} \onevec_{s_k}\right \rangle \right| \geq B_5 \sqrt{p \log m} \right) \leq O(m^{-5}).
    \end{align*}
    Now taking a union bound over at most $n \leq m$ pairs $(k,j)$ we obtain the desired conclusion.
\end{proof}

\begin{proof}[Proof of Lemma \ref{lem:noise_LOO_soln_inner_product_concentration}]
Fix a cluster $k$ satisfying $\frac{p-q}{2}s_k \geq B_4\sqrt{pn \log m}$ and fix $j \in [s_k]$.
Technically, we have not assumed that we are in the event $F^{k,j} := E_1^{k,j} \cap \left\{\lambda_1\left(\Ashift^{k,j}\right) > \lambda^k\right\}$, which is the event that the $(k,j)$th LOO relaxed oracle SDP (\ref{eq:SDP_relaxed_oracle_LOO}) is guaranteed to have rank-one solution $\Yhats^{k,j} =\yhats^{k,j} {\yhats^{k,j}}^\top$ by Lemma \ref{lem:SDP_relaxed_oracle_rank_one}. We simply extend the definition of $\yhats^{k,j}$ to be $0$ outside the event $F^{k,j}$. Then the desired bound trivially holds on ${F^{k,j}}^c$, so we have that
\begin{align*}
    \P \left( \left| \left \langle w^{k,j}, \frac{1}{\sqrt{s_k}}\yhats^{k,j}\right \rangle \right| \geq t \right) &= \P \left( \left| \left \langle w^{k,j}, \frac{1}{\sqrt{s_k}}\yhats^{k,j}\right \rangle \right| \geq t \middle| F^{k,j}\right) \P(F^{k,j}) \\
    &\leq \P \left( \left| \left \langle w^{k,j}, \frac{1}{\sqrt{s_k}}\yhats^{k,j}\right \rangle \right| \geq t \middle| F^{k,j}\right).
\end{align*}
Since $\infnorm{\yhats^{k,j}} \leq 1$ from Lemma \ref{lem:SDP_relaxed_oracle_rank_one}, we have $\infnorm{\frac{1}{\sqrt{s_k}}\yhats^{k,j}} \leq \frac{1}{\sqrt{s_k}}$ and $\twonorm{\frac{1}{\sqrt{s_k}}\yhats^{k,j}} \leq 1$. Since $\yhats^{k,j}$ and $F^{k,j}$ are independent of $w^{k,j}$, we can apply Bernstein's inequality to obtain that
\begin{align*}
    \P \left( \left| \left \langle w^{k,j}, \frac{1}{\sqrt{s_k}}\yhats^{k,j}\right \rangle \right| \geq t \right) & \leq \P \left( \left| \left \langle w^{k,j}, \frac{1}{\sqrt{s_k}}\yhats^{k,j}\right \rangle \right| \geq t \middle| F^{k,j}\right) \\
    &\leq 2 \exp \left( \frac{-t^2 / 2}{p1 + \frac{t}{3} \frac{1}{\sqrt{s_k}} }\right) \\
    & \leq O(m^{-5})
\end{align*}
by taking $t = B_5 \sqrt{p \log m}$ identically to the proof of Lemma \ref{lem:LOO_evec_noise_inner_prod_concentration}. Now taking a union bound over all $\leq n \leq m$ pairs $(k,j)$, we obtain that with probability at least $1-O(m^{-4})$, for all $k$ such that $\frac{p-q}{2}s_k \geq B_4\sqrt{pn \log m}$ and all $j \in [s_k]$, $\left| \left \langle w^{k,j}, \frac{1}{\sqrt{s_k}}\yhats^{k,j}\right \rangle \right| \leq B_5\sqrt{p \log m}$. Now multiplying both sides by $\sqrt{s_k}$, we obtain the desired conclusion.
\end{proof}

\begin{proof}[Proof of Lemma \ref{lem:offdiagonal_noise_concentration}]
    We will use Bernstein's inequality. Let $S$ be the set of all $k$ satisfying $\frac{p-q}{2}s_k \geq B_4\sqrt{pn \log m}$. Fix $k \in S$, fix another cluster $j \neq k$, and let $w$ be one row of $W^{(kj)}$. Note $w \in \R^{s_k}$, and all entries have mean zero, variance $q(1-q) \leq p$, and magnitude bounded by $\max\{q, 1-q\} \leq 1$. 

    By Lemma \ref{lem:x_infnorm_bound} that under the event $E_3$ we have $\infnorm{x^k} \leq \frac{2}{\sqrt{s_k}}$ and also note that $E_3$ is independent of $w$ (since $E_3$ only involves noise variables which are part of the diagonal blocks $(W^{(i)})_{i \in [K]}$). Therefore Bernstein's inequality gives that
    \begin{align*}
        \P \left( \left| w^\top x^k \right| \geq t \middle| E_3 \right) & \leq 2\exp \left( -\frac{t^2/2}{p + \frac{2t}{3\sqrt{s_k}}}  \right) \\
        & \leq O(m^{-5})
    \end{align*}
    by taking $t = B_5 \sqrt{p \log m}$ in the same way as in the proof of Lemma \ref{lem:LOO_evec_noise_inner_prod_concentration}. Taking a union bound over all clusters $j \neq k$ all rows ($\leq n\leq m$ in total), we have
    \begin{align*}
        \P \left( \sup_{j \neq k} \infnorm{W^{(jk)}x^k} \geq B_5 \sqrt{p \log m} \middle| E_3 \right) & \leq O(m^{-4})
    \end{align*}
    and taking a union bound over all $k \in S$
    \begin{align*}
        \P \left( \sup_{k \in S}\sup_{j \neq k} \infnorm{W^{(jk)}x^k} \geq B_5 \sqrt{p \log m} \middle| E_3 \right) & \leq O(m^{-3}).
    \end{align*}
    Therefore recalling that $\P(E_3^c) \leq O(m^{-3})$,
    \begin{align*}
        \P \left( \sup_{k \in S}\sup_{j \neq k} \infnorm{W^{(jk)}x^k} \geq B_5 \sqrt{p \log n}  \right) & \leq  \P \left( \sup_{k \in S}\sup_{j \neq k} \infnorm{W^{(jk)}x^k} \geq B_5 \sqrt{p \log n} \middle| E_3 \right) + \P(E_3^c) \\
        & \leq O(m^{-3}).
    \end{align*}
\end{proof}

\begin{proof}[Proof of Lemma \ref{lem:op_norm_concentration_entire_noise}]
During the proof of Lemma \ref{lem:op_norm_concentration} we already check that with probability at least $1-O(m^{-8})$, we have $\opnorm{W} \leq B_1 \sqrt{pn}$.
\end{proof}

\section{Proofs for Recursive Clustering}
\label{sec:recursive_clustering_proofs}
\begin{proof}[Proof of Theorem \ref{thm:recursive_clsutering}]
    First, we note that assuming the recovery SDP subroutine always succeeds, since in Algorithm \ref{alg:recursive_clustering} we only remove recovered clusters which are so large that they are guaranteed to be all-one, the sequence of clusters which will be recovered is actually known a priori. (If we removed clusters which were recovered but are mid-size, then since their recovery is highly noise-dependent, we would not be able to predict which mid-size clusters would be recovered in a given round. We believe a more careful analysis would have no trouble with this added complexity, but since the ultimate conclusions of Theorem \ref{thm:recursive_clsutering} would not change, we elect to only remove clusters which are guaranteed to be all-one.) 

    From an analysis standpoint, we have access to the list $(s_1, \dots, s_K)$ of all true cluster sizes. Let $Z_1 = (1, \dots, K)$ represent the original SBM instance. There is obviously an algorithm which, given $Z_\ell$, either recursively constructs $Z_{\ell+1}$ by removing all clusters from $Z_\ell$ which meet the guaranteed all-ones recovery threshold of $Z_\ell$ (case \ref{item:main_thm_all_ones_case} of Theorem \ref{thm:main_theorem_recovery_sdp}), or if there are no such clusters, simply terminates. This will produce a sequence $Z_1, \dots, Z_r$ as well as the sequence of instance sizes $n_1,\dots, n_r$. Now for each fixed $\ell \in [r]$, we apply Theorem \ref{thm:main_theorem_recovery_sdp} with $m = n$ to guarantee that all clusters of size large enough to meet the all-ones recovery threshold $C'(p,q) \sqrt{n_\ell \log n}$ will be successfully recovered by the recovery SDP with failure probability at most $O(n^{-3})$. Note that although their sizes are initially unknown to a user, after being recovered their sizes will be known, allowing us to verify that they meet the all-ones recovery threshold and then remove them as specified in Algorithm \ref{alg:recursive_clustering}. Taking a union bound over these $r \leq K \leq n$ rounds, we are guaranteed with probability at least $1-O(n^{-2})$ that Algorithm \ref{alg:recursive_clustering} will recover all clusters of size at least $C'(p,q) \sqrt{n_\ell \log n}$ at each round $\ell$.

    Now to prove Theorem \ref{thm:recursive_clsutering} we can simply analyze the sequence $Z_1, \dots, Z_r$. First, again the total number of rounds is $\leq K$ since we stop once no clusters are recovered in a given round.
    
    Next, considering the round $r$ where the algorithm terminates, since no clusters were recovered, this means that all clusters sizes in $Z_r$ must be below $C'(p,q) \sqrt{n_r \log n}$. Since there are at most $K$ total clusters in $Z_r$, this allows us to bound the total number of remaining nodes $n_r$ as
    \begin{align*}
        n_r &\leq K C'(p,q) \sqrt{n_r \log n} \\
        & \leq \frac{n^{\frac{1}{2}-\alpha}}{C'(p,q) \sqrt{\log n}} C'(p,q) \sqrt{n_r \log n} \\
        &= n^{\frac{1}{2}-\alpha} \sqrt{n_r}.
    \end{align*}
    Rearranging, we obtain that $n_r \leq n^{1-2\alpha}$.

    We prove the final part of the theorem by induction over rounds $\ell$. First, after $\ell = 0$ rounds, there will be all $n$ nodes remaining, matching the claim. Next, suppose inductively that after some fixed round $\ell$ we have $\leq n^{1-\alpha \sum_{i = 0}^{\ell-1}2^{- i}}$ unrecovered nodes remaining. 
    Then the recovery threshold for the next round will be $\leq T = C'(p,q) \sqrt{\log n} n^{\frac{1}{2}\left(1-\alpha \sum_{i = 0}^{\ell-1}2^{- i}\right)}$. The maximum number of nodes that can be below this threshold during the next round is less than or equal to
    \begin{align*}
        KT &\leq \frac{n^{\frac{1}{2}-\alpha}}{C'(p,q) \sqrt{\log n}} C'(p,q) \sqrt{\log n} n^{\frac{1}{2}\left(1-\alpha \sum_{i = 0}^{\ell-1}2^{- i}\right)} \\
        &= n^{\frac{1}{2}-\alpha} n^{\frac{1}{2}\left(1-\alpha \sum_{i = 0}^{\ell-1}2^{- i}\right)} \\
        &= n^{\frac{1}{2}-\alpha} n^{\frac{1}{2}-\alpha \sum_{i = 1}^{\ell}2^{- i}} \\
        &= n^{1-\alpha \sum_{i = 0}^{\ell}2^{- i}}
    \end{align*}
    proving the desired claim.
\end{proof}

\section{Proofs for Clustering With a Faulty Oracle}
\label{sec:clustering_with_a_faulty_oracle_proofs}

In this section we prove Theorem \ref{thm:faulty_oracle_clustering_size_param} on faulty oracle clustering with a target recovery size of $s$, we prove our instance-adaptive query complexity guarantees in Theorem \ref{thm:clustering_with_faulty_oracle_adaptive}, and then finally we prove our Theorem \ref{thm:faulty_oracle_clustering_improved_complexity_small_k} which achieves improved query complexity for the case that $K$ is bounded.

First, we state several concentration inequalities which will be used in this section.
\begin{thm}[Hoeffding's Inequality \citep{hoeffding1994probability}]
    Suppose $(X_i)_{i=1}^n$ are independent random variables such that $X_i \in [a_i, b_i]$ almost surely. Let $X = \sum_{i=1}^n X_i$ and $\mu = \E X$. Then for any $t \geq 0$,
    \begin{align*}
        \P\left( \left|X-\mu \right| \geq t\right) \leq 2\exp \left( -\frac{2 t^2}{\sum_{i=1}^n (b_i - a_i)^2}\right).
    \end{align*}
\end{thm}

\begin{thm}[Chernoff's Inequality for Bernoulli random variables \citep{mitzenmacher_probability_2005}]
    Suppose $(X_i)_{i=1}^n$ are independent random variables such that $X_i \in \{0,1\}$ almost surely. Let $X = \sum_{i=1}^n X_i$ and $\mu = \E X$. Then
    \begin{enumerate}
        \item For any $\delta \geq 0$, $\P\left( X \geq (1+\delta)\mu\right) \leq \exp \left( -\frac{\delta^2 \mu}{2 + \delta}\right)$.
        \item For any $0 < \delta < 1$, $\P\left( X \leq (1-\delta)\mu\right) \leq \exp \left( -\frac{\delta^2 \mu}{2}\right)$.
        \item For any $0 < \delta < 1$, $\P\left( \left|X-\mu \right| \geq \delta\mu\right) \leq 2\exp \left( -\frac{\delta^2 \mu}{3}\right)$.
    \end{enumerate}
\end{thm}

Now before proving our theorems on faulty oracle clustering, we first start by fixing a subsample $T$ and analyzing the majority voting procedure and the performance of the recovery SDP~\eqref{eq:SDP_regularized_original} on the subgraph on nodes $T$ to prove two lemmas which will be useful for all subsequent proofs.

\begin{lem}
    \label{lem:majority_voting_success}
        Let $T \subseteq [n]$ be fixed. There exists $C_1$ such that the following holds: With probability at least $1-O(n^{-3})$, for each cluster $k \in [K]$ such that $|V_k \cap T| \geq \frac{C_1 \log n}{\delta^2}$, letting $S'_k$ be the first $\frac{C_1 \log n}{\delta^2}$ nodes in $V_k \cap T$, for each node $i \not \in T$, if we query $(i, j)$ for all $j \in S'_k$, then:
        \begin{enumerate}
            \item If $i \in V_k$, then the majority of queries will be $1$.
            \item If $i \not \in V_k$, then the majority of queries will be $0$.
        \end{enumerate}
    \end{lem}
    \begin{proof}
        First we fix $S'_k$ and $i$. For each $j \in S'_k$ let $I_j$ be the outcome of the query $(i,j)$, and let $N = \sum_{j=1}^{|S'_k|}$ be the sum of the queries. First, if $i \in V_k$, then $I_j$ has distribution $\text{Bernoulli}(\frac{1}{2}+\frac{\delta}{2})$, so then by Hoeffding's inequality
        \begin{align*}
            \P \left( N \leq \frac{|S'_k|}{2}\right) &= \P \left( N - \E N \leq -\frac{\delta |S'_k|}{2} \right) \\
            &\leq \exp \left( \frac{-2 \left( \frac{\delta |S'_k|}{2}\right)^2}{|S'_k|}\right) \\
            &= \exp\left(-\frac{1}{2} \delta^2 |S'_k| \right)
        \end{align*}
        which is $O(n^{-5})$ if $|S'_k| = \frac{C_1 \log n}{\delta^2}$ for sufficiently large $C_1$ (in fact $C_1 = 10$). For the other case that $i \not \in V_k$, then $I_j$ has distribution $\text{Bernoulli}(\frac{1}{2}-\frac{\delta}{2})$, so again by Hoeffding's inequality
        \begin{align*}
            \P \left( N \geq \frac{|S'_k|}{2}\right) &= \P \left( N - \E N \geq \frac{\delta |S'_k|}{2} \right) \\
            &\leq \exp\left(-\frac{1}{2} \delta^2 |S'_k| \right)
        \end{align*}
        which is again $O(n^{-5})$.

        Now taking a union bound over all $\leq n$ nodes $i \not \in T$ and all $\leq n$ clusters $k$ satisfying $|V_k \cap T| \geq \frac{C_1 \log n}{\delta^2}$, we obtain the desired conclusion with a failure probability of at most $O(n^{-3})$.
    \end{proof}

\begin{lem}
\label{lem:faulty_oracle_recovery_SDP_success}
    Let $T \subseteq [n]$ be fixed. There exists $C$ such that with probability at least $1-O(n^{-3})$, if we query all $i,j \in T$ to form an adjacency matrix $A$ and apply the recovery SDP~\eqref{eq:SDP_regularized_original} with $m=n$, then
    \begin{enumerate}
        \item All clusters $k$ such that $|T \cap V_k| \geq \frac{C \sqrt{|T| \log n}}{\delta}$ will be recovered.
        \item No clusters $k$ such that $|T \cap V_k| \leq \frac{C \sqrt{|T| \log n}}{3\delta}$ will be recovered (and the blocks of the recovery SDP solution will be zero).
    \end{enumerate}
\end{lem}
\begin{proof}
    This follows immediately from Theorem \ref{thm:main_theorem_recovery_sdp}, once we note that by querying all $i,j \in T$, this is equivalent to an SBM instance with $p = \frac{1+\delta}{2} $ and $q = \frac{1-\delta}{2}$, and so $p-q = \delta$.
\end{proof}

With Lemmas \ref{lem:majority_voting_success} and \ref{lem:faulty_oracle_recovery_SDP_success} we can now easily prove Theorem \ref{thm:faulty_oracle_clustering_size_param}.
\begin{proof}[Proof of Theorem \ref{thm:faulty_oracle_clustering_size_param}]
    Now we can simply unfix $T$ and find a high-probability event under which Algorithm \ref{alg:clustering_with_faulty_oracle_meta_algorithm} succeeds.
    \begin{lem}
    \label{lem:T_subsample_good_event}
        There exists $C_2$ such that for any $s \geq C_2\frac{\sqrt{n \log n}}{\delta}$, if we take $\gamma = C_2 \frac{n \log n}{s^2 \delta^2}$, then with probability at least $1-O(n^{-3})$:
        \begin{enumerate}
            \item $\frac{7}{8}\gamma n \leq |T| \leq \frac{9}{8}\gamma n$.
            \item For all clusters $k$ such that $s_k \geq s$, $\frac{7}{8}\gamma s_k \leq |T \cap V_k| \leq \frac{9}{8}\gamma s_k$.
            \item For all clusters $k$ such that $s_k \geq s$, $|T \cap V_k| \geq \frac{C \sqrt{|T| \log n}}{\delta}$.
        \end{enumerate}
    \end{lem}
    \begin{proof}
        We let $I_i$ be an indicator for the event that node $i$ is included in $T$. Then $I_i$ is distributed as $\text{Bernoulli}(\gamma)$ and $|T| = \sum_{i=1}^n I_i$. Then by Chernoff's inequality,
        \begin{align*}
            \P \left( \left| |T|-\gamma n \right| \geq  \left(1+\frac{1}{8}\right) \gamma n \right) & \leq \exp \left( -\frac{\gamma n}{8^2 \cdot 3} \right).
        \end{align*}
        Now since $\frac{\gamma n}{8^2 \cdot 3} \geq \frac{C_2}{8^2 \cdot 3} \frac{n^2}{s^2} \frac{\log n}{\delta^2} \geq \frac{C_2}{8^2 \cdot 3} \frac{\log n}{\delta^2} \geq 5 \log n$ for sufficiently large $C_2$, we have that $\frac{7}{8}\gamma n \leq |T| \leq \frac{9}{8}\gamma n$ with probability at least $1-O(n^{-5})$. An identical argument holds for all clusters $k$ such that $s_k \geq s$ (except we will use $\frac{s_k}{s} \geq 1$ in place of $\frac{n}{s}\geq 1$). Union bounding over all these events, we have a failure probability which is less than $O(n^{-3})$.

        Lastly, under these events, for each $k$ such that $s_k \geq s$ we have
        \begin{align*}
            |T \cap V_k| \geq \frac{7}{8} \gamma s_k \geq \frac{7}{8}  \frac{s_k}{s} \frac{\sqrt{C_2 \gamma n \log n}}{\delta} \geq \frac{7}{8} \frac{\sqrt{C_2 \gamma n \log n}}{\delta} \geq \frac{C \sqrt{\frac{9}{8}}\sqrt{\gamma n \log n}}{\delta} \geq \frac{C \sqrt{|T| \log n}}{\delta}
        \end{align*}
        where the second-last inequality holds as long as $C_2$ is sufficiently large.
    \end{proof}

Now we can complete the proof of Theorem \ref{thm:faulty_oracle_clustering_size_param}. For any fixed $T \subseteq [n]$ which satisfies the conclusions of Lemma \ref{lem:T_subsample_good_event}, from Lemma \ref{lem:faulty_oracle_recovery_SDP_success} we have that all clusters $k$ of size at least $s$ will have $V_k \cap T$ correctly recovered by the recovery SDP with failure probability $O(n^{-3})$. Then by taking the first $\frac{C_1 \log n}{\delta^2}$ nodes from each recovered (sub)clusters, Lemma \ref{lem:majority_voting_success} guarantees that with failure probability $O(n^{-3})$ we will successfully recover each entire cluster. Note that Lemma \ref{lem:T_subsample_good_event} guarantees all clusters $k$ with $s_k \geq s$ will have $|T \cap V_k| \geq \frac{C_1 \log n}{\delta^2}$, since we have
\[|T \cap V_k| \geq \frac{7}{8} \gamma n = \frac{7}{8}C_2 \frac{n^2}{s^2}\frac{\log n}{\delta^2} \geq \frac{7}{8}C_2  \frac{\log n}{\delta^2} \geq C_1 \frac{\log n}{\delta^2}.\]
Now unfixing $T$, the event in Lemma \ref{lem:T_subsample_good_event} occurs with probability at least $1 -O(n^{-3})$, giving an overall failure probability of $O(n^{-3})$.

Finally, to check the sample complexity, we first note that by Lemma \ref{lem:T_subsample_good_event}, we have
\[|T| \leq \frac{9}{8}\gamma n = \frac{9}{8}C_2 \frac{n^2 \log n}{s^2 \delta^2}\]
so querying all pairs in $T$ requires $O\left(|T|^2 \right) \leq O\left(\frac{n^4 \log^2 n}{s^4 \delta^4} \right)$ queries. Next, to bound the number of queries used for the majority voting phase, each subcluster of $T$ which gets recovered will lead to at most $n \cdot \frac{C_1 \log n}{\delta^2} \leq O\left( \frac{n \log n}{\delta^2}\right)$ queries, so it remains to bound the number of subclusters contained in $T$ that are recovered. From Lemma \ref{lem:faulty_oracle_recovery_SDP_success} we know that for the subcluster of cluster $k$ to be recovered we need $|T \cap V_k| \geq \frac{C \sqrt{|T| \log n}}{3 \delta}$. Therefore under the event from Lemma \ref{lem:T_subsample_good_event} the maximum number of such clusters is
\begin{align*}
    \frac{|T|}{\frac{C \sqrt{|T| \log n}}{3 \delta}}  \leq O\left( \frac{\delta \sqrt{\gamma n}}{\sqrt{\log n}}\right) \leq O\left( \frac{n}{s}\right).
\end{align*}
Therefore the majority voting procedure requires $O\left( \frac{n^2 \log n}{s\delta^2}\right)$ queries, leading to the claimed total query complexity.
\end{proof}

Next we present the proof of the guarantees for our instance-adaptive clustering algorithm.

\begin{proof}[Proof of Theorem \ref{thm:clustering_with_faulty_oracle_adaptive}]
   We reuse the same value of $C_1$ as from Theorem \ref{thm:faulty_oracle_clustering_size_param}, and we require $C_2$ to be at least as large as required in Theorem \ref{thm:faulty_oracle_clustering_size_param}. Let $\underline{r}$ be the smallest $r$ such that $\hat{s}_r = \frac{n}{2^{r-1}} \geq s_1$. Now we can outline our proof strategy. With this choice of $\underline{r}$, we are essentially guaranteed by Theorem \ref{thm:faulty_oracle_clustering_size_param} and the form of $\gamma_{\underline{r}}$ that cluster $1$ would be recovered on round $\underline{r}$ (if that round is reached). This immediately implies the claimed query complexity since the algorithm will terminate no later than round $\underline{r}$. Then we will ensure that no clusters which are significantly smaller than $s_1$ will be recovered during any of the rounds before $\underline{r}$ (inclusive).

    Our first step will be to analyze the subsampling procedure. We define $T_r$ for all $r = 1, \dots, \underline{r}$ as the subsample of $[n]$ which would be sampled if we reached round $r$.
    \begin{lem}
    \label{lem:T_subsample_good_event_adaptive}
        There exist absolute constants $C_2, C_3$ such that if we take $\gamma_r = C_2 \frac{n \log n}{\hat{s}_r^2 \delta^2}$ for each $r\in [\underline{r}]$, then with probability at least $1-O(n^{-2})$:
        \begin{enumerate}
            \item For all $r \in [\underline{r}]$, $\frac{7}{8}\gamma_r n \leq |T_r| \leq \frac{9}{8}\gamma_r n$.
            \item For all $r \in [\underline{r}]$ and for all $k$ such that $s_k \leq \frac{s_1}{C_3}$, we have $|T_r \cap V_k| \leq \frac{C \sqrt{|T_r| \log n}}{3 \delta}$.
            \item We have $|T_{\underline{r}} \cap V_1| \geq \frac{C \sqrt{|T_{\underline{r}}|\log n}}{\delta}$.
        \end{enumerate}
    \end{lem}
    \begin{proof}
        For each fixed $r$, $\frac{7}{8}\gamma_r n \leq |T_r| \leq \frac{9}{8}\gamma_r n$ follows by applying Lemma \ref{lem:T_subsample_good_event}, so then we can obtain the first item by applying a union bound over all $\underline{r} \leq n$ rounds.

        The third item also follows from the third item of Lemma \ref{lem:T_subsample_good_event} (the conditions of which hold for round $\underline{r}$ and for cluster $1$ by the definition of $\underline{r}$). 

        Now we check the second item. First we fix a cluster $k$ such that $s_k \leq \frac{s_1}{C_3}$ and a round $r \leq \underline{r}$. We know from the first item that $|T_r| \geq \frac{7}{8}\gamma_r n$, so it suffices to show that $|T_r \cap V_k| \leq \sqrt{\frac{7}{8}}\frac{C \sqrt{\gamma_r n \log n}}{3 \delta}$. Now for each $i \in V_k$ we let $I_i$ be an indicator for the event that node $i$ is included in $T_r$ and note that the $I_i$ are independent, they have distribution $\text{Bernoulli}(\gamma_r)$, and their sum is $|T_r \cap V_k|$. We calculate that
        \begin{align*}
            \P\left(|T_r \cap V_k| \geq \sqrt{\frac{7}{8}}\frac{C \sqrt{\gamma_r n \log n}}{3 \delta} \right) &= \P\left(|T_r \cap V_k| - \E |T_r \cap V_k| \geq \sqrt{\frac{7}{8}}\frac{C \sqrt{\gamma_r n \log n}}{3 \delta}  - \gamma_r s_k \right) \\
            &= \P\left(|T_r \cap V_k| - \E |T_r \cap V_k| \geq \left(\sqrt{\frac{7}{8}}\frac{C \sqrt{ n \log n}}{3 \delta \sqrt{\gamma_r} s_k}  - 1\right)\gamma_r s_k \right) \\
            &= \P\left(|T_r \cap V_k| - \E |T_r \cap V_k| \geq \left( C'  \frac{\hat{s}_r}{s_k} - 1\right)\gamma_r s_k \right)
        \end{align*}
        where we use the fact that by definition of $\gamma_r$,
        \[
        \sqrt{\frac{7}{8}}\frac{C \sqrt{ n \log n}}{3 \delta \sqrt{\gamma_r} s_k} = \sqrt{\frac{7}{8}}\frac{C \sqrt{ n \log n}}{3 \delta  s_k} \frac{\hat{s}_r \delta}{\sqrt{C_2 n \log n}} = \frac{\sqrt{\frac{7}{8}}C}{3\sqrt{C_2}} \frac{\hat{s}_r}{s_k} = C'  \frac{\hat{s}_r}{s_k}
        \]
        and the last equality is a definition. Therefore by Chernoff's inequality
        \begin{align}
            \P\left(|T_r \cap V_k| \geq \sqrt{\frac{7}{8}}\frac{C \sqrt{\gamma_r n \log n}}{3 \delta} \right) &\leq \exp \left( -\frac{\left( C'  \frac{\hat{s}_r}{s_k} -1\right)^2 \gamma_r s_k}{1 + C'  \frac{\hat{s}_r}{s_k} }\right). \label{eq:adaptive_faulty_oracle_clustering_chernoff}
        \end{align}
        Note that since $r \leq \underline{r}$ and $\frac{s_1}{C_3} \geq s_k$, by definition of $\underline{r}$ we have that
        \[\hat{s}_r \geq \hat{s}_{\underline{r}} = \frac{1}{2}\hat{s}_{\underline{r}-1} > \frac{1}{2}s_1 \geq \frac{C_3}{2}s_k.\]
        Also if $x$ is sufficiently large ($x \geq 5$ works) then we will have
        \begin{align*}
            \frac{(x-1)^2}{1+x} = \frac{x^2 -2x + 1}{1+x} \geq \frac{x}{2}.
        \end{align*}
        Using this in~\eqref{eq:adaptive_faulty_oracle_clustering_chernoff} by requiring $C_3$ to be sufficiently large so that $C'  \frac{\hat{s}_r}{s_k} \geq 5$, we have
        \begin{align*}
            \P\left(|T_r \cap V_k| \geq \sqrt{\frac{7}{8}}\frac{C \sqrt{\gamma_r n \log n}}{3 \delta} \right) &\leq \exp \left( - \frac{C'\frac{\hat{s}_r}{s_k}}{2} \gamma_r s_k\right) \\
            &= \exp \left( - \frac{C'}{2}\hat{s}_r \gamma_r \right) \\
            &= \exp \left( - \frac{C'}{2}C_2\frac{n}{\hat{s}_r} \frac{\log n}{\delta^2} \right) \\
            &\leq \exp \left( - \frac{C'}{2}C_2 \log n\right) \\
            &= \exp \left( - \frac{1}{2}\frac{\sqrt{\frac{7}{8}}C}{3\sqrt{C_2}} C_2 \log n\right) \\
            & \leq O(n^{-5})
        \end{align*}
        where we obtain the final inequality by requiring $C_2$ sufficiently large. Now we can conclude by taking a union bound over all $\leq n$ clusters $k$ such that $\frac{s_1}{C_3} \geq s_k$ and all $\leq n $ rounds $r \leq \underline{r}$.
    \end{proof}

    Now, for each round $r \in [\underline{r}]$, we can apply Lemma \ref{lem:faulty_oracle_recovery_SDP_success} to ensure the success of the recovery SDP when applied to the subsample $T_r$, and we can apply Lemma \ref{lem:majority_voting_success} to ensure the success of the majority voting procedure (if applied to a recovered subcluster). Each of these lemmas has failure probability $O(n^{-3})$ and we apply each $\underline{r} \leq n$ times, giving the claimed failure probability by a union bound. 
    
    The third item of Lemma \ref{lem:T_subsample_good_event_adaptive} combined with Lemma \ref{lem:faulty_oracle_recovery_SDP_success} guarantees that if the algorithm has not terminated by round $\underline{r}$, then a subcluster from cluster $1$ will be recovered on this round (as well as possibly other subclusters) so the algorithm will terminate by round $\underline{r}$. Therefore the maximum query complexity due to querying all pairs within subsamples $T_r$ (up until the first round a subcluster is recovered) is
    \begin{align*}
        O\left( |T_1|^2  + \dots + |T_{\underline{r}}|^2\right) &\leq O\left( \left(\frac{9}{8} \gamma_1 n\right)^2 + \dots + \left(\frac{9}{8} \gamma_{\underline{r}} n\right)^2 \right) \\
        & \leq O\left( \left(\frac{9}{8} \left(2^{1-1}\right)^2 \frac{\log n}{\delta^2}\right)^2 + \dots + \left(\frac{9}{8} \left(2^{\underline{r}-1}\right)^2 \frac{\log n}{\delta^2}\right)^2 \right) \\
        & \leq O\left(\frac{n^4}{s_1^4} \frac{\log^2 n}{\delta^4} \right)
    \end{align*}
    using the facts that $\gamma_r n = \frac{n^2 \log n}{\hat{s}_r^2 \delta^2} = \left(2^{r-1}\right)^2\frac{ \log n}{ \delta^2}$, that $2^{\underline{r}-1} \leq \frac{n}{s_1}$, and that a geometric series is bounded by a multiple of its largest term. Adding in the $O\left(\frac{n \log n}{\delta^2} \right)$ sample complexity for the majority voting procedure (since we only apply it to find one cluster), we obtain the claimed query complexity.

    Lastly, the second item of Lemma \ref{lem:T_subsample_good_event_adaptive} combined with Lemma \ref{lem:faulty_oracle_recovery_SDP_success} guarantees that no subcluster from a cluster $k$ such that $s_k \leq \frac{s_1}{C_3}$ will be recovered at any round up to and including round $\underline{r}$, so we are guaranteed that whenever the first subcluster is recovered, it must be from a cluster $k$ such that $s_k \geq \frac{s_1}{C_3}$. Furthermore from our earlier application of Lemma \ref{lem:majority_voting_success} the entirety of this cluster $k$ will be correctly recovered.
\end{proof}

Finally we present the proof of the guarantees for our algorithm for faulty oracle clustering when $K$ is bounded.

\begin{proof}[Proof of Theorem \ref{thm:faulty_oracle_clustering_improved_complexity_small_k}]
We face the same technical issue as in the proof of Theorem \ref{thm:recursive_clsutering} on recursive clustering, which concisely is that we would like to apply Theorem \ref{thm:main_theorem_recovery_sdp} to guarantee success of the recovery SDP~\eqref{eq:SDP_regularized_original} when given each constructed adjacency matrix $A_r$ as input, for each round $r$. This would simply follow from a union bound over all rounds if we knew which clusters would be present at each round. However, since there are mid-size clusters whose recovery is random, we face a combinatorial explosion of lists of clusters which could still be remaining at each round, and therefore we cannot guarantee the success of the recovery SDP for each option at each round via a union bound. Again, we believe that this could ideally be overcome by very carefully tracking each random event required for the success of the recovery SDP for each possible instance of remaining clusters and noting that the same events should guarantee success of the recovery SDP for many different instances. 

Instead, we devise the following technical workaround which involves modifying Algorithm \ref{alg:clustering_with_faulty_oracle_small_K}: Each round $r$, for any subcluster recovered by the recovery SDP, the majority voting procedure will (w.h.p.) correctly identify all nodes in the cluster, making the size of the cluster known. If the size of the cluster is at least $\frac{n_r}{K_r}$, we remove it as usual. However, if we recover a cluster smaller than $\frac{n_r}{K_r}$, we do not remove its nodes, but we remember the set $S$ of nodes. This way, we are guaranteed that each round $r$, only the clusters of size at least $\frac{n_r}{K_r}$ are removed, making the list of clusters which remain at each round deterministic (under the good event). If, during a later round $r'$, we again recover a subcluster of $S$, we do not apply the majority voting procedure again, and we will remove the nodes $S$ if and only if we now have $|S| \geq \frac{n_{r'}}{K_{r'}}$. This ensures that each cluster will only have the majority voting procedure applied to it once, ensuring that our workaround does not increase the query complexity.

Now we analyze this modified algorithm. First, similarly to the proof of Theorem \ref{thm:recursive_clsutering}, we let $F_1 = [K]$ and given $F_r$, we recursively construct $F_{r+1} = \{k \in F_r : s_k < \frac{\sum_{i \in F_r} s_i }{|F_r|}\}$.  This choice is so that (under the good event) $F_r$ will represent all clusters which have not been removed at the start of round $r$, and then $F_{r+1}$ is obtained by removing all clusters which have size at least $(\text{\# nodes remaining in $F_r$})/(\text{\# clusters remaining in $F_r$})$.

Now we define a good event for the subsamples $T_1, \dots, T_K$ (which may not all be reached by the algorithm). First, for convenience, we let $\hat{n}_r = \sum_{i \in F_r} s_i $ denote the number of nodes remaining in the clusters represented by $F_r$ and we let $\hat{K}_r = |F_r|$ denote the number of clusters represented by $F_r$. For each $r$, we sample $\hat{T}_r$ from the nodes which are members of the clusters included in $F_r$.
We differentiate these notationally from the $n_r, K_r, T_r$ which are defined over the course of the algorithm, but under the good event below, $\hat{n}_r,\hat{K}_r$ will be identical to $n_r, K_r$ and $\hat{T}_r$ will have the same distribution as $T_r$.
\begin{lem}
\label{lem:T_subsample_good_event_small_k}
    There exists an absolute constant $C_2$ such that, taking $\gamma_r = C_2 \frac{\hat{K}_r^2 \log n}{\hat{n}_r \delta^2}$, with probability at least $1- O(n^{-2})$:
    \begin{enumerate}
        \item For each $r \in [K]$, we have $\frac{7}{8}\gamma_r \hat{n}_r \leq \hat{T}_r \leq \frac{9}{8}\gamma_r \hat{n}_r$.
        \item For each cluster $k$, letting $r$ be the first round such that $s_k \geq \frac{\hat{n}_r}{\hat{K}_r}$ (if such a round exists), we have $\frac{7}{8}\gamma_r s_k \leq |\hat{T}_r \cap V_k| \leq \frac{9}{8}\gamma_r s_k$.
        \item For each cluster $k$, letting $r$ be the first round such that $s_k \geq \frac{\hat{n}_r}{\hat{K}_r}$ (if such a round exists), we have $|\hat{T}_r \cap V_k| \geq \frac{C \sqrt{|\hat{T}_r|\log n}}{\delta}$.
    \end{enumerate}
\end{lem}
\begin{proof}
    The first item is very similar to the first item of Lemma \ref{lem:T_subsample_good_event} except we use a different choice subsampling ratio $\gamma_r$. Fix a round $r$. For each node $i$ which is in a cluster $k \in F_r$, let $I_i$ be an indicator for the event that node $i$ is included in $\hat{T}_r$. Then $I_i$ is distributed as $\text{Bernoulli}(\gamma_r)$ and $|\hat{T}_r| = \sum_{i=1}^{\hat{n}_r} I_i$. Then by Chernoff's inequality,
        \begin{align*}
            \P \left( \left| |\hat{T}_r|-\gamma_r \hat{n}_r \right| \geq  \left(1+\frac{1}{8}\right) \gamma_r \hat{n}_r \right) & \leq \exp \left( -\frac{\gamma_r \hat{n}_r}{8^2 \cdot 3} \right) \\
            & = \exp \left( -\frac{C_2}{8^2 \cdot 3}  \frac{\hat{K}_r^2 \log n}{ \delta^2}\right) \\
            & \leq \exp \left( -\frac{C_2}{8^2 \cdot 3}  \log n \right) \\
            & \leq O(n^{-3})
        \end{align*}
        where the final inequality is for sufficiently large $C_2$. Now unfixing $r$ and applying a union bound over all $\leq K \leq n$ rounds, we obtain the first item.

        For the second item, fix a cluster $k$ and let $r$ be the first round such that $s_k \geq \frac{\hat{n}_r}{\hat{K}_r}$. By an identical argument to the previous step we have
        \begin{align*}
            \P \left( \left| |\hat{T}_r \cap V_k|-\gamma_r s_k \right| \geq  \left(1+\frac{1}{8}\right) \gamma_r s_k \right) & \leq \exp \left( -\frac{\gamma_r s_k}{8^2 \cdot 3} \right) \\
            & = \exp \left( -\frac{C_2}{8^2 \cdot 3}  \frac{\hat{K}_r^2 \log n}{ \hat{n}_r \delta^2}s_k\right) \\
            & \leq \exp \left( -\frac{C_2}{8^2 \cdot 3}  \frac{\hat{K}_r \log n}{ \delta^2}\right) \\
            & \leq \exp \left( -\frac{C_2}{8^2 \cdot 3}  \log n \right) \\
            & \leq O(n^{-3})
        \end{align*}
        using the fact that $s_k \geq \frac{\hat{n}_r}{\hat{K}_r}$. Now we obtain the second item by taking a union bound over all $K \leq n$ clusters.

        The third item follows from the previous two. For each cluster $k$, letting $r$ be the first round such that $s_k \geq \frac{\hat{n}_r}{\hat{K}_r}$ (if such a round exists), we have
        \begin{align*}
            |\hat{T}_r \cap V_k| &\geq \frac{7}{8}\gamma_r s_k \geq \frac{7 \sqrt{C_2}}{8} \frac{\hat{K}_r \sqrt{\log n}}{\delta \sqrt{\hat{n}_r}} \sqrt{\gamma_r} s_k \geq \frac{7 \sqrt{C_2}}{8} \frac{ \sqrt{\hat{n}_r \log n}}{\delta} \sqrt{\gamma_r} \\
            &\geq \frac{C \sqrt{\frac{9}{8}} \sqrt{\gamma_r \hat{n}_r \log n}}{\delta} \geq \frac{C \sqrt{|\hat{T}_r|\log n}}{\delta}
        \end{align*}
        where the second line is for sufficiently large $C_2$.
\end{proof}

Now, for each $r$, we apply Lemma \ref{lem:faulty_oracle_recovery_SDP_success} to guarantee the success of the recovery SDP when applied to the adjacency matrix obtained from the subsample $\hat{T}_r$ with probability at least $1-O(n^{-3})$, and combining this with the third item of the previous Lemma \ref{lem:T_subsample_good_event_small_k}, this guarantees that when the nodes are those belonging to the clusters of $F_r$, we will recover subclusters from each cluster $k$ such that $s_k \geq \frac{\hat{n}_r}{\hat{K}_r}$. Now applying Lemma \ref{lem:majority_voting_success} we are guaranteed that the majority voting procedure will correctly recover the remaining nodes from each such subcluster, because Lemma \ref{lem:T_subsample_good_event_small_k} guarantees that we will have
\[|\hat{T}_r \cap V_k| \geq \frac{7}{8}\gamma_r s_k = \frac{7}{8} C_2 \frac{\hat{K}_r^2 \log n}{\hat{n}_r \delta^2}s_k \geq \frac{7}{8}C_2 \frac{\hat{K}_r \log n}{\delta^2} \geq \frac{C_1 \log n}{\delta^2}.\]
The (modified) algorithm therefore ensures that all nodes belonging to clusters such that $s_k \geq \frac{\hat{n}_r}{\hat{K}_r}$ would be recovered and removed.

Now since this holds for any round $r$, and we start with $\hat{n}_1 = n_1 = n$ and $F_1 = [K]$, this guarantees that $T_1$ has the same distribution as $\hat{T}_1$, and then the above argument shows that with probability at least $1-O(n^{-2})$, for each round $r$, we have $n_r = \hat{n}_r$, $K_r = \hat{K}_r$, and the clusters $k$ such that $s_k \geq \frac{n_r}{K_r} = \frac{\hat{n}_r}{\hat{K}_r}$ are recovered and removed (which is at least one cluster), leaving exactly the clusters remaining in $F_{r+1}$.

Now it remains to check the query complexity and establish the recovery threshold. The majority voting procedure requires $O(\frac{n \log n}{\delta^2})$ queries per subcluster and gets applied at most $K$ times, giving a term of $O(\frac{nK \log n}{\delta^2})$. Under the aforementioned events, the subsampling procedure for the $r$th round makes
\[O(|T_r|^2) = O(\gamma_r^2 n_r^2) = O\left( \frac{K_r^4 \log^2 n}{\delta^4}\right) \leq O\left( \frac{K^4 \log^2 n}{\delta^4}\right)\]
queries, and there are obviously at most $K$ rounds, leading to a total query complexity of \[O\left(\frac{nK \log n}{\delta^2}  + \frac{K^5\log^2 n}{\delta^4}\right).\] To establish the recovery threshold, note that the algorithm terminates at the first round $r$ such that $n_r \leq C_2 \frac{K_r^2 \log n}{ \delta^2}$, implying that each remaining cluster is smaller than this amount and thus also smaller than $C_2 \frac{K^2 \log n}{ \delta^2}$.
\end{proof}

\section{Semirandom Recovery Proof}
\label{sec:semirandom_recovery_proofs}

\begin{proof}[Proof of Theorem \ref{thm:semirandom_recovery}]
    Recall that $A'$ is the adjacency matrix sampled from the standard SBM, and the observed $A$ is $A'$ subject to the semirandom perturbations. For purposes of analysis, we can apply our Theorem \ref{thm:main_theorem_recovery_sdp} to the SBM problem with adjacency matrix $A'$. Let $E$ be the event that the conclusions of Theorem \ref{thm:main_theorem_recovery_sdp} hold for the solution $\Yhat'$ of the analogous recovery SDP~\eqref{eq:SDP_regularized_original} but with $A'$ in place of $A$ (it is shown in Theorem \ref{thm:main_theorem_recovery_sdp} that $\P(E) \geq 1 - O(n^3)$). Under the event $E$, we have that $\Yhat'$ is the unique solution to the recovery SDP, meaning for all feasible $Y \neq \Yhat'$ we have
    \begin{align}
        &\left\langle \Yhat', A' - \frac{p+q}{2}J \right\rangle - \lambda \tr(\Yhat') > \left\langle Y, A' - \frac{p+q}{2}J \right\rangle - \lambda \tr(Y). \label{eq:semirandom_proof_step1}
    \end{align}
    Next we argue that also for all feasible $Y$, we have
    \begin{align}
        \left\langle \Yhat', A - A'  \right\rangle & \geq \left\langle Y, A - A'  \right\rangle.\label{eq:semirandom_proof_step2}
    \end{align}
    We show this in an elementwise fashion by using the form of $\Yhat'$ and the definition of the large cluster semirandom model. If $A_{ij} > A'_{ij}$, then $i,j$ must be two nodes within a cluster $k$ such that $\frac{p-q}{2}s_k \geq \frac{3}{2}B \sqrt{pn \log n} \geq \left(1 + \frac{1}{\overline{B}} \right)\sqrt{pn \log n}$, using the fact that $\kappa = B$ and $\overline{B} > 2$ so $\left(1 + \frac{1}{\overline{B}} \right) \leq \frac{3}{2}$. Therefore the conclusion of Theorem \ref{thm:main_theorem_recovery_sdp}, which hold under the event $E$, guarantees that $\Yhat'_{ij} = 1$. By feasibility of $Y$, we have $Y_{ij} \leq 1$, and therefore $\Yhat'_{ij}( A_{ij} - A'_{ij}) \geq Y_{ij}( A_{ij} - A'_{ij})$. If $A_{ij} < A'_{ij}$, then $i,j$ belong to two different clusters, and under the event $E$ we have $\Yhat'_{ij} = 0$. Since $Y_{ij} \geq 0$ by feasibility of $Y$, we again have that $\Yhat'_{ij}( A_{ij} - A'_{ij}) \geq Y_{ij}( A_{ij} - A'_{ij})$. Finally, this same inequality obviously holds if $A_{ij} = A'_{ij}$. Summing over $i,j$ we obtain the inequality~\eqref{eq:semirandom_proof_step2}.

    Finally, by adding inequalities~\eqref{eq:semirandom_proof_step1} and~\eqref{eq:semirandom_proof_step2} and cancelling $A'$, we obtain that
    \begin{align*}
        &\left\langle \Yhat', A - \frac{p+q}{2}J \right\rangle - \lambda \tr(\Yhat') > \left\langle Y, A - \frac{p+q}{2}J \right\rangle - \lambda \tr(Y).
    \end{align*}
    Therefore we have shown that $\Yhat'$ is also the unique optimal solution to the recovery SDP~\eqref{eq:SDP_regularized_original} (which uses $A$, not $A'$), meaning that $\Yhat = \Yhat'$ under the event $E$. Therefore we have shown that all conclusions of Theorem \ref{thm:main_theorem_recovery_sdp} hold under the event $E$, and thus Theorem \ref{thm:main_theorem_recovery_sdp} still holds under the large cluster semirandom model, as desired.
\end{proof}

\section{Eigenvalue Perturbation Bound Proofs}
\label{sec:eigenvalue_perturbation_bound_proofs}

\begin{proof}[Proof of Theorem \ref{thm:general_eval_pert_bound}]
We begin in the same way as the argument from \cite[Theorem 6]{eldridge2018unperturbed}.
Define $S_{t:T}$ as the set of all vectors in $\col(U_{t:T})$ with norm $1$, and likewise for $S_{T+1:n}$ and $S_{t:n}$.
Then by the Courant-Fischer min-max principle we have that
\begin{align*}
    \lambda_t\left(M+H\right) \leq \max_{x \in S_{t:n}} x^\top \left(M+H\right) x.
\end{align*}
Now we can write any $x \in S_{t:n}$ as $x = \alpha u + \beta u_\perp$, where $u \in S_{t:T}$, $u_\perp \in S_{T+1:n}$ and $\alpha^2 + \beta^2 = 1$.
Therefore the above is equivalent to
\begin{align*}
    \lambda_t(M+H) \leq  &\max_{\alpha, \beta, \alpha^2 + \beta^2 = 1} \max_{u \in S_{t:T}} \max_{u_\perp \in S_{T+1:n}} \bigg \{ \\
    &\alpha^2 u^\top M u + \alpha^2 u^\top H u + 2 \alpha \beta u^\top M u_\perp + 2 \alpha \beta u^\top H u_\perp + \beta^2 u_\perp ^\top M u_\perp + \beta^2 u_\perp^\top H u_\perp \bigg \}.
\end{align*}
Now we begin to deviate slightly from \cite[Theorem 6]{eldridge2018unperturbed}. We have $u^\top M u \leq \lambda_t$, $u^\top H u \leq h$ by assumption, $u^\top M u_\perp = 0$, $u^\top H u_\perp \leq \opnorm{H U_{t:T}}$, $u_\perp ^\top M u_\perp \leq \lambda_{T+1}$, and $u_\perp^\top H u_\perp \leq \opnorm{H}$. Thus
\begin{align*}
    \lambda_t(M+H) & \leq \max_{\alpha, \beta, \alpha^2 + \beta^2 = 1} \alpha^2 \lambda_t + \alpha^2 h + \alpha \beta  \opnorm{H U_{t:T}} + \beta^2 \lambda_{T+1} + \beta^2 \opnorm{H} \\
    & =\max_{\beta \in [0,1]} (1-\beta^2) (\lambda_t + h) + 2(1-\beta)\beta \opnorm{H U_{t:T}} + \beta^2 (\lambda_{T+1}+\opnorm{H}) \\
    &= \max_{\beta \in [0,1]} \lambda_t + h + \beta^2 (\lambda_{T+1} - \lambda_t + \opnorm{H} - h - 2\opnorm{H U_{t:T}}) + 2 \beta \opnorm{H U_{t:T}} .
\end{align*}
Now by our assumption the leading coefficient $(\lambda_{T+1} - \lambda_t + \opnorm{H} - h - 2\opnorm{H U_{t:T}})$ is $< 0$, so the quadratic has unique maximum over $\mathbb{R}$.

Using the elementary formula for the maximum of a quadratic, we have that 
\begin{align*}
    \lambda_t(M+H) &\leq \lambda_t + h + \frac{-4 \opnorm{H U_{t:T}}^2}{4(\lambda_{T+1} - \lambda_t + \opnorm{H} - h - 2\opnorm{H U_{t:T}})} \\
    &= \lambda_t + h + \frac{\opnorm{H U_{t:T}}^2}{\lambda_t(M) - \lambda_{T+1}(M) - \opnorm{H} + h + 2 \opnorm{HU_{t:T}}}.
\end{align*}
\end{proof}

\end{document}